\documentclass{article}

\usepackage[main,final]{neurips_2026}

\usepackage[utf8]{inputenc} %
\usepackage[T1]{fontenc}    %
\usepackage{hyperref}       %
\usepackage{url}            %
\usepackage{booktabs}       %
\usepackage{caption}
\usepackage{amsfonts}       %
\usepackage{nicefrac}       %
\usepackage{microtype}      %
\usepackage{xcolor}         %
\usepackage{pifont}         %
\usepackage{mathtools}       %

\definecolor{NeuripsDarkBlue}{RGB}{0, 51, 153}
\hypersetup{
  colorlinks=true,
  citecolor=NeuripsDarkBlue,
  linkcolor=NeuripsDarkBlue,
  urlcolor=NeuripsDarkBlue,
  hypertexnames=false
}
\setcitestyle{round}
\usepackage{algorithm}

\usepackage{algorithmic}
\usepackage{amsthm}
\usepackage{thmtools, thm-restate}

\newcommand{\mc}{\mathcal}

\newcommand{\A}{\mathcal{A}}
\newcommand{\Jset}{\mathcal{J}}
\newcommand{\Sset}{\mathcal{S}}

\usepackage{amsmath}
\DeclareMathOperator*{\argmax}{arg\,max}

\DeclareMathOperator{\softmax}{softmax}

\usepackage{appendix}
\usepackage{titletoc}
\titlecontents{section}
  [0em]
  {\small\bfseries}
  {\contentslabel{2.3em}}
  {}
  {\hfill\contentspage}
\titlecontents{subsection}
  [2.3em]
  {\small}
  {\contentslabel{2.8em}}
  {}
  {\hfill\contentspage}

\usepackage{amssymb}
\usepackage{tikz}
\usetikzlibrary{arrows.meta,positioning,calc,fit,shapes.misc,shapes.geometric,backgrounds,calc,bayesnet}
\usepackage[capitalize,noabbrev]{cleveref}
\usepackage{enumitem}
\usepackage{multirow}
\makeatletter
\renewcommand{\paragraph}{%
  \@startsection{paragraph}{4}{\z@}%
                {0.1ex}%
                {-0.5em}%
                {\normalsize\bfseries}%
}
\makeatother
\newtheoremstyle{tightthm}%
  {0pt}%
  {0pt}%
  {\itshape}%
  {}%
  {\bfseries}%
  {.}%
  {0.5em}%
  {}%

\newtheoremstyle{tightdef}%
  {0pt}%
  {0pt}%
  {\normalfont}%
  {}%
  {\bfseries}%
  {.}%
  {0.5em}%
  {}%

\newcommand{\BlackBox}{\rule{1.5ex}{1.5ex}}  %
\ifdefined\proof
    \renewenvironment{proof}{\par\noindent{\bf Proof\ }}{\hfill\BlackBox\par\medskip}
\else
    \newenvironment{proof}{\par\noindent{\bf Proof\ }}{\hfill\BlackBox\par\medskip}
\fi
\theoremstyle{tightthm}

\newtheorem{theorem}{Theorem}

\newtheorem{proposition}[theorem]{Proposition}

\newtheorem{corollary}[theorem]{Corollary}

\theoremstyle{tightdef}
\newtheorem{definition}[theorem]{Definition}
\definecolor{lightgray}{RGB}{240, 240, 240}
\definecolor{bordergray}{RGB}{180, 180, 180}

\definecolor{routerblue}{RGB}{218, 232, 252}
\definecolor{routerborder}{RGB}{108, 142, 191}

\definecolor{scorergreen}{RGB}{213, 232, 212}
\definecolor{scorerborder}{RGB}{130, 179, 102}

\definecolor{expertorange}{RGB}{255, 235, 204}
\definecolor{expertborder}{RGB}{215, 155, 0}

\definecolor{advicepurple}{RGB}{225, 213, 231}
\definecolor{adviceborder}{RGB}{150, 115, 166}

\definecolor{lossred}{RGB}{248, 206, 204}
\definecolor{lossborder}{RGB}{184, 84, 80}
\title{Beyond Augmented-Action Surrogates for Multi-Expert Learning-to-Defer}

\author{
  Yannis Montreuil \\
  School of Computing\\
  National University of Singapore\\
  Singapore, 117418 \\
  \texttt{yannis.montreuil@u.nus.edu} \\
   \And
    Axel Carlier \\
    Fédération ENAC ISAE-SUPAERO ONERA \\
    Université de Toulouse\\
    Toulouse, 31555, France \\
   \AND
    Lai Xing Ng \\
    Agency for Science, Technology and Research \\
    Institute for Infocomm Research \\
    Singapore, 138634 \\
    \And
    Wei Tsang Ooi \\
    School of Computing\\
    National University of Singapore\\
    Singapore, 117418 \\
    }

\begin{document}

\maketitle

\begin{abstract}
A learning-to-defer (L2D) system decides, for each input, whether to predict
on its own or defer to one of several experts. A widely used recipe trains
the classifier and the router jointly by treating the $K$ classes and $J$
experts as competing actions in one shared $(K{+}J)$-action geometry. We
analyze the population targets and optimization geometry of five
representative surrogates and identify target distortion, gradient
amplification, winner-take-all starvation, set-mass collapse, and
class--expert coupling, which can persist alongside statistical
consistency; shared competition and local overlap amplification also persist
beyond log-softmax. We depart from the augmented-action formulation with a
\emph{decoupled surrogate}, which pairs a softmax classifier head with an
independent sigmoid head for each expert. Its population minimizer recovers
the class posterior and expert utilities exactly, its coordinatewise
expert updates avoid overlap amplification and starvation, and its mixed head
Hessian block is zero even with a shared representation. We prove an
excess-risk bound with calibration constant
$\max\{\sqrt{2},\sqrt{2J/\lambda}\}$, which does not depend on the number of
experts at fixed per-expert weight. In synthetic studies the decoupled
surrogate preserves rare specialists, on CIFAR-10, CIFAR-10H, Covertype, and
Fashion-MNIST it preserves classifier quality, and Fashion-MNIST ablations
favor it over additive cross-entropy in all $36$ matched comparisons across
head layouts and trunk choices. With disjoint networks, the same objective
also supports two-stage training.
\end{abstract}

\section{Introduction}

Learning-to-Defer (L2D) augments a classifier with the option of handing an
input to an expert
\citep{madras2018predict, mozannar2021consistent}. In the multi-expert setting,
this means deciding, for each example, whether to predict directly or defer
to one of several experts. The Bayes rule compares the class posterior
$\eta_k(x)=\Pr(Y\!=\!k\mid X\!=\!x)$ with the expert utilities
$\alpha_j(x)=\Pr(M_j\!=\!Y\mid X\!=\!x)$, predicting when the best class
posterior exceeds the best expert utility and deferring otherwise. A surrogate should learn both quantities and support this comparison
from finite data.

Most existing multi-expert surrogates introduce an augmented action space
containing the $K$ class actions and $J$ defer-to-expert actions. They learn
a single score vector $a(x)\in\mathbb{R}^{K+J}$ and derive class and expert
supervision from these shared scores
\citep{mozannar2021consistent, Verma2022LearningTD, Cao_Mozannar_Feng_Wei_An_2023, mao2025mastering, liu2026more}.
This route has delivered statistically consistent surrogates, but it places
a categorical posterior and expert utilities inside one shared action
geometry. Prior work already shows that this geometry can produce unbounded
and misaligned defer-probability estimates
\citep{Verma2022LearningTD, Cao_Mozannar_Feng_Wei_An_2023}, and that adding
experts can cause systematic underfitting as the surrogate reward structure
distorts the gradient budget \citep{liu2026more}. Consistency secures the
Bayes decision, but it does not determine whether expert overlap amplifies
updates, whether correct non-winning experts are suppressed, or whether class
and expert scores remain coupled; these local effects follow from the loss
itself and can change as the expert pool grows. We therefore study each
surrogate along two axes: what it identifies at its population optimum, and
how it distributes supervision during training.

Our analysis begins with the general additive reduction of
\citet{mao2024principledapproacheslearningdefer}, whose comp-sum instances
retain shared competition and local overlap amplification beyond
log-softmax (Proposition~\ref{prop:additive-gradient}). The log-loss instance
has the exact, score-independent factor $1+|\Jset(y,m)|$ in its gradient
and Hessian, where $\Jset(y,m)$ is the realized set of correct experts
(Corollary~\ref{cor:mozannar-ce-geometry}), and its conditional target is
$(\eta,\alpha)/(1+\sum_j\alpha_j)$
(Proposition~\ref{prop:additive-optimum}). This common normalization
preserves the Bayes ordering but changes the raw probability representation
and the relative sample weights.

PiCCE~\citep{liu2026more} removes the amplification by
selecting a single correct expert per sample, but then pushes every other
correct expert \emph{down}, creating a starvation effect that can suppress rare
specialists (Proposition~\ref{prop:picce-starvation}).
\citet{mao2025mastering} avoid amplification
but learn only the total acceptable-set mass, with no ranking inside that set
(Proposition~\ref{prop:mao25-setmass}). A-SM
\citep{Cao_Mozannar_Feng_Wei_An_2023} corrects the target and produces bounded
estimates, but retains a class--expert gradient coupling that can be $O(J)$ (Proposition~\ref{prop:asm-coupling}). OvA
\citep{Verma2022LearningTD} fully decouples expert gradients, but treats the
class posterior as $K$ independent binary tasks rather than one categorical
distribution (Proposition~\ref{prop:ova-bernoulli}). 

We propose a \emph{decoupled surrogate} that separates class and expert
estimation: a softmax class head produces $p(x)\in\Delta^K$, and an
independent sigmoid head for each expert produces $u_j(x)\in(0,1)$. This
parameterization matches the Bayes-sufficient pair $(\eta,\alpha)$ without a
shared class--expert normalizer and without assuming that expert correctness
events are independent. Every correct expert receives its own attractive
gradient regardless of rival scores, and expert overlap does not rescale the
classifier score gradient. Separate head parameters remain uncoupled even
with a shared representation, and disjoint networks also permit two-stage
training of the same objective.

\paragraph{Contributions.}
\begin{itemize}[itemsep=0.1pt,topsep=2pt,leftmargin=*]
  \item[--] We analyze five representative augmented-action surrogates,
  with exact target and gradient/curvature statements showing how
  their action geometry
  causes target distortion, gradient amplification, winner-take-all
  starvation, set-mass collapse, class--expert coupling, or
  improper class geometry. The comp-sum analysis shows that competition
  persists beyond log-softmax (Section~\ref{sec:mismatch}).

  \item[--] We propose a decoupled surrogate that estimates the two Bayes
  objects on their native scales: a categorical class posterior and
  separately parameterized expert utilities. We prove exact recovery of $(\eta,\alpha)$
  at the population optimum, coordinatewise expert gradients without
  overlap amplification or starvation, and uniformly bounded score curvature.
  Separate head-parameter blocks remain uncoupled with a learned shared
  trunk, and disjoint networks support two-stage training of the same
  objective (Section~\ref{sec:approach}).

  \item[--] We prove a quantitative consistency bound: decoupled surrogate
  excess risk controls L2D excess risk with calibration constant
  $\max\{\sqrt{2},\sqrt{2J/\lambda}\}$, which is independent of~$J$ when
  the per-expert weight $\lambda/J$ is fixed. We also
  establish a dimension-independent transfer for unit-weight OvA
  (Theorem~\ref{thm:dec-consistency}, Proposition~\ref{prop:revision-ova-bound}).

  \item[--] Synthetic studies isolate the predicted pathologies, and four
  real-data benchmarks show that classifier quality is preserved and routing
  remains stable.
  Controlled Fashion-MNIST ablations favor the decoupled loss in all $36$
  paired seed--pool--architecture comparisons, including separate trunks.
  Weight, cost, and schedule studies test robustness and two-stage training
  (Section~\ref{sec:experiments}).
\end{itemize}

\section{Related Work}\label{sec:related}

Learning-to-defer extends selective prediction and classification with a
reject option \citep{Chow_1970, Bartlett_Wegkamp_2008, Geifman_El-Yaniv_2017, cortes, cao2022generalizing, cortes2024cardinalityaware} by allowing a learner
to route inputs to external experts rather than abstaining altogether. The
modern L2D formulation was introduced for single-expert systems by
\citet{madras2018predict} and \citet{mozannar2021consistent}, and has since
grown into a broader literature on surrogate design and statistical guarantees
\citep{Mao_Mohri_Zhong_2023, charusaie2022sample, Mozannar2023WhoSP, mao2024principledapproacheslearningdefer, mao2025mastering, montreuil2024twostagelearningtodefermultitasklearning, montreuil2026why, montreuil2026expertconditioned}.
Our surrogate-regret analysis complements the restricted-class
$\mathcal H$-consistency program; the distinction and further references
are discussed in Appendix~\ref{app:related-extended}.

\paragraph{Surrogate design and underfitting.}
Early predictor-rejector frameworks separate the classification and rejection
functions, but \citet{ni2019calibration} show that such surrogates are
inconsistent in the multiclass case. \citet{mozannar2021consistent} introduce
the first consistent multiclass surrogate via an augmented action space that
scores the class and deferral actions jointly, and subsequent work refines
this shared-vector design. \citet{Verma2022LearningTD} extend it to multiple
experts, compare coupled softmax and one-vs-all (OvA)
parameterizations, and show that the softmax defer-probability estimate is
unbounded, whereas OvA remains bounded.
\citet{Cao_Mozannar_Feng_Wei_An_2023} show that symmetric reductions can
produce invalid defer probabilities and propose an asymmetric softmax
parameterization with a multi-expert extension.
\citet{mao2025mastering} design a
single-stage surrogate family with formal consistency guarantees for
multi-action deferral. A complementary line studies optimization pathologies:
\citet{Narasimhan} first identify classifier underfitting in the
single-expert setting and propose a post-hoc correction, and
\citet{liu2024mitigating} propose a training-time mitigation in the same
setting. More recently, \citet{liu2026more} show that in the multi-expert
setting even redundant experts can induce underfitting through the additive
expert reward, and they propose a winner-take-all correction (PiCCE).

\paragraph{Two-stage and utility-based approaches.}
Two-stage L2D first learns a predictor and then fits a router
\citep{Narasimhan,mao2023twostage,mao2024regressionmultiexpertdeferral}.
These methods provide guarantees relative to the frozen predictor, including
restricted-class and realizability results. Our disjoint formulation also
supports two-stage training, while learning marginal utilities directly;
Appendix~\ref{app:revision-two-stage} compares the underlying objectives.
Classification with rejection can likewise be reduced to cost-sensitive
classification~\citep{charoenphakdee2021classification}, and dependence-aware
L2D uses further information about human--AI complementarity~\citep{wei2024exploiting}.
Relative to these lines of work, our contribution is a joint analysis of
population targets and optimization geometry, followed by an objective that
keeps the categorical class posterior and the expert utilities on their
native scales.

\section{Preliminaries}\label{sec:preliminaries}

\paragraph{Multi-expert learning to defer.} Following the modern
learning-to-defer literature \citep{madras2018predict, mozannar2021consistent, mao2024principledapproacheslearningdefer, montreuil2026why}, let
$(\Omega,\mathcal{F},\mathbb{P})$ be a probability space supporting a
triple $(X,Y,\mathbf{M})$ with distribution $\mathcal{D}$ on
$\mathcal{X}\times[K]\times[K]^J$, where
$\mathcal{X}\subseteq\mathbb{R}^d$ is the input space,
$[K]\coloneqq\{1,\dots,K\}$ is the label space, and
$\mathbf{M}=(M_1,\dots,M_J)$ records the predictions of $J$ experts. We write
$(x,y,m)$ for a realization of $(X,Y,\mathbf{M})$.

The learner consists of a classifier $h:\mathcal{X}\to[K]$ and a router
$r:\mathcal{X}\to\{0\}\cup[J]$. The decision $r(x)=0$ means that the learner
predicts with $h(x)$, whereas $r(x)=j$ means that the instance is routed to
expert~$j$. In the classification-specialized setting we consider \citep{mozannar2021consistent, Verma2022LearningTD, Cao_Mozannar_Feng_Wei_An_2023, mao2024principledapproacheslearningdefer, liu2026more}, the only question is which of the $K$ class predictions and
the $J$ expert actions yields the smallest conditional error.

\begin{definition}[Multi-Expert Defer Loss]\label{def:defer-loss}
The multi-expert defer loss incurred on a realization $(x,y,m)$ is
\begin{equation}\label{eq:defer-loss}
\ell_\perp(h,r;\,x,y,m)
\;=\;
\mathbf{1}\{h(x)\neq y\}\mathbf{1}\{r(x)=0\}
\;+\;
\sum_{j\in[J]}\mathbf{1}\{m_j\neq y\}\mathbf{1}\{r(x)=j\}.
\end{equation}
Its population risk is
$\mathcal{E}_\perp(h,r)\coloneqq\mathbb{E}[\ell_\perp(h,r;\,X,Y,\mathbf{M})]$.
\end{definition}

At each input, the learner either predicts on its own and incurs the usual
misclassification loss, or it transfers control to one of the experts and pays
that expert's error on the same example. The population objective is to choose
the best option pointwise in expectation.

\paragraph{Bayes quantities.}
The Bayes rule is governed by two families of conditional probabilities:
the class posterior $\eta_k(x)\coloneqq \Pr(Y=k\mid X=x)$ for $k\in[K]$, and
the expert utilities $\alpha_j(x)\coloneqq \Pr(M_j=Y\mid X=x)$ for $j\in[J]$.
These quantities are the natural objects of the defer problem, since
predicting class $k$ incurs conditional risk $1-\eta_k(x)$, while deferring
to expert~$j$ incurs conditional risk $1-\alpha_j(x)$. Thus the optimal action depends only
on the largest class posterior and the largest expert utility, not on an
augmented distribution over $K+J$ actions.

\begin{restatable}[Bayes Rule for Multi-Expert L2D]{lemma}{bayesrule}\label{lem:bayes-rule}
For every $x\in\mathcal{X}$, a Bayes-optimal decision is obtained by
\begin{equation}\label{eq:bayes-rule}
r^\star(x)
\;\in\;
\begin{cases}
\{0\}, &
\text{if }\max_{k\in[K]}\eta_k(x)\ge\max_{j\in[J]}\alpha_j(x),\\[4pt]
\argmax_{j\in[J]}\alpha_j(x), & \text{otherwise,}
\end{cases}
\end{equation}
with $h^\star(x)\in\argmax_{k\in[K]}\eta_k(x)$.
\end{restatable}

Lemma~\ref{lem:bayes-rule} exposes the statistical structure of the
problem: the class posterior $\eta(x)$ is a categorical distribution and
therefore lies in the simplex $\Delta^K$, whereas the expert utilities
$\alpha_1(x),\dots,\alpha_J(x)$ are $J$ conditional probabilities in $[0,1]$
with no analogous normalization constraint. The Bayes-sufficient object is
therefore the pair $(\eta(x),\alpha(x))\in\Delta^K\times[0,1]^J$.

\paragraph{Augmented-action viewpoint.}
Because the loss in Definition~\ref{def:defer-loss} is discontinuous,
gradient-based learning must proceed through a surrogate \citep{Statistical, bartlett2006convexity, Steinwart2007HowTC}. The dominant one-stage reduction, introduced
by \citet{mozannar2021consistent}, casts the problem as prediction over one enlarged action space
$\A \coloneqq [K]\cup\{\perp_1,\dots,\perp_J\}$ and learns one score per
action.

\begin{definition}[Augmented-Action Surrogate Family]\label{def:aug-family}
An \emph{augmented-action surrogate} is specified by a score map
$a:\mathcal{X}\to\mathbb{R}^{K+J}$ together with a samplewise loss of the form
\[
\Phi^{\mathrm{aug}}(a;\,x,y,m)
\;=\;
\Phi^{\mathrm{aug}}_{\mathrm{cls}}(a;\,x,y)
\;+\;
\Phi^{\mathrm{aug}}_{\mathrm{exp}}(a;\,x,y,m),
\]
and a prediction rule that selects one augmented action from the same shared
score vector $a(x)$, either by $\argmax_i a_i(x)$ or by the argmax of a
maxima-preserving transform of $a(x)$.
\end{definition}

This definition covers the baselines studied in the paper: additive
cross-entropy \citep{mozannar2021consistent}, Verma's OvA surrogate
\citep{Verma2022LearningTD}, A-SM \citep{Cao_Mozannar_Feng_Wei_An_2023},
PiCCE \citep{liu2026more}, and Mao25
\citep{mao2025mastering}. They differ in how
they build the class and expert terms, but they all learn defer decisions
through one shared action geometry over classes and experts.

\section{The Augmented-Action Mismatch}\label{sec:mismatch}

We analyze the augmented-action surrogates of
Definition~\ref{def:aug-family} along two axes:
\emph{(i)}~the conditional target at the population optimum, and
\emph{(ii)}~the local gradient and curvature geometry during training.
Throughout,
$\Jset(y,m)\coloneqq \{j\in[J]:m_j=y\}$ denotes the set of correct experts
and $q_i(x)\coloneqq\exp(a_i(x))/\sum_{r=1}^{K+J}\exp(a_r(x))$. Detailed
proofs, curvature calculations, and worked examples for the baseline
analyses are collected in
Appendices~\ref{app:proofs-mismatch}--\ref{app:constant-comparison}.

\subsection{General Additive Reductions: Competition and Overlap}
\label{subsec:amplification}

The general additive reduction of \citet{mao2024principledapproacheslearningdefer}
uses a multiclass base loss $\ell$ and score-independent costs $c_j\in[0,1]$:
$\Phi_\ell=\ell(a,y)+\sum_j(1-c_j)\ell(a,K+j)$.
In our zero-one setting, $c_j=\mathbf1\{m_j\ne y\}$, giving one term for
the true class and one for each correct expert. These terms share the same
augmented scores; the base loss determines how their gradients interact.

\paragraph{Shared competition across the comp-sum family.}
Consider the comp-sum base losses
\citep{mao2023crossentropylossfunctionstheoretical}:
\[
\ell_\tau(a,i)=\Upsilon_\tau\!\left(\sum_{r\ne i}e^{a_r-a_i}\right)
=\begin{cases}
(q_i^{\tau-1}-1)/(1-\tau),&\tau\ne1,\\
-\log q_i,&\tau=1,
\end{cases}\qquad \tau\ge0.
\]
This family includes sum-exponential ($\tau=0$), log loss ($\tau=1$),
generalized CE ($1<\tau<2$), and MAE ($\tau=2$).

\begin{restatable}[Competition and overlap beyond log loss]{proposition}{additivegradient}
\label{prop:additive-gradient}
Let $d=K+J$, $a\in\mathbb R^d$, $q=\softmax(a)$, and let $e_i$ be the standard basis vectors. Set
$\mathcal T=\{y\}\cup\{K+j:j\in\Jset(y,m)\}$, $N=|\mathcal T|$, and
$L_\tau=\sum_{i\in\mathcal T}\ell_\tau(a,i)$.
Write $C(q)=\mathrm{Diag}(q)-qq^\top$, and define
\[
\bar q_\tau=\frac1N\sum_{i\in\mathcal T}q_i^{\tau-1},\qquad
\omega_i=\frac{q_i^{\tau-1}}{N\bar q_\tau},\qquad
\bar e_\tau=\sum_{i\in\mathcal T}\omega_i e_i,\qquad
\bar R_\tau=\sum_{i\in\mathcal T}\omega_i(q-e_i)(q-e_i)^\top.
\]
Then, for every $\tau\ge0$,
\begin{equation}\label{eq:general-compsum-geometry}
\nabla_aL_\tau=N\bar q_\tau(q-\bar e_\tau),\qquad
\nabla_a^2L_\tau=N\bar q_\tau[C(q)-(\tau-1)\bar R_\tau].
\end{equation}
Every omitted action $k\notin\mathcal T$ has
$\partial L_\tau/\partial a_k=N\bar q_\tau q_k>0$.
At uniform scores this derivative is $Nd^{-\tau}$, and every nonzero
zero-sum direction supported on omitted actions is an eigenvector of the
Hessian with eigenvalue $Nd^{-\tau}$. Relative to the true-class-only loss at the same
scores and dimension, both quantities are amplified exactly by $N$.
\end{restatable}

Every added target changes all other coordinate gradients through $q$, and
since the proof holds for any target set $\mathcal T$, a correct expert
omitted by winner-take-all selection is pushed down throughout the family. The effective weight $N\bar q_\tau$ depends on the
number of active targets and their current probabilities. For
$\tau\le1$, the Hessian contributions are positive semidefinite; for
$\tau>1$, the correction $-(\tau-1)\bar R_\tau$ can introduce negative
curvature. At $\tau=2$,
$N\bar q_2=\sum_{i\in\mathcal T}q_i$ and
$L_2=N-\sum_{i\in\mathcal T}q_i$, giving the set-mass geometry studied
in Section~\ref{subsec:refined}. Weighted reductions and arbitrary increasing
comp-sum outer functions are analyzed in Appendix~\ref{app:revision-compsum}.

\paragraph{Log-softmax: global amplification and target distortion.}
For $\tau=1$, $\bar q_1=1$ and the correction term vanishes,
giving Mozannar's additive CE surrogate \citep{mozannar2021consistent}:
\begin{equation}\label{eq:additive-ce}
\Phi^{\mathrm{CE}}(a;x,y,m)
=-\log q_y(x)-\sum_{j\in\Jset(y,m)}\log q_{K+j}(x).
\end{equation}
\begin{restatable}[Mozannar's log-softmax instance]{cecorollary}{mozannarcegeometry}
\label{cor:mozannar-ce-geometry}
For score-independent weights $v_i\ge0$, set $V=\sum_i v_i$
and $L_v=-\sum_i v_i\log q_i$. Then
\begin{equation}\label{eq:additive-gradient}
\frac{\partial L_v}{\partial a_i}=Vq_i-v_i,
\end{equation}
\begin{equation}\label{eq:ce-hessian}
\nabla_a^2L_v=VC(q),\qquad
\lambda_{\max}(\nabla_a^2L_v)\le V/2.
\end{equation}
For Add.\ CE~\citep{mozannar2021consistent}, $V=N=1+|\Jset|$.
\end{restatable}

The conditional log-loss optimum is
$q^\star=(\eta,\alpha)/(1+U)$, where $U(x)=\sum_j\alpha_j(x)$
(Appendix~\ref{app:proof-additive-optimum}). This target preserves the Bayes
ordering but normalizes the class posterior and the expert utilities
together, so duplicating an expert with $\alpha_j(x)>0$ enlarges the
denominator without changing the best action; the same normalization
underlies the unbounded defer-probability reconstruction
\citep{Verma2022LearningTD,Cao_Mozannar_Feng_Wei_An_2023}. Overlap therefore
changes both the represented target and the samplewise update weight,
consistent with redundant-expert underfitting \citep{liu2026more}.

\subsection{Why Winner-Take-All Fails}\label{subsec:wta}

For log-softmax, Section~\ref{subsec:amplification} identifies the count
factor caused by rewarding every correct expert. PiCCE~\citep{liu2026more}
caps this factor by rewarding only one correct expert per sample.
We analyze its log-loss instance and show that the same selection creates
the omitted-action suppression predicted by Proposition~\ref{prop:additive-gradient}. Let $j^\star$ be the deterministically tie-broken
maximizer of $a_{K+j}$ over $j\in\Jset(y,m)$ when this set is nonempty. Then
\begin{equation}\label{eq:picce}
\Phi^{\mathrm{PiCCE}}(a;\,x,y,m)
\;=\;
-\log q_y(x)
\;-\;\mathbf{1}\{\Jset\neq\varnothing\}\log q_{K+j^\star}(x).
\end{equation}
With at most two log-loss terms active, the total positive gradient mass is
capped at~$2$ however many experts are correct, and with a unique correct-expert winner the Hessian is
$\nabla_a^2\Phi^{\mathrm{PiCCE}}=2(\mathrm{Diag}(q)-qq^\top)$, so
$\lambda_{\max}\le 1$ uniformly and the multiplicity factor
of~\eqref{eq:ce-hessian} disappears on axis~(ii)
(Appendix~\ref{app:proof-picce-curvature}).

\paragraph{Starvation of non-winning experts (axis~(ii)).}
\begin{restatable}[Winner-take-all starvation in PiCCE]{proposition}{piccestarvation}
\label{prop:picce-starvation}
For $\Jset(y,m)\neq\varnothing$ and a unique winning correct expert, the gradient on any logit is
$\partial\Phi^{\mathrm{PiCCE}}/\partial a_i
= 2q_i - \mathbf{1}\{i=y\}-\mathbf{1}\{i=K\!+\!j^\star\}$.
For any correct expert $j'\in\Jset$ with $j'\neq j^\star$,
\begin{equation}\label{eq:picce-starvation}
\frac{\partial\Phi^{\mathrm{PiCCE}}}{\partial a_{K+j'}}
\;=\; 2q_{K+j'}(x) \;>\; 0.
\end{equation}
Gradient descent pushes expert~$j'$'s logit \emph{downward}, even though it is
correct. At a tie, the same formula is the derivative of the selected branch.
\end{restatable}

The positive term $2q_i$ is intended for wrong classes, but it also
suppresses every correct expert $j'\in\Jset\setminus\{j^\star\}$ solely
because a rival expert had a higher score.

\paragraph{Self-reinforcing lock-in and a moving target.}
If $j_1$ starts with a marginally
higher logit than $j_2$ on a shared region, $j_1$ wins the $\argmax$ and
receives gradient $2q_{K+j_1}-1$, while $j_2$ is pushed down
($2q_{K+j_2}>0$). In the non-saturated regime $q_{K+j_1}<1/2$, the winner is
pulled up and the gap widens; even after saturation, the non-winning correct
expert continues to receive the repulsive positive gradient. This positive
feedback can turn a small initial score gap into persistent suppression.
The same mechanism also distorts the statistical target on axis~(i): PiCCE's
target is proportional to $(\eta,\gamma(a,\cdot))$ rather than $(\eta,\alpha)$, where
$\gamma_j(a,x)=\Pr(j=j^\star \text{ and } j\in\Jset\mid X\!=\!x)$ depends
on the current scores through the winner map
(Appendix~\ref{app:picce-moving-target}). As starvation suppresses
experts, $\gamma_j$ shifts mass toward surviving winners and entrenches
them, so the target \emph{moves} with the scores rather than being fixed by
the data distribution.

\subsection{Three Refined Surrogates, Three Residual Gaps}
\label{subsec:refined}

We now examine three alternatives to winner selection: an acceptable-set
objective, a separate class normalizer, and independent binary losses.

\paragraph{Mao25: correct set, no ranking.}
The linear member of the family of \citet{mao2025mastering}
penalizes the total mass assigned outside the acceptable set,
$\Phi^{\mathrm{Mao}} = 1-S_{\Sset}$ with
$S_{\Sset}=q_y+\sum_{j\in\Jset} q_{K+j}$. The loss has consistency
guarantees \citep{mao2025mastering}, but it is \emph{linear} in the
acceptable mass $S_{\Sset}$, which makes it a set-valued analogue of
multiclass MAE \citep{Ghosh}.

\begin{restatable}[Acceptable-set mass in Mao25]{proposition}{maosetmass}
\label{prop:mao25-setmass}
The gradient is
$\partial\Phi^{\mathrm{Mao}}/\partial a_i
= q_i\bigl(S_{\Sset}-\mathbf{1}\{i\in\Sset(y,m)\}\bigr)$,
where $\Sset(y,m)=\{y\}\cup\{K+j:j\in\Jset\}$ and
$S_{\Sset}=\sum_{c\in\Sset}q_c$. The loss depends on the acceptable set only
through $S_{\Sset}$, and for acceptable coordinates the gradient has the common
factor $S_{\Sset}-1$. Thus the samplewise update distinguishes acceptable from
unacceptable actions, but it contains no term that prefers one acceptable
expert over another on the basis of expert quality.
\end{restatable}

Because the samplewise supervision cannot express preferences inside the
acceptable set, Mao25 does not recover the expert utilities directly: it
can identify a Bayes-optimal action, but it does not estimate the full
vector $(\eta,\alpha)$ needed to compare experts at the utility level, a
limitation on axis~(i). The same MAE-like structure also weakens optimization:
the total attractive gradient mass on acceptable actions is
$\sum_{i\in\Sset} -\partial\Phi^{\mathrm{Mao}}/\partial a_i
= S_{\Sset}(1-S_{\Sset})$, which is at most $1/4$ and vanishes as
$S_{\Sset}\to 0$. When the model assigns little mass to the acceptable set, its corrective
signal is therefore weak, consistent with the optimization
behavior reported for MAE-like losses
\citep{zhang2018generalizedcrossentropyloss}.

\paragraph{A-SM: correct target, coupled geometry.}
The asymmetric softmax of \citet{Cao_Mozannar_Feng_Wei_An_2023} defines
class probabilities
$\xi_k=\exp(a_k)/\sum_{k'}\exp(a_{k'})$ over classes only and expert
estimates via a sigmoid
$\psi_j=\sigma\!\bigl(a_{K+j}-\log\sum_{k'\neq k^\star}\exp(a_{k'})\bigr)$
with $k^\star\!\in\!\argmax_k a_k$
(Definition~\ref{def:appendix-asm}). At the population optimum
$\xi_k^\star=\eta_k$ and $\psi_j^\star=\alpha_j$: axis~(i) is fully
corrected. However, $\psi_j$ still depends on the class logits, which
couples the two estimation problems on axis~(ii).

\begin{restatable}[Class--expert gradient coupling in A-SM]{proposition}{asmcoupling}
\label{prop:asm-coupling}
Fix $k^\star\!\in\!\argmax_{k\in[K]} a_k$. For any $r\neq k^\star$, let
$\pi_r=\exp(a_r)/\sum_{k\neq k^\star}\exp(a_k)$. Then
\begin{equation}\label{eq:asm-coupling}
\frac{\partial\Phi^{\mathrm{ASM}}}{\partial a_r}
\;=\;
\underbrace{(\xi_r-\mathbf{1}\{r\!=\!y\})}_{\text{class term}}
\;-\;\pi_r\!\underbrace{\sum_{j=1}^J
\bigl(\psi_j-\mathbf{1}\{m_j\!=\!y\}\bigr)}_{\text{expert leakage}}.
\end{equation}
\end{restatable}

The leakage term vanishes in expectation at the optimum but can be of
order $J$ on individual samples, and the mixed Hessian block has operator
norm at most $\sqrt{J}/4$ (Appendix~\ref{app:proof-asm-curvature}), a bound
that grows with the expert pool. A-SM is thus the strongest augmented
baseline on axis~(i), and its residual weakness lies on axis~(ii).

\paragraph{OvA: decoupled gradients, improper class posterior.}
OvA \citep{Verma2022LearningTD} abandons the shared simplex entirely,
replacing it with $K+J$ independent binary logistic losses
(Definition~\ref{def:appendix-ova}). Because every coordinate is optimized
independently, the amplification, starvation, and coupling pathologies are
all absent, which resolves axis~(ii). The remaining gap lies on axis~(i):
the class head consists of $K$ independent sigmoids, so
the estimates need not lie in $\Delta^K$ away from the optimum \citep{Cao_Mozannar_Feng_Wei_An_2023}.

\paragraph{Synthesis.}
Table~\ref{tab:comparison} summarizes these properties: shared
normalization accounts for the pathologies of the coupled softmax
surrogates, while OvA avoids it but does not impose the categorical
constraint on its class block. Two design principles follow, one for each
axis: \emph{separate} the
class and expert estimation problems (no shared normalization, no
gradient coupling), addressing axis~(ii); and \emph{match} each head
to the statistical type of its target (softmax for $\eta\in\Delta^K$,
independent sigmoids for $\alpha_j\in[0,1]$), addressing axis~(i).

\begin{table}[t]
\centering
\caption{Structural properties of multi-expert L2D surrogates.
\checkmark\ = satisfied; \ding{55}\ = violated; -- = no bound asserted
here. Axis column: \textsc{(i)} statistical target,
\textsc{(ii)} local geometry, \textsc{(th)} theoretical guarantee. Each
entry cites the proposition, definition, equation, or appendix subsection
that establishes it.}
\label{tab:comparison}
\footnotesize
\setlength{\tabcolsep}{3pt}
\resizebox{\textwidth}{!}{%
\begin{tabular}{@{}lccccccc@{}}
\toprule
&& Add.\ CE & PiCCE & Mao25 & A-SM & OvA & \textbf{Decoupled} \\
\midrule
Conditional minimizer recovers the Bayes pair $(\eta,\alpha)$
& \textsc{(i)}
& \ding{55}\,(\ref{prop:additive-optimum})
& \ding{55}\,(\ref{app:picce-moving-target})
& \ding{55}\,(\ref{prop:mao25-setmass})
& \checkmark\,(\ref{def:appendix-asm})
& \checkmark\,(\ref{prop:ova-bernoulli})
& \checkmark\,(\ref{thm:dec-consistency}) \\
Per-sample gradient and Hessian free of multiplicative $O(|\Jset|)$ factor
& \textsc{(ii)}
& \ding{55}\,(\ref{prop:additive-gradient})
& \checkmark\,(\ref{app:proof-picce-curvature})
& \checkmark\,(\ref{prop:mao25-setmass})
& \checkmark\,(\ref{prop:asm-coupling})
& \checkmark\,(\ref{app:proof-ova-curvature})
& \checkmark\,(\ref{prop:due-gradient}) \\
No winner-take-all suppression of correct experts
& \textsc{(ii)}
& \checkmark\,(\ref{prop:additive-gradient})
& \ding{55}\,(\ref{prop:picce-starvation})
& \checkmark\,(\ref{prop:mao25-setmass})
& \checkmark\,(\ref{prop:asm-coupling})
& \checkmark\,(\ref{prop:ova-bernoulli})
& \checkmark\,(\ref{prop:due-gradient}) \\
No class--expert coupling in the per-sample gradient
& \textsc{(ii)}
& \ding{55}\,(\ref{prop:additive-gradient})
& \ding{55}\,(\ref{prop:picce-starvation})
& \ding{55}\,(\ref{prop:mao25-setmass})
& \ding{55}\,(\ref{prop:asm-coupling})
& \checkmark\,(\ref{app:proof-ova-curvature})
& \checkmark\,(\ref{eq:due-hessian}) \\
Class head is a categorical posterior throughout training
& \textsc{(i)}
& \ding{55}\,(\ref{def:aug-family})
& \ding{55}\,(\ref{def:aug-family})
& \ding{55}\,(\ref{def:aug-family})
& \checkmark\,(\ref{def:appendix-asm})
& \ding{55}\,(\ref{prop:ova-bernoulli})
& \checkmark\,(\ref{def:due}) \\
Quantitative surrogate-to-target excess-risk bound
& \textsc{(th)}
& \checkmark$^\dagger$ & -- & \checkmark$^\ddagger$
& -- & \checkmark\,(\ref{prop:revision-ova-bound})
& \checkmark\,(\ref{thm:dec-consistency}) \\
\bottomrule
\end{tabular}}

{\scriptsize ${}^\dagger$Comp-sum reduction
\citep{mao2024principledapproacheslearningdefer}, constant $\sqrt{2(J{+}1)}$.
${}^\ddagger$Linear member of the single-stage family
\citep{mao2025mastering}: $(K+J)\Delta\mathcal E_{\mathrm{Mao}}$
(Appendix~\ref{app:constant-comparison}).}
\end{table}

\section{The Decoupled Surrogate}\label{sec:approach}

The decoupled surrogate implements the two principles from
Section~\ref{sec:mismatch}: separate class and expert scores, and match
each head to its target so that the conditional minimizer recovers
$(\eta,\alpha)$ directly.

\subsection{Formulation and Prediction}\label{subsec:formulation}

Let $w:\mathcal{X}\to\mathbb{R}^K$ and $s:\mathcal{X}\to\mathbb{R}^J$ denote the classifier and expert score maps.
The decoupled surrogate estimates the class posterior and the expert utilities on their native
probability scales: $p_k(x)=\exp(w_k(x))/\sum_{\ell=1}^K\exp(w_\ell(x))$ and $u_j(x)=\sigma(s_j(x))$,
where $\sigma$ is the logistic sigmoid. The class estimate
$p(x)\in\Delta^K$ is a proper categorical distribution, while each expert
estimate $u_j(x)\in(0,1)$ is parameterized separately, without assuming that
the expert correctness events are independent. These types mirror the
Bayes-sufficient pair $(\eta,\alpha)\in\Delta^K\times[0,1]^J$ exactly, and no
shared simplex forces classes and experts to compete for probability mass.
The corresponding bounded cost-sensitive formulation is given in
Appendix~\ref{app:proofs-approach}; the main text restricts attention to the
classification-specialized zero-one defer case.

\begin{definition}[Decoupled surrogate]\label{def:due}
For $\lambda>0$, the decoupled surrogate is
\begin{equation}\label{eq:due}
\Phi^{\mathrm{dec}}_\lambda(w,s;\,x,y,m)
\;\coloneqq\;
-\log p_y(x)
\;-\;\frac{\lambda}{J}\sum_{j=1}^J
\Bigl(
t_j\log u_j(x)
+(1-t_j)\log\bigl(1-u_j(x)\bigr)
\Bigr),
\end{equation}
where $t_j\coloneqq \mathbf{1}\{m_j=y\}$.
\end{definition}

The first term, $\Phi_{\mathrm{cls}}$, is the multiclass log loss and the
second, $\Phi_{\mathrm{exp}}$, is the $\lambda$-weighted mean binary
cross-entropy (BCE) over experts. The total expert weight is
$\lambda$, and each expert coordinate receives the per-expert weight
$\lambda/J$, which scales its gradient, curvature, and transfer coefficient.
Every $\lambda>0$ yields the same unrestricted conditional minimizer, but
with restricted models or a shared representation the weight also balances
the two estimation tasks (Appendix~\ref{app:due-lambda-choice}). A single
$(K{+}J)$-dimensional output projection can implement both heads, with a
softmax on the class block and a sigmoid on the expert block; the decoupling
comes from the loss and its separate normalizers rather than from separate
output layers.

\paragraph{Why we average expert losses.}
Suppose each of the $J$ experts is copied $r\ge1$ times and duplicate
heads have identical outputs. Writing $B_j$ for expert~$j$'s BCE term,
the expert loss at fixed total weight $\lambda$ satisfies
\[
\frac{\lambda}{rJ}\sum_{b=1}^{r}\sum_{j=1}^J B_j
=\frac{\lambda}{J}\sum_{j=1}^J B_j.
\]
Since the class term is unchanged, duplication preserves the full sample
loss at these matched outputs; in the Fashion-MNIST ablations, the duplicate
heads are instead learned freely.

At test time, $h_{\mathrm{dec}}(x)\in\argmax_kp_k(x)$, and the router
compares the largest class and expert estimates:
\begin{equation}\label{eq:due-routing}
r_{\mathrm{dec}}(x)
\;\in\;
\begin{cases}
\{0\}, &
\text{if }\max_{k\in[K]}p_k(x)\ge\max_{j\in[J]}u_j(x),\\[4pt]
\argmax_{j\in[J]}u_j(x), & \text{otherwise.}
\end{cases}
\end{equation}

\subsection{Gradient Structure}\label{subsec:gradient}

The decoupled score derivatives avoid the pathologies of Section~\ref{sec:mismatch}
(proofs in Appendix~\ref{app:proof-due-gradient}).

\begin{restatable}[Per-sample gradient of the decoupled surrogate]{proposition}{duegradient}\label{prop:due-gradient}
For one sample $(x,y,m)$, regardless of how the representation is shared,
the output-score derivatives are
\begin{equation}\label{eq:due-gradient}
\frac{\partial\Phi^{\mathrm{dec}}_\lambda}{\partial w_r}
\;=\;
p_r(x)-\mathbf{1}\{r\!=\!y\},
\qquad
\frac{\partial\Phi^{\mathrm{dec}}_\lambda}{\partial s_j}
\;=\;
\frac{\lambda}{J}\bigl(u_j(x)-t_j\bigr).
\end{equation}
Moreover, suppose $w=w(\theta,\theta_c)$ and
$s=s(\theta,\theta_e)$ are twice differentiable, where $\theta$ is a shared
trunk and $\theta_c,\theta_e$ are separate head parameters. Then
\[
\nabla^2_{\theta_c\theta_e}\Phi^{\mathrm{dec}}_\lambda=0,
\qquad
\nabla^2_\theta\Phi^{\mathrm{dec}}_\lambda
=\nabla^2_\theta\Phi_{\mathrm{cls}}
+\nabla^2_\theta\Phi_{\mathrm{exp}}.
\]
\end{restatable}

Three properties address the failure modes of Section~\ref{sec:mismatch}:
\begin{itemize}[itemsep=1pt,topsep=2pt,leftmargin=*]
\item[--] \emph{No amplification.} The gradient for expert~$j$ depends only
on its prediction $u_j$ and target $t_j$. At fixed per-expert weight, score,
and target, adding other correct experts leaves this derivative unchanged.

\item[--] \emph{No starvation.} If expert~$j$ is correct, then
$\partial\Phi/\partial s_j=(\lambda/J)(u_j-1)<0$ whenever $u_j<1$:
its logit is pushed \emph{up} even when a rival has a higher score.

\item[--] \emph{No direct coupling.} The classifier score gradient involves only the
class probabilities $p(x)$ and the label~$y$, and it does not depend on
the expert predictions or targets.
\end{itemize}

The second-order geometry confirms this decoupling
(Appendix~\ref{app:proof-due-curvature}). The class and expert terms depend
on disjoint score coordinates, so the score Hessian has a \emph{zero} mixed
block:
\begin{equation}\label{eq:due-hessian}
\nabla^2_{(w,s)}\Phi^{\mathrm{dec}}_\lambda
\;=\;
\begin{pmatrix}
\mathrm{Diag}(p)-pp^\top & 0 \\
0 & \tfrac{\lambda}{J}\,\mathrm{Diag}\bigl(u_j(1\!-\!u_j)\bigr)
\end{pmatrix}.
\end{equation}
The class block is the standard softmax covariance with
$\lambda_{\max}\le 1/2$; the expert block is diagonal with entries bounded
by $\lambda/(4J)$. In contrast, the curvature bound of additive CE scales
as $(1+|\Jset|)/2$ (Corollary~\ref{cor:mozannar-ce-geometry}), and the
mixed-block bound of A-SM grows as $\sqrt{J}$
(Appendix~\ref{app:proof-asm-curvature}). OvA also separates the score
geometry and recovers the Bayes targets at the unrestricted optimum, but its
$K$ class sigmoids need not sum to one during training.

\paragraph{Shared representations and two-stage training.}
Proposition~\ref{prop:due-gradient} also separates the head-parameter
updates when the shared trunk is learned: expert supervision enters neither
the class loss nor its head update directly. Both tasks can still shape the
common representation, since
\[
\nabla_\theta\Phi^{\mathrm{dec}}_\lambda
=(D_\theta w)^\top(p-e_y)
+\frac{\lambda}{J}(D_\theta s)^\top(u-t).
\]
The separation therefore removes the additional direct coupling of A-SM
while retaining joint representation learning
(Appendix~\ref{app:revision-sharing}).
With disjoint networks and a product feasible set, the empirical
objective separates into class CE plus $\lambda$ times mean expert BCE,
and its global minimizer set, when nonempty, is the product of the component
minimizer sets. The method can therefore be trained in two stages without
changing its objective or its global solutions, whereas published two-stage
softmax routers learn a different, predictor-dependent action representation
(Appendix~\ref{app:revision-two-stage}).

\subsection{Consistency and Surrogate-Regret Transfer}\label{subsec:excess-risk}

The conditional risk of~\eqref{eq:due} is the class CE plus separate expert
BCE terms, and its unique minimizer over the closed probability domain is
$p^\star(x)=\eta(x)$ and $u_j^\star(x)=\alpha_j(x)$ for all $j\in[J]$
(Appendix~\ref{app:dec-consistency}). We write $\mc{E}_{\mathrm{dec}}$ for
the population risk of~\eqref{eq:due} and use stars for infima.

\begin{restatable}[Consistency bound for the decoupled surrogate]{theorem}{decconsistency}\label{thm:dec-consistency}
Let $p\colon\mathcal{X}\to\Delta^K$ and
$u\colon\mathcal{X}\to[0,1]^J$ be measurable, and let
$f_{p,u}=(h_{\mathrm{dec}},r_{\mathrm{dec}})$ be the plug-in policy
of~\eqref{eq:due-routing}. For any distribution $\mathcal{D}$,
\begin{equation}\label{eq:dec-consistency}
\mc{E}_\perp(f_{p,u})-\mc{E}_\perp^\star
\le
\max\!\left\{\sqrt{2},\;\sqrt{\frac{2J}{\lambda}}\right\}
\sqrt{
\mc{E}_{\mathrm{dec}}(p,u)-\mc{E}_{\mathrm{dec}}^\star}\,.
\end{equation}
\end{restatable}

Consequently, any sequence whose surrogate excess risk tends to zero also
has L2D excess risk tending to zero, which proves Bayes consistency. The
coefficient equals $\sqrt2$ whenever $\lambda/J\ge1$ and is independent of
$J$ at any fixed positive per-expert weight, whereas at fixed total weight
$\lambda$ it eventually grows as $\sqrt J$. Both regimes concern
surrogate-regret transfer; neither prescribes an optimal training weight.

\paragraph{Comparison with existing bounds.}
The Add.\ CE reduction has transfer coefficient $\sqrt{2(J+1)}$
\citep{mao2024principledapproacheslearningdefer}, the linear Mao25 member
admits the linear transfer $(K+J)\Delta\mathcal E_{\mathrm{Mao}}$,
and our proof also yields $\sqrt2\sqrt{\Delta\mathcal E_{\mathrm{OvA}}}$ for
unit-weight OvA (Proposition~\ref{prop:revision-ova-bound});
Appendix~\ref{app:constant-comparison} compares these transfer shapes and
normalizations.

\paragraph{Consultation costs.}
For expert costs $c_j\in[0,1]$, replace $t_j$ by $1-c_j$;
the expert optimum becomes $1-\mathbb E[c_j\mid X=x]$.
For additive prices $b_j\ge0$, costs
$c_j=b_j+\mathbf1\{m_j\ne y\}$ can exceed one. With a common bound
$B\ge1+\max_j b_j$, train on $1-c_j/B$ and compare the expert estimates
with $(B-1+\max_k p_k)/B$, keeping both alternatives on the same cost
scale. Appendix~\ref{app:revision-costs} derives this rule and a
class-weighted extension.

\section{Experiments}\label{sec:experiments}

We evaluate the decoupled surrogate on four synthetic suites and four
real-data benchmarks spanning vision,
human responses, and tabular model routing. Controlled Fashion-MNIST studies
separate the effects of loss, head layout, representation sharing, and training
schedule. Backbone architectures, training budgets, and method-specific
settings are reported in the appendix.

\paragraph{Synthetic suites (Appendix~\ref{app:synthetic}).}
Each suite isolates one surrogate pathology under a distribution whose
Bayes rule is known analytically.
\emph{(i)}~Nested redundant experts, targeting Add.\ CE amplification
(Prop.~\ref{prop:additive-gradient}): at $J{=}24$, defer regret
$0.0002$ vs.\ $0.348$ for Add.\ CE (best baseline: PiCCE at $0.238$).
\emph{(ii)}~Rare specialist, targeting PiCCE starvation
(Prop.~\ref{prop:picce-starvation}): on shared-correct samples, the
decoupled surrogate selects the specialist with probability $0.994$ vs.\
$0$ for PiCCE.
\emph{(iii)}~Shared acceptability, targeting Mao25's set-mass collapse
(Prop.~\ref{prop:mao25-setmass}): the decoupled surrogate reaches defer
regret $0.0008$ vs.\ $0.154$ for Mao25.
\emph{(iv)}~Class--expert coupling, targeting A-SM
(Prop.~\ref{prop:asm-coupling}): the norm of A-SM's mixed class--expert
Hessian block grows from $0.12$ at $J{=}1$ to $0.48$ at $J{=}17$
($\approx 0.12\sqrt{J}$), while the decoupled surrogate's is
identically zero by construction (Appendix~\ref{app:synth-geometry}).

\paragraph{CIFAR-10 with synthetic experts (Appendix~\ref{app:cifar10}).}
We train a ResNet-12 end-to-end on CIFAR-10 with $J$ nested redundant
experts whose correctness is drawn from a shared latent
(Appendix~\ref{app:synth-ce}). Table~\ref{tab:cifar10-main} reports system
and classifier accuracy: the decoupled surrogate is the only method in our
comparison whose system accuracy \emph{beats} the standalone
classifier at every~$J$, while Add.\ CE, A-SM, and OvA degrade as the
redundant pool grows. A-SM's classifier accuracy collapses from $0.831$
($J{=}8$) to $0.680$ ($J{=}32$), consistent with the class--expert coupling
of Proposition~\ref{prop:asm-coupling}.

\begin{table}[ht]
\centering
\footnotesize
\renewcommand{\arraystretch}{1}
\setlength{\tabcolsep}{3pt}
\caption{CIFAR-10 test accuracy with redundant synthetic experts (mean$\pm$std, $4$ seeds).
Best in \textbf{bold}; second-best \underline{underlined}.}
\label{tab:cifar10-main}
\begin{tabular}{@{}l ccc ccc@{}}
\toprule
 & \multicolumn{3}{c}{Sys.\ Acc.\ $\uparrow$} & \multicolumn{3}{c}{Cls.\ Acc.\ $\uparrow$} \\
\cmidrule(lr){2-4}\cmidrule(lr){5-7}
Method & $J{=}8$ & $J{=}16$ & $J{=}32$ & $J{=}8$ & $J{=}16$ & $J{=}32$ \\
\midrule
\textbf{Decoupled}  & $\mathbf{.918}_{\pm.001}$ & $\mathbf{.919}_{\pm.003}$ & $\mathbf{.919}_{\pm.002}$ & $\underline{.909}_{\pm.001}$ & $\underline{.909}_{\pm.002}$ & $\mathbf{.908}_{\pm.003}$ \\
Classif.\ only & $\underline{.910}_{\pm.002}$ & $\underline{.911}_{\pm.001}$ & $\underline{.908}_{\pm.002}$ & $\mathbf{.910}_{\pm.002}$ & $\mathbf{.911}_{\pm.001}$ & $\mathbf{.908}_{\pm.002}$ \\
OvA           & $.901_{\pm.001}$ & $.890_{\pm.002}$ & $.880_{\pm.003}$ & $.875_{\pm.001}$ & $.859_{\pm.004}$ & $.844_{\pm.003}$ \\
PiCCE         & $.858_{\pm.007}$ & $.860_{\pm.011}$ & $.859_{\pm.008}$ & $.888_{\pm.008}$ & $.895_{\pm.006}$ & $\underline{.880}_{\pm.014}$ \\
A-SM          & $.860_{\pm.002}$ & $.827_{\pm.003}$ & $.783_{\pm.006}$ & $.831_{\pm.002}$ & $.772_{\pm.003}$ & $.680_{\pm.009}$ \\
Add.\ CE      & $.815_{\pm.002}$ & $.790_{\pm.013}$ & $.712_{\pm.015}$ & $.902_{\pm.004}$ & $.895_{\pm.003}$ & $.872_{\pm.009}$ \\
Mao25         & $.703_{\pm.000}$ & $.703_{\pm.000}$ & $.703_{\pm.000}$ & $.105_{\pm.024}$ & $.111_{\pm.010}$ & $.105_{\pm.025}$ \\
\bottomrule
\end{tabular}
\renewcommand{\arraystretch}{1}
\end{table}

\paragraph{CIFAR-10H with human annotators (Appendix~\ref{app:cifar10h}).}
We freeze a pretrained ResNet-12 and train linear L2D heads on
CIFAR-10H~\citep{peterson2019human}. For each image, $J$ observed human
responses are sampled without replacement; expert coordinates are annotation
slots whose worker identity can vary across images. We use three training
seeds for each of five response draws; Table~\ref{tab:cifar10h-main} reports
the fixed draw $0$, which isolates training variation, and
Appendix~\ref{app:revision-cifar10h} reports the full five-draw study. In
this table, the decoupled surrogate has the highest mean system accuracy at
each $J$, with $0.962$ throughout. Across all $45$ runs, its classifier accuracy remains
$0.886$--$0.894$ and it defers less than the other L2D methods in $44$ or
$45$ matched runs per comparison.

\begin{table}[htbp]
\centering\footnotesize
\setlength{\tabcolsep}{3pt}
\caption{CIFAR-10H test accuracies for fixed per-image annotation draw 0 (mean $\pm$ sample SD over three training seeds). System accuracy uses the routed prediction; classifier accuracy uses the model without deferral. Bold marks the highest mean in each column before rounding.}
\label{tab:cifar10h-main}
\begin{tabular}{@{}lcccccc@{}}
\toprule
 & \multicolumn{3}{c}{System accuracy $\uparrow$} & \multicolumn{3}{c}{Classifier accuracy $\uparrow$} \\
\cmidrule(lr){2-4}\cmidrule(lr){5-7}
Method & $J=5$ & $J=10$ & $J=20$ & $J=5$ & $J=10$ & $J=20$ \\
\midrule
\textbf{Decoupled} & $\mathbf{0.962}_{\pm0.002}$ & $\mathbf{0.962}_{\pm0.003}$ & $\mathbf{0.962}_{\pm0.003}$ & $\mathbf{0.891}_{\pm0.002}$ & $0.891_{\pm0.004}$ & $\mathbf{0.890}_{\pm0.001}$ \\
OvA & $0.957_{\pm0.003}$ & $0.955_{\pm0.003}$ & $0.960_{\pm0.004}$ & $0.889_{\pm0.001}$ & $0.889_{\pm0.001}$ & $0.888_{\pm0.002}$ \\
A-SM & $0.960_{\pm0.002}$ & $0.957_{\pm0.003}$ & $0.953_{\pm0.002}$ & $0.890_{\pm0.001}$ & $0.888_{\pm0.001}$ & $0.886_{\pm0.001}$ \\
Add. CE & $0.956_{\pm0.003}$ & $0.955_{\pm0.004}$ & $0.956_{\pm0.002}$ & $0.888_{\pm0.002}$ & $0.891_{\pm0.002}$ & $0.890_{\pm0.002}$ \\
PiCCE & $0.957_{\pm0.003}$ & $0.956_{\pm0.001}$ & $0.956_{\pm0.005}$ & $0.889_{\pm0.002}$ & $\mathbf{0.892}_{\pm0.002}$ & $0.889_{\pm0.001}$ \\
Classifier only & $0.890_{\pm0.001}$ & $0.889_{\pm0.001}$ & $0.889_{\pm0.002}$ & $0.890_{\pm0.001}$ & $0.889_{\pm0.001}$ & $0.889_{\pm0.002}$ \\
Mao25 (linear) & $0.948_{\pm0.002}$ & $0.948_{\pm0.005}$ & $0.946_{\pm0.005}$ & $0.060_{\pm0.054}$ & $0.072_{\pm0.036}$ & $0.082_{\pm0.045}$ \\
\bottomrule
\end{tabular}
\end{table}

\paragraph{Covertype with model experts (Appendix~\ref{app:covertype}).}
On Forest CoverType~\citep{blackard1999comparative} with $J{=}5$ heterogeneous
pretrained model experts (oracle $+9.8$~pp over the best single model), the
decoupled surrogate, \emph{the only L2D method in our comparison whose
classifier improves over the standalone classifier} ($0.929$), reaches system
accuracy $0.934$ and classifier accuracy $0.941$; every augmented-action
baseline degrades the classifier, and Mao25 collapses to $100\%$ deferral.

\paragraph{Separating loss and architecture (Appendix~\ref{app:revision-ablation}).}
On Fashion-MNIST~\citep{xiao2017fashion}, we exactly duplicate an eight-specialist panel to form
$J=16$ and $J=32$, which adds neither new information nor a better action.
A four-seed ablation crosses the decoupled and Add.\ CE losses with
shared/split, shared/single, and parameter-matched separate-trunk
architectures. The decoupled loss wins all $36$ paired loss contrasts, and
its weakest mean ($0.9544$) exceeds the strongest Add.\ CE mean ($0.9351$).
The advantage persists with separate trunks ($2.18$--$2.50$ percentage
points), which shows that neither final-layer layout nor feature sharing
explains it; sharing features adds a further gain. At aligned
initialization, its expert trunk gradients are also smaller relative to the
class gradients and nearly orthogonal to them, unlike those of Add.\ CE
(Appendix~\ref{app:revision-ablation}).

\paragraph{Schedules, weights, and consultation costs.}
With disjoint networks, the mean accuracies of one-stage and two-stage
training differ by less than $0.001$. The shared one-stage model
reaches $0.9585$--$0.9593$, versus $0.9546$--$0.9557$ for disjoint one-stage
training and $0.9036$--$0.9118$ for the predictor--rejector of
\citet{mao2024regressionmultiexpertdeferral} under the same protocol (Appendix~\ref{app:revision-schedules}). In the $96$-cell
weight sweep, the spread of mean accuracy across pool sizes is at most
$0.0016$ at fixed total $\lambda$ (Appendix~\ref{app:revision-weight}).
With consultation prices $\beta/j$, $\beta\in\{0,0.2,0.4,0.6,0.8\}$, the
decoupled surrogate has the lowest mean raw risk at every price; its
coverage rises from $0.774$ to $0.936$, monotonically in each seed, while
classifier accuracy stays within $0.915$--$0.920$
(Appendix~\ref{app:revision-price}).

\section{Conclusion}
The shared competition of augmented-action surrogates is not specific to
log-softmax and persists across the comp-sum family. The decoupled surrogate
removes it: it recovers $(\eta,\alpha)$ directly, with coordinatewise expert
gradients, zero mixed score and head Hessian blocks, and a $J$-independent
calibration constant at fixed $\lambda/J$. In our experiments it preserves
rare specialists and classifier quality, and its advantage over Add.\ CE
persists with separate trunks and with a single output projection.
With disjoint networks, the same objective and global solution set admit
two-stage training.

\section*{Acknowledgments}
This research is supported by the National Research Foundation, Singapore
under its AI Singapore Programme (AISG Award No: AISG2-PhD-2023-01-041-J) and
by A*STAR, and is part of the programme DesCartes which is supported by the
National Research Foundation, Prime Minister's Office, Singapore under its
Campus for Research Excellence and Technological Enterprise (CREATE)
programme.

\bibliographystyle{plainnat}
\bibliography{biblio}

\appendix

\newcommand{\stdp}[1]{\,{\scriptstyle(#1)}}

\section{Extended Related Work}\label{app:related-extended}

Learning-to-defer extends selective prediction and classification with a
reject option \citep{Chow_1970, Bartlett_Wegkamp_2008, Geifman_El-Yaniv_2017, cortes, cao2022generalizing, cortes2024theory} and is also closely related to
the abstention literature
\citep{Mao_Mohri_Zhong_2023, theoretically, cortes2024cardinalityaware}. The
distinguishing feature of L2D is that abstention is replaced by routing to an
external expert.

\paragraph{One-stage methods.}
One-stage methods learn prediction and deferral jointly:
\citet{madras2018predict} introduced adaptive deferral in the single-expert
setting, while \citet{mozannar2021consistent} gave the modern multiclass
formulation by casting learning to defer as prediction over an augmented action
space. Most subsequent one-stage work refines this augmented-action view rather
than departing from it. \citet{Verma2022LearningTD} compare coupled softmax and
one-vs-all surrogates for multiple experts; \citet{Cao_Mozannar_Feng_Wei_An_2023}
analyze the calibration pathology of symmetric softmax reductions and propose an
asymmetric softmax alternative; \citet{charusaie2022sample} study
complementarity and sample efficiency; \citet{Mozannar2023WhoSP} analyze
realizability and exact optimization; \citet{wei2024exploiting} incorporate
human--AI dependence into surrogate design; and \citet{liu2024mitigating} show
that even statistically consistent single-expert surrogates can underfit when
non-zero expert-query costs introduce redundant label smoothing. In the
multi-expert setting,
\citet{mao2024principledapproacheslearningdefer} derive quantitative
guarantees for comp-sum-style reductions,
\citet{mao2024realizablehconsistentbayesconsistentloss, mao2025mastering} develop consistent loss
families for learning to defer, and \citet{liu2026more} identify
redundancy-driven underfitting and propose PiCCE as a correction.

More recent extensions that continue to optimize prediction and deferral
jointly also fit naturally in the one-stage category. These include top-$k$
deferral \citep{montreuil2026why}, adversarially robust one-stage L2D
\citep{montreuil2026adversarial}, non-stationary time-series deferral
\citep{montreuil2026learning}, missing-annotation variants
\citep{nguyen2025probabilistic}, deferral to a population or to unseen experts
\citep{Tailor, strong2026identityfree}, and application-driven one-stage
systems in sequential medicine and healthcare
\citep{joshi2021learning, strong2025trustworthy}. Overall,
the one-stage literature is best viewed as a sequence of refinements and
extensions of the augmented-action formulation introduced by
\citet{mozannar2021consistent}, first in the single-expert $K+1$ setting and
then in the multi-expert $K+J$ setting.

\paragraph{Two-stage methods.}
By contrast, \emph{two-stage} methods first train a predictor and then learn a
defer rule or router on top of the frozen predictor. These include post-hoc
estimators for deferring to a single expert \citep{Narasimhan}, principled
two-stage formulations for multiple experts \citep{mao2023twostage}, and
extensions to multitask and regression settings
\citep{montreuil2024twostagelearningtodefermultitasklearning, mao2024regressionmultiexpertdeferral}. Recent extensions of the two-stage line
also include adversarial robustness guarantees
\citep{montreuil2025adversarialrobustnesstwostagelearningtodefer}, online
learning with varying experts \citep{montreuil2026online}, budgeted L2D
\citep{desalvo2025budgeted}, deferral in expert-imbalanced
regimes \citep{cortes2026optimized}, and extractive question
answering \citep{montreuil2026optimal}.

The two families therefore differ in what they optimize: one-stage methods
redesign the joint prediction--deferral surrogate, whereas two-stage methods
optimize only the routing rule after the predictor has been fixed. Our decoupled parameterization
is introduced here as a one-stage surrogate, but the same separation between
class prediction and expert-utility estimation can also be used within a
two-stage pipeline.

\paragraph{Quantitative consistency bounds.}
Classical Bayes consistency \citep{bartlett2006convexity, tewari07a, Statistical, Steinwart2007HowTC} guarantees that the conditional minimizer of the surrogate
recovers the Bayes decision. Quantitative versions turn surrogate excess risk
into target excess risk through an explicit calibration function. This line
includes multiclass and cross-entropy analyses
\citep{Awasthi_Mao_Mohri_Zhong_2022_multi, mao2024h, mao2023crossentropylossfunctionstheoretical}, smooth surrogates
\citep{mao2024universalgrowth}, regression \citep{mao2024hconsistencyregression},
multi-label learning \citep{mao2024multilabel}, abstention
\citep{theoretically, Mao_Mohri_Zhong_2023, cortes2024cardinalityaware},
imbalanced and balanced binary classification
\citep{cortes2025balancingscales, cortes2025improvedbalanced, mao2025principledbinary}, sharper
enhanced bounds \citep{mao2025enhanced}, noise-adaptive refinements
\citep{mohri2026beyond}, linear-rate smooth surrogates
\citep{mohri2026linear}, and structure-aware consistency for preference
learning \citep{mohri2026mind}; see \citet{zhong2025thesis, mao2025thesis} for
unified treatments. Theorem~\ref{thm:dec-consistency} gives such a
$\sqrt{\cdot}$-form bound for the decoupled surrogate, whose calibration
constant is independent of $J$ at fixed per-expert weight $\lambda/J$.

\paragraph{Positioning.}
Our work is closest to the one-stage literature, but differs in aim. Rather
than proposing another repair within an augmented-action parameterization, we
compare representative one-stage surrogates along two axes (population target
and local optimization geometry) and then propose a decoupled surrogate
outside the shared-action-space design. In that sense, our contribution is both
comparative and constructive: we isolate a recurrent tradeoff within existing
multi-expert surrogates and show that the decoupled surrogate avoids it while
retaining a quantitative consistency guarantee.

\section{Proof of Lemma~\ref{lem:bayes-rule}}\label{app:proof-bayes-rule}

\bayesrule*
\begin{proof}
We prove the statement directly from the true defer loss
\eqref{eq:defer-loss}. Fix an input $x\in\mathcal X$ and condition on the event
$X=x$. For a decision pair $(h,r)$, the corresponding conditional true
deferral risk is
\[
\mathcal C_\perp(h,r\mid x)
\;\coloneqq\;
\mathbb E\!\left[\ell_\perp(h,r;\,X,Y,\mathbf M)\mid X=x\right].
\]
Substituting the definition of $\ell_\perp$ gives
\begin{align}
\mathcal C_\perp(h,r\mid x)
&=
\mathbb E\!\left[
\mathbf 1\{h(x)\neq Y\}\mathbf 1\{r(x)=0\}
+\sum_{j=1}^J \mathbf 1\{M_j\neq Y\}\mathbf 1\{r(x)=j\}
\;\middle|\; X=x
\right]
\nonumber\\
&=
\mathbf 1\{r(x)=0\}\Pr\!\bigl(h(x)\neq Y\mid X=x\bigr)
+\sum_{j=1}^J \mathbf 1\{r(x)=j\}\Pr\!\bigl(M_j\neq Y\mid X=x\bigr).
\label{eq:lemma2-conditional-risk}
\end{align}
Thus, once $x$ is fixed, the learner is simply choosing one action among the
$K$ class actions and the $J$ defer actions, and each action has an explicit
conditional error probability.

Take a class label $k\in[K]$ and consider the action ``predict class $k$'',
that is, choose $h(x)=k$ and $r(x)=0$. By \eqref{eq:lemma2-conditional-risk},
its conditional risk is
\[
\mathcal C_\perp(k,0\mid x)
=
\Pr(Y\neq k\mid X=x).
\]
Since $\eta_k(x)=\Pr(Y=k\mid X=x)$, taking the complementary event gives
\[
\Pr(Y\neq k\mid X=x)=1-\Pr(Y=k\mid X=x)=1-\eta_k(x).
\]
Therefore the best classifier action at $x$ is any
\[
h^\star(x)\in\argmax_{k\in[K]}\eta_k(x),
\]
and its conditional risk is
\[
\min_{k\in[K]}\mathcal C_\perp(k,0\mid x)
=
1-\max_{k\in[K]}\eta_k(x).
\]

Now take an expert index $j\in[J]$ and consider the action ``defer to
expert~$j$'', that is, choose $r(x)=j$. Again by
\eqref{eq:lemma2-conditional-risk},
\[
\mathcal C_\perp(h,j\mid x)
=
\Pr(M_j\neq Y\mid X=x)
\qquad\text{for any }h.
\]
By definition of the expert utility
$\alpha_j(x)=\Pr(M_j=Y\mid X=x)$, we obtain
\[
\Pr(M_j\neq Y\mid X=x)=1-\alpha_j(x).
\]
Hence the best defer action at $x$ is any
\[
j^\star(x)\in\argmax_{j\in[J]}\alpha_j(x),
\]
and its conditional risk is
\[
\min_{j\in[J]}\mathcal C_\perp(h,j\mid x)
=
1-\max_{j\in[J]}\alpha_j(x).
\]

It remains to compare the best class action with the best defer action. The
Bayes decision at $x$ is the action with the smallest conditional true
deferral risk among all $K+J$ available actions, so combining the two displays
above gives
\begin{align*}
\inf_{h,r}\mathcal C_\perp(h,r\mid x)
&=
\min\!\Bigl\{
1-\max_{k\in[K]}\eta_k(x),\;
1-\max_{j\in[J]}\alpha_j(x)
\Bigr\}\\
&=
1-\max\!\Bigl\{
\max_{k\in[K]}\eta_k(x),\;
\max_{j\in[J]}\alpha_j(x)
\Bigr\}.
\end{align*}
Therefore:
\begin{itemize}[topsep=2pt,itemsep=2pt,leftmargin=*]
\item if $\max_{k\in[K]}\eta_k(x)\ge \max_{j\in[J]}\alpha_j(x)$, then the best
class action is at least as good as every defer action, so one may choose
$r^\star(x)=0$ together with any
$h^\star(x)\in\argmax_{k\in[K]}\eta_k(x)$;
\item if $\max_{k\in[K]}\eta_k(x)< \max_{j\in[J]}\alpha_j(x)$, then deferring
is strictly better, and one may choose
$r^\star(x)\in\argmax_{j\in[J]}\alpha_j(x)$.
\end{itemize}
This is exactly the rule stated in \eqref{eq:bayes-rule}.
\end{proof}

\section{The Decoupled Surrogate}\label{app:proofs-approach}

We present the decoupled surrogate first because each baseline analysis that
follows is compared against it on the same two axes. The decoupled surrogate abandons the
augmented-action reduction and instead learns the class posterior and the
expert utilities in their native forms.

\subsection{Formulation}

\paragraph{Cost-sensitive formulation.}
Before specializing to the zero-one defer loss, we record the more general
bounded-cost version of the decoupled construction. Let
$c_j(x,y)\in[0,1]$ denote the cost of routing $(x,y)$ to expert~$j$, where
the dependence on the expert prediction $m_j$ is left implicit, and let
\[
\tilde t_j(x,y)\coloneqq 1-c_j(x,y)\in[0,1]
\]
denote the corresponding utility target. The cost-sensitive decoupled
surrogate is
\[
\Phi^{\mathrm{dec,cs}}_\lambda(w,s;\,x,y)
\coloneqq
-\log p_y(x)
-\frac{\lambda}{J}\sum_{j=1}^J
\Bigl(
\tilde t_j(x,y)\log u_j(x)
\;+\;
\bigl(1-\tilde t_j(x,y)\bigr)\log(1-u_j(x))
\Bigr).
\]
Thus the class side is unchanged, while each expert head is trained with a
Bernoulli cross-entropy against a possibly soft utility target
$\tilde t_j(x,y)$.

The decoupled surrogate uses one softmax head for the classifier and one sigmoid head for each
expert. The softmax head models the categorical class posterior, while the
expert heads model the Bernoulli events ``expert $j$ is correct on this input''.

\begin{definition}[Decoupled surrogate]
\label{def:appendix-due}
Let $w(x)\in\mathbb{R}^K$ denote class logits and
$s(x)\in\mathbb{R}^J$ denote expert logits. Define
\[
p_k(x)\coloneqq \frac{\exp(w_k(x))}{\sum_{\ell=1}^K \exp(w_\ell(x))},
\qquad
u_j(x)\coloneqq \sigma(s_j(x)).
\]
For $\lambda>0$, define
\[
\Phi^{\mathrm{dec}}_\lambda(w,s;\,x,y,m)
\coloneqq
-\log p_y(x)
-\frac{\lambda}{J}\sum_{j=1}^J
\Bigl(
t_j\log u_j(x)
\;+\;
(1-t_j)\log(1-u_j(x))
\Bigr),
\]
where $t_j\coloneqq \mathbf{1}\{m_j=y\}$.
\end{definition}

Prediction compares the classifier confidence $\max_{k\in[K]} p_k(x)$ with the
largest estimated expert utility $\max_{j\in[J]} u_j(x)$ and defers whenever
the latter is larger. Definition~\ref{def:appendix-due} is the
classification-specialized case of the bounded-cost formulation above, since
in the present paper
\[
c_j(x,y)=\mathbf 1\{m_j\neq y\},
\qquad
\tilde t_j(x,y)=\mathbf 1\{m_j=y\},
\]
so the expert-side target reduces exactly to binary correctness.

Two structural features are already visible from the definition. First, the
class side remains a categorical distribution throughout training because
$p(x)\in\Delta^K$ for every parameter value. Second, the expert side is
coordinatewise: each $u_j(x)$ is a bounded Bernoulli probability attached to
one expert only. The rest of this section formalizes the consequences of these
two features.

\subsection{Per-Sample Gradient Structure}
\label{app:proof-due-gradient}

\duegradient*
\begin{proof}
The samplewise decoupled loss is the sum of one multiclass cross-entropy term and $J$
separate Bernoulli cross-entropy terms, so we can differentiate the two
pieces separately.

\paragraph{Classifier head.}
The classifier contribution is
\[
\Phi_{\mathrm{cls}}(w;y)
\;=\;
-\log p_y
\;=\;
-w_y+\log\!\sum_{\ell=1}^K e^{w_\ell}.
\]
Differentiate with respect to $w_r$:
\[
\frac{\partial \Phi_{\mathrm{cls}}}{\partial w_r}
\;=\;
-\mathbf{1}\{r=y\}
\;+\;
\frac{e^{w_r}}{\sum_{\ell=1}^K e^{w_\ell}}
\;=\;
p_r-\mathbf{1}\{r=y\},
\]
which is the usual multiclass cross-entropy gradient.

\paragraph{Expert heads.}
For expert $j$, write $t_j=\mathbf{1}\{m_j=y\}$. Its contribution is
\[
\Phi_j(s_j;t_j)
\;=\;
-t_j\log u_j -(1-t_j)\log(1-u_j),
\qquad
u_j=\sigma(s_j).
\]
Differentiating with respect to $s_j$ and applying the chain rule gives
\[
\frac{\partial \Phi_j}{\partial s_j}
\;=\;
\left(
-\frac{t_j}{u_j}+\frac{1-t_j}{1-u_j}
\right)
\frac{du_j}{ds_j}.
\]
Since $\frac{du_j}{ds_j}=u_j(1-u_j)$, we obtain
\[
\frac{\partial \Phi_j}{\partial s_j}
\;=\;
\left(
-\frac{t_j}{u_j}+\frac{1-t_j}{1-u_j}
\right)u_j(1-u_j)
\;=\;
u_j-t_j.
\]
Multiplying by the prefactor $\lambda/J$ gives
\[
\frac{\partial \Phi^{\mathrm{dec}}_\lambda}{\partial s_j}
\;=\;
\frac{\lambda}{J}(u_j-t_j).
\]

For the parameter-level statement, write the loss as
$\Phi_{\mathrm{cls}}(w(\theta,\theta_c))+\Phi_{\mathrm{exp}}(s(\theta,\theta_e))$,
where $\Phi_{\mathrm{cls}}=-\log p_y$ and $\Phi_{\mathrm{exp}}$ is the
weighted expert term of~\eqref{eq:due}. Since the first summand
does not depend on $\theta_e$ and the second does not depend on
$\theta_c$, their mixed head derivative $\nabla^2_{\theta_c\theta_e}$ is
zero, and differentiating the sum twice in $\theta$ gives the stated sum of
trunk Hessians. Appendix~\ref{app:revision-sharing} displays all parameter
blocks.

The classifier residuals and the expert residuals therefore live on
disjoint coordinates:
the class gradient depends only on $(w,y)$, and the $j$-th expert gradient
depends only on $(s_j,t_j)$.
\end{proof}

This proposition isolates the first-order reason the decoupled surrogate behaves differently from
the augmented-action surrogates. A sample contributes one classifier residual
and $J$ expert residuals, but these residuals do not compete through a shared
normalization. Correct experts are reinforced independently of one another, and
expert-side errors cannot directly distort the class-side update.

\subsection{Per-Sample Curvature Structure}
\label{app:proof-due-curvature}

We now derive the Hessian explicitly, which shows both that the curvature is
bounded and that the class and expert directions remain decoupled at second
order.

Start with the classifier block. From the
softmax Jacobian,
\[
\frac{\partial p_r}{\partial w_\ell}
\;=\;
p_r(\mathbf{1}\{r=\ell\}-p_\ell),
\]
so the classifier Hessian is
\[
\nabla_w^2(-\log p_y)
\;=\;
\mathrm{Diag}(p)-pp^\top.
\]
This is the usual softmax covariance matrix, which captures competition only
among the class coordinates themselves.
For the expert side, differentiate the gradient
$\frac{\lambda}{J}(u_j-t_j)$ once more:
\[
\frac{\partial^2 \Phi^{\mathrm{dec}}_\lambda}{\partial s_j^2}
\;=\;
\frac{\lambda}{J}\frac{du_j}{ds_j}
\;=\;
\frac{\lambda}{J}u_j(1-u_j),
\]
and for $j\neq \ell$,
\[
\frac{\partial^2 \Phi^{\mathrm{dec}}_\lambda}{\partial s_j\,\partial s_\ell}
\;=\;0,
\]
because $\Phi_j$ depends only on $s_j$.

Finally, the class term depends only on $w$, and the expert terms depend only
on $s$. Therefore every mixed class--expert derivative is zero:
\[
\frac{\partial^2 \Phi^{\mathrm{dec}}_\lambda}{\partial w_r\,\partial s_j}
\;=\;0
\qquad
\text{for all }r\in[K],\;j\in[J].
\]
Hence the full Hessian is block diagonal:
\[
\nabla^2_{(w,s)}\Phi^{\mathrm{dec}}_\lambda
\;=\;
\begin{pmatrix}
\mathrm{Diag}(p)-pp^\top & 0 \\
0 & \tfrac{\lambda}{J}\,\mathrm{Diag}\bigl(u_j(1-u_j)\bigr)
\end{pmatrix}.
\]

The block-diagonal form makes the spectral analysis immediate: the eigenvalues
of a block-diagonal matrix are the union of the eigenvalues of its diagonal
blocks, so the largest eigenvalue of the full Hessian is the larger of the
largest class-block eigenvalue and the largest expert-block eigenvalue.

For the class block, the softmax covariance matrix $\mathrm{Diag}(p)-pp^\top$
has largest eigenvalue at most $1/2$ \citep{Foundations}. This is the
intrinsic curvature scale of ordinary multiclass cross-entropy, and the
decoupled surrogate leaves it unchanged.

For the expert block, the matrix is diagonal, so its eigenvalues are exactly
the diagonal entries
\[
\frac{\lambda}{J}u_j(1-u_j),\qquad j\in[J].
\]
Since $0\le u_j(1-u_j)\le 1/4$ for every $j$, each expert-side eigenvalue is
bounded by $\lambda/(4J)$. Therefore
\[
\lambda_{\max}\!\bigl(\nabla^2_{(w,s)}\Phi^{\mathrm{dec}}_\lambda\bigr)
\;\le\;
\max\!\left\{\frac12,\;\frac{\lambda}{4J}\right\}.
\]

The class-side curvature is therefore bounded by the same constant as
ordinary multiclass cross-entropy, while the expert-side bound $\lambda/(4J)$
does not depend on the number of experts at fixed per-expert weight
$\lambda/J$ and decreases as $J$ grows when the total expert weight $\lambda$
is held fixed. Adding experts thus creates neither the multiplicity factor of
additive CE nor the mixed class--expert block of A-SM; the absence of
winner-take-all suppression is a first-order property, given by
Proposition~\ref{prop:due-gradient}.

The class and expert losses have no mixed block, either in the output scores
or between the two sets of head parameters. With a learned shared trunk, both
terms still contribute to the trunk through that common representation, as
described in Proposition~\ref{prop:due-gradient}.

\subsection{The Role of \texorpdfstring{$\lambda/J$}{lambda/J} and How to Choose \texorpdfstring{$\lambda$}{lambda}}
\label{app:due-lambda-choice}

The coefficient $\lambda$ multiplies the mean expert BCE, so it is the
\emph{total} expert weight, while each expert's score gradient and
curvature scale with the per-expert weight $\lambda/J$. Fixing $\lambda$
therefore fixes the total weight while reducing the individual weights as
$J$ grows; fixing $\lambda/J$ instead scales $\lambda$ linearly with $J$.
Both choices preserve the unrestricted probability targets, since every
$\lambda>0$ has the same conditional minimizer.

Theorem~\ref{thm:dec-consistency} has coefficient
$\max\{\sqrt2,\sqrt{2J/\lambda}\}$, which is independent of $J$ at any
fixed positive per-expert weight and equals $\sqrt2$ when
$\lambda/J\ge1$. This coefficient describes the transfer from a given
surrogate excess; it does not control the surrogate excess attained after
training. With a shared trunk, the total expert weight also controls the
balance of representation learning between the two tasks, so the theorem
gives no universal recommendation to increase $\lambda$ with $J$.

Appendix~\ref{app:revision-weight} compares fixed-total, fixed-per-expert,
and validation-selected policies on exactly duplicated Fashion-MNIST
expert panels. All three preserve system accuracy in that controlled
setting, including fixed-total weights across a 128-fold range.
In every policy, weights and checkpoints are selected on the validation
deferral loss, so these results describe the robustness of the tested
protocol rather than an invariance that holds for every optimizer or shared
trunk.

\subsection{Consistency Bound}
\label{app:dec-consistency}

We now prove the main statistical guarantee of the decoupled surrogate.

\decconsistency*

\begin{proof}
Throughout this proof we abbreviate the per-expert weight as
$\beta\coloneqq\lambda/J$; this is a local shorthand for the fraction
appearing in~\eqref{eq:due}, not a separate hyperparameter.
Fix $x$ and define
\[
\delta_\perp(p,u\mid x)
\coloneqq
\mathcal{C}_\perp(f_{p,u}\mid x)-\mathcal{C}_\perp^\star(x)
\]
for the conditional L2D regret and
\[
\delta_{\mathrm{dec}}(p,u\mid x)
\coloneqq
\mathcal{C}_{\mathrm{dec}}(p,u\mid x)-\mathcal{C}_{\mathrm{dec}}^\star(x)
\]
for the conditional surrogate excess, where $\mathcal{C}_\perp^\star(x)$ and
$\mathcal{C}_{\mathrm{dec}}^\star(x)$ denote the corresponding conditional
infima.

Conditioning on $X=x$ in the samplewise surrogate gives
\begin{align}
\mathcal{C}_{\mathrm{dec}}(p,u\mid x)
&=
-\sum_{k=1}^K \eta_k(x)\log p_k(x)
\nonumber\\
&\quad
-\beta\sum_{j=1}^J
\Bigl(
\alpha_j(x)\log u_j(x)
+(1-\alpha_j(x))\log(1-u_j(x))
\Bigr).
\label{eq:due-cond-risk}
\end{align}
The class term depends only on $p(x)$, and the
$j$-th expert term depends only on $u_j(x)$. We write
$\mathcal E_{\mathrm{dec}}(p,u)\coloneqq\mathbb E_X[\mathcal C_{\mathrm{dec}}(p,u\mid X)]$
for the population surrogate risk and $\mathcal E^\star_{\mathrm{dec}}$ for
its infimum over measurable $p:\mathcal X\to\Delta^K$ and
$u:\mathcal X\to[0,1]^J$; since $(\eta,\alpha)$ is measurable and attains
every conditional infimum (shown below),
$\mathcal E^\star_{\mathrm{dec}}=\mathbb E_X[\mathcal C^\star_{\mathrm{dec}}(X)]$.

With the conventions $0\log0=0$ and $-\log0=+\infty$, for any
$q\in\Delta^K$,
\[
-\sum_{k=1}^K \eta_k(x)\log q_k
=
H(\eta(x))
+\mathrm{KL}\bigl(\eta(x)\,\|\,q\bigr),
\]
where $H(\eta(x))=-\sum_k \eta_k(x)\log \eta_k(x)$. For any $v\in[0,1]$,
\[
-\alpha_j(x)\log v-(1-\alpha_j(x))\log(1-v)
=
H_{\mathrm B}(\alpha_j(x))
+\mathrm{KL}\!\Bigl(
\mathrm{Bern}(\alpha_j(x))
\big\|
\mathrm{Bern}(v)
\Bigr),
\]
where
$H_{\mathrm B}(t)=-t\log t-(1-t)\log(1-t)$
is the Bernoulli entropy. Substituting these identities into
\eqref{eq:due-cond-risk} yields
\begin{align}
\mathcal{C}_{\mathrm{dec}}(p,u\mid x)
&=
H(\eta(x))
+\beta\sum_{j=1}^J H_{\mathrm B}(\alpha_j(x))
\nonumber\\
&\quad
+\mathrm{KL}\bigl(\eta(x)\,\|\,p(x)\bigr)
+\beta\sum_{j=1}^J
\mathrm{KL}\!\Bigl(
\mathrm{Bern}(\alpha_j(x))
\big\|
\mathrm{Bern}(u_j(x))
\Bigr).
\label{eq:due-cond-decomp}
\end{align}
The first line is independent of $(p,u)$, so it is exactly the conditional
minimum. Therefore
\begin{equation}\label{eq:due-cond-kl}
\delta_{\mathrm{dec}}(p,u\mid x)
\;=\;
\mathrm{KL}\bigl(\eta(x)\,\|\,p(x)\bigr)
\;+\;
\beta\sum_{j=1}^J
\mathrm{KL}\!\Bigl(
\mathrm{Bern}(\alpha_j(x))
\big\|
\mathrm{Bern}(u_j(x))
\Bigr).
\end{equation}
This identity shows that the conditional surrogate excess is exactly the sum of one
class estimation error and $J$ expert estimation errors weighted by $\beta$,
each measured in the KL divergence of its own subproblem.
Since every KL term is nonnegative and vanishes only when its two arguments
coincide, the conditional surrogate risk is minimized exactly when
$p(x)=\eta(x)$ and $u_j(x)=\alpha_j(x)$ for every $j$.

We now translate this conditional estimation error into conditional L2D regret.
Define
\[
\varepsilon_{\mathrm{cls}}(x)\coloneqq\|p(x)-\eta(x)\|_\infty,
\qquad
\varepsilon_{\mathrm{exp}}(x)\coloneqq\|u(x)-\alpha(x)\|_\infty.
\]
Also let
$y^\star\in\argmax_y\eta_y(x)$,
$\hat y\in\argmax_y p_y(x)$,
$j^\star\in\argmax_j\alpha_j(x)$,
$\hat j\in\argmax_j u_j(x)$,
$a^\star=\eta_{y^\star}(x)$,
and
$b^\star=\alpha_{j^\star}(x)$.

If both the plug-in rule and the Bayes rule predict, then
\[
\delta_\perp(p,u\mid x)=\eta_{y^\star}(x)-\eta_{\hat y}(x).
\]
Since $\hat y$ maximizes $p$, we have $p_{\hat y}\ge p_{y^\star}$, and hence
\[
\eta_{y^\star}-\eta_{\hat y}
=
(\eta_{y^\star}-p_{y^\star})
+(p_{y^\star}-p_{\hat y})
+(p_{\hat y}-\eta_{\hat y})
\le 2\varepsilon_{\mathrm{cls}}(x).
\]

If the plug-in rule predicts while the Bayes rule defers, then
$b^\star>a^\star$ and
$p_{\hat y}\ge u_{\hat j}\ge u_{j^\star}$. Therefore
\[
\delta_\perp(p,u\mid x)
=
\alpha_{j^\star}-\eta_{\hat y}
\le
(\alpha_{j^\star}-u_{j^\star})
+(u_{j^\star}-p_{\hat y})
+(p_{\hat y}-\eta_{\hat y})
\le
\varepsilon_{\mathrm{exp}}(x)+\varepsilon_{\mathrm{cls}}(x).
\]

If both rules defer, then
\[
\delta_\perp(p,u\mid x)=\alpha_{j^\star}(x)-\alpha_{\hat j}(x).
\]
Since $\hat j$ maximizes $u$, the same argument gives
\[
\alpha_{j^\star}-\alpha_{\hat j}\le 2\varepsilon_{\mathrm{exp}}(x).
\]

If the plug-in rule defers while the Bayes rule predicts, then
$a^\star\ge b^\star$ and
$u_{\hat j}>p_{\hat y}\ge p_{y^\star}$. Hence
\[
\delta_\perp(p,u\mid x)
=
\eta_{y^\star}-\alpha_{\hat j}
\le
(\eta_{y^\star}-p_{y^\star})
+(p_{y^\star}-u_{\hat j})
+(u_{\hat j}-\alpha_{\hat j})
\le
\varepsilon_{\mathrm{cls}}(x)+\varepsilon_{\mathrm{exp}}(x).
\]

All four cases are therefore covered by
\begin{equation}\label{eq:due-regret-via-errors}
\delta_\perp(p,u\mid x)
\le
2\max\{\varepsilon_{\mathrm{cls}}(x),\varepsilon_{\mathrm{exp}}(x)\}.
\end{equation}

We next bound these two sup-norm errors by the conditional surrogate excess.
Since $p(x)$ and $\eta(x)$ both sum to one, the positive and negative parts of
$p(x)-\eta(x)$ each have total mass $\tfrac12\|p(x)-\eta(x)\|_1$, so no
coordinate can exceed this value, and Pinsker's inequality for categorical
distributions gives
\[
\varepsilon_{\mathrm{cls}}(x)
\le \tfrac12\|p(x)-\eta(x)\|_1
\le \sqrt{\tfrac12\,\mathrm{KL}(\eta(x)\|p(x))}.
\]
For each expert coordinate, Pinsker's inequality for Bernoulli distributions
gives
\[
|u_j(x)-\alpha_j(x)|
\le
\sqrt{\frac12\,
\mathrm{KL}\!\Bigl(
\mathrm{Bern}(\alpha_j(x))
\big\|
\mathrm{Bern}(u_j(x))
\Bigr)}.
\]
Taking the maximum over $j$ and upper bounding the maximum by the Euclidean
norm yields
\[
\varepsilon_{\mathrm{exp}}(x)
=
\max_j |u_j(x)-\alpha_j(x)|
\le
\sqrt{\frac12\sum_{j=1}^J
\mathrm{KL}\!\Bigl(
\mathrm{Bern}(\alpha_j(x))
\big\|
\mathrm{Bern}(u_j(x))
\Bigr)}.
\]
Combining these inequalities with \eqref{eq:due-cond-kl} gives
\[
\varepsilon_{\mathrm{cls}}(x)
\le
\sqrt{\frac{\delta_{\mathrm{dec}}(p,u\mid x)}{2}},
\qquad
\varepsilon_{\mathrm{exp}}(x)
\le
\sqrt{\frac{\delta_{\mathrm{dec}}(p,u\mid x)}{2\beta}}.
\]
Substituting into \eqref{eq:due-regret-via-errors},
\begin{align*}
\delta_\perp(p,u\mid x)
&\le
2\max\{\varepsilon_{\mathrm{cls}}(x),\,\varepsilon_{\mathrm{exp}}(x)\}
=
\max\{2\varepsilon_{\mathrm{cls}}(x),\,2\varepsilon_{\mathrm{exp}}(x)\}\\
&\le
\max\!\Bigl\{
2\sqrt{\tfrac{\delta_{\mathrm{dec}}(p,u\mid x)}{2}},\;
2\sqrt{\tfrac{\delta_{\mathrm{dec}}(p,u\mid x)}{2\beta}}
\Bigr\}
=
\max\!\left\{\sqrt{2},\;\sqrt{\frac{2}{\beta}}\right\}
\sqrt{\delta_{\mathrm{dec}}(p,u\mid x)}.
\end{align*}
Defining $\Gamma(t)\coloneqq\max\!\bigl\{\sqrt{2},\sqrt{2/\beta}\bigr\}\sqrt{t}$,
this is the pointwise calibration inequality
\begin{equation}\label{eq:pointwise-calib}
\delta_\perp(p,u\mid x)
\le
\Gamma\!\bigl(\delta_{\mathrm{dec}}(p,u\mid x)\bigr).
\end{equation}

Taking expectations over $X$ and using Jensen's inequality, with the
concavity of $\Gamma$, gives
\[
\mc{E}_\perp(f_{p,u})-\mc{E}_\perp^\star
=
\mathbb{E}_X[\delta_\perp(p,u\mid X)]
\le
\mathbb{E}_X[\Gamma(\delta_{\mathrm{dec}}(p,u\mid X))]
\le
\Gamma\!\left(\mathbb{E}_X[\delta_{\mathrm{dec}}(p,u\mid X)]\right).
\]
Since
\[
\mathbb{E}_X[\delta_{\mathrm{dec}}(p,u\mid X)]
=
\mc{E}_{\mathrm{dec}}(p,u)-\mc{E}_{\mathrm{dec}}^\star,
\]
the claimed bound follows from the definition of $\Gamma$ and
$\beta=\lambda/J$.
\end{proof}

The argument thus has three steps: the pointwise KL decomposition
gives exact conditional alignment, the calibration inequality turns this
alignment into regret control, and Jensen's inequality lifts the pointwise
statement to population excess risk.

\subsection{Bayes-Consistency Corollary}
\label{app:proof-fisher}
\label{app:proof-excess-risk}

\begin{corollary}[Bayes consistency of the decoupled surrogate]\label{thm:fisher}
Every exact population minimizer of the decoupled surrogate satisfies
$p=\eta$ and $u_j=\alpha_j$ almost surely for all $j\in[J]$, and the induced
plug-in rule is Bayes-optimal. More generally, if
$\mc{E}_{\mathrm{dec}}(p_n,u_n)\to\mc{E}_{\mathrm{dec}}^\star$, then
$\mc{E}_\perp(f_{p_n,u_n})\to\mc{E}_\perp^\star$.
\end{corollary}

\begin{proof}
Since the unrestricted class contains $(\eta,\alpha)$, any exact minimizer
$(p,u)$ achieves the population infimum, and hence its conditional excess
satisfies $\delta_{\mathrm{dec}}(p,u\mid x)=0$ almost surely. By the KL decomposition
\eqref{eq:due-cond-kl}, this forces $p=\eta$ and $u_j=\alpha_j$ a.s.\ for all
$j$, so the plug-in rule is Bayes-optimal by Lemma~\ref{lem:bayes-rule}. The
approximate statement follows from Theorem~\ref{thm:dec-consistency}.
\end{proof}

\section{Additive Cross-Entropy \texorpdfstring{\citep{mozannar2021consistent}}{}}\label{app:proofs-mismatch}

This section collects the results on the additive cross-entropy baseline in
one place. Additive CE is misaligned on both axes at once: it changes the
population target, and it rescales the local geometry by the realized number
of correct experts.

We follow the order of the main paper's two-axis analysis. We first restate
the surrogate and derive what it estimates at the conditional optimum, and we
then derive its samplewise gradient and Hessian, making the multiplicity
factor $1+|\Jset|$ explicit. A comparison with the decoupled surrogate and a
worked example close the section and show how the two effects arise on a
concrete sample.

\subsection{Formulation}

The standard multi-expert cross-entropy baseline treats the $K$ classifier
actions and the $J$ defer actions as one augmented action space and rewards the
true class together with every correct expert.

\begin{definition}[Additive augmented cross-entropy, \citet{mozannar2021consistent}]
\label{def:appendix-ce}
Let $a(x)\in\mathbb{R}^{K+J}$ be an augmented score vector and define
\[
q_i(x)\coloneqq \frac{\exp(a_i(x))}{\sum_{r=1}^{K+J}\exp(a_r(x))},
\qquad i\in[K+J].
\]
The additive augmented cross-entropy is
\[
\Phi^\mathrm{CE}(a;x,y,m)
\coloneqq
-\log q_y(x)-\sum_{j:\,m_j=y}\log q_{K+j}(x).
\]
\end{definition}

Prediction uses the augmented argmax of $a(x)$, and the additive reward
structure analyzed below is consistent with the redundancy-driven underfitting
reported by \citet{liu2026more}.

\subsection{Conditional Target Distortion}
\label{app:proof-additive-optimum}

\begin{restatable}[Conditional minimizer of additive CE]{proposition}{additiveoptimum}
\label{prop:additive-optimum}
The conditional risk of additive CE at~$x$ is uniquely minimized over
$q\in\Delta^{K+J}$ at
\begin{equation}\label{eq:additive-optimum}
q_y^\star(x)=\frac{\eta_y(x)}{1+U(x)},
\qquad
q_{K+j}^\star(x)=\frac{\alpha_j(x)}{1+U(x)},
\qquad y\in[K],\; j\in[J].
\end{equation}
\end{restatable}

\begin{proof}
We begin by writing the conditional risk as a cross-entropy against a
normalized target. Once this normalization is explicit, the optimizer follows
from the standard uniqueness of the cross-entropy minimizer over a simplex.

The conditional surrogate risk at a fixed $x$ is
\[
\mc{C}_{\mathrm{add}}(q \mid x)
=
-\sum_{y=1}^K \eta_y(x)\log q_y
-\sum_{j=1}^J \alpha_j(x)\log q_{K+j},
\]
subject to $q\in\Delta^{K+J}$. Since $\sum_y\eta_y(x)=1$, the total target
mass is $1+U(x)$ where $U(x)=\sum_j\alpha_j(x)$. Define the normalized target
\[
\tilde{t}_y(x)=\frac{\eta_y(x)}{1+U(x)},
\qquad
\tilde{t}_{K+j}(x)=\frac{\alpha_j(x)}{1+U(x)}.
\]
Then $\mc{C}_{\mathrm{add}}(q\mid x) = (1+U(x))H(\tilde{t}(x),q)$, where $H$
denotes cross-entropy. Cross-entropy over a simplex is uniquely minimized at
$q=\tilde{t}$, which gives the stated formula.
\end{proof}

\subsection{Gradient Amplification and Curvature}
\label{app:proof-additive-gradient}

\additivegradient*
\begin{proof}
For $\tau\ne1$, differentiating the base loss with respect to $q_i$
gives $-q_i^{\tau-2}$; this also holds at $\tau=1$.
Since $\nabla_a q_i=q_i(e_i-q)$,
\[
\nabla_a\ell_\tau(a,i)=q_i^{\tau-1}(q-e_i).
\]
Differentiating this product, using $D_aq=C(q)$ and
$\nabla_a q_i^{\tau-1}=-(\tau-1)q_i^{\tau-1}(q-e_i)$, gives
\[
\nabla_a^2\ell_\tau(a,i)
=q_i^{\tau-1}[C(q)-(\tau-1)(q-e_i)(q-e_i)^\top].
\]
Summing over $\mathcal T$ and factoring out
$\sum_{i\in\mathcal T}q_i^{\tau-1}=N\bar q_\tau$ proves both identities.
If $k\notin\mathcal T$, $(\bar e_\tau)_k=0$, so its derivative is positive
at finite scores.
At uniform scores $q_i=1/d$, this derivative equals $Nd^{-\tau}$.
For a unit vector $z$ supported on omitted coordinates with
$\sum_kz_k=0$, $q^\top z=0$ and $e_i^\top z=0$ for every active $i$.
Consequently $C(q)z=z/d$ and $\bar R_\tau z=0$, giving
$\nabla_a^2L_\tau z=Nd^{-\tau}z$. None of these steps uses the specific
form of $\mathcal T$, so the identities hold for any target set.
The true-class-only loss has $N=1$ and the same equations on these
coordinates, proving the stated comparison at fixed $d$ and scores.
\end{proof}

\paragraph{General target weighting, without a comp-sum assumption.}
For any multiclass base loss $\ell:\mathbb R^d\times[d]\to\mathbb R$
and finite, score-independent weights $v_i\ge0$ with $V=\sum_i v_i>0$,
set $\pi_i=v_i/V$. Then
\begin{equation}\label{eq:general-additive-reduction}
L_{\ell,v}(a)=\sum_i v_i\ell(a,i)
=V\,\mathbb E_{I\sim\pi}[\ell(a,I)].
\end{equation}
If $\ell$ is twice differentiable, its gradient and Hessian are $V$ times
the corresponding $\pi$-averaged derivatives. This factorization alone
need not imply monotone growth of derivative norms.
For the additive reduction, $v=e_y+\sum_{j\in\Jset}e_{K+j}$,
so $V=N$ and $\pi$ is uniform on $\mathcal T$.
Conditioning on $X=x$ gives
\[
\mathcal C_\ell(a\mid x)
=\sum_{k=1}^K\eta_k(x)\ell(a,k)
+\sum_{j=1}^J\alpha_j(x)\ell(a,K+j)
=(1+U(x))\mathbb E_{I\sim\bar\pi(x)}[\ell(a,I)],
\]
where $\bar\pi=(\eta,\alpha)/(1+U)$.
For bounded costs in Mao et al.'s general construction, replace
$\alpha_j(x)$ by $1-\mathbb E[c_j\mid X=x]$.
These identities follow from finite sums and conditional expectation;
no properness, softmax representation, or expert independence is required.
If $V=0$, the loss and its derivatives vanish.

\mozannarcegeometry*
\begin{proof}
For $V>0$, dividing each weight by $V$ gives
$L_v=V\operatorname{CE}(v/V,q)$ exactly. For the derivatives, write
$L_v=-v^\top a+V\log\sum_r e^{a_r}$. Differentiating gives
$\nabla L_v=Vq-v$ and $\nabla^2L_v=V(\mathrm{Diag}(q)-qq^\top)$.
For any unit vector $z$, the covariance quadratic form is the variance of
$z_A$ for $A\sim q$. The range of $z_A$ is at most $\sqrt2$, since
$|z_i-z_j|\le\sqrt{2(z_i^2+z_j^2)}\le\sqrt2$, and the variance of a random
variable supported on an interval of length $\sqrt2$ is at most
$(\sqrt2)^2/4=1/2$. This proves the $V/2$ eigenvalue bound, including the
case $V=0$. For additive CE we also give a direct calculation that keeps
track of where the factor $1+|\mathcal{J}|$ enters; since the loss is a
scaled log-sum-exp plus linear terms, we derive the gradient first and then
the Hessian.

Write $\mathcal{J}=\mathcal{J}(y,m)$. The loss can be rewritten as
\[
\Phi^\mathrm{CE}(a;x,y,m)
=
(1+|\mathcal{J}|)\log\sum_{r=1}^{K+J}e^{a_r}
-a_y-\sum_{j\in \mathcal{J}}a_{K+j}.
\]

\paragraph{Part (a): gradient.}
Since $\partial\log\sum_r e^{a_r}/\partial a_i = q_i$, differentiating gives
\[
\frac{\partial\Phi^\mathrm{CE}}{\partial a_i}
=
(1+|\mathcal{J}|)\,q_i
-\mathbf{1}\{i=y\}
-\mathbf{1}\{i=K+j,\; j\in \mathcal{J}\}.
\]
The total positive mass is the sum of the indicator coefficients:
$1+|\mathcal{J}|$.

\paragraph{Part (b): curvature.}
We now differentiate once more. Recall that
\[
q_i=\frac{e^{a_i}}{\sum_{r=1}^{K+J}e^{a_r}}.
\]
Differentiating $q_i$ with respect to $a_r$ gives the standard softmax
Jacobian
\[
\frac{\partial q_i}{\partial a_r}
=
q_i(\mathbf{1}\{i=r\}-q_r).
\]
Equivalently,
\[
\frac{\partial q_i}{\partial a_i}=q_i(1-q_i),
\qquad
\frac{\partial q_i}{\partial a_r}=-q_iq_r\quad (r\neq i).
\]
Therefore the Hessian of the log-sum-exp term
$\log\sum_r e^{a_r}$ is the matrix whose $(i,r)$ entry is
$q_i(\mathbf{1}\{i=r\}-q_r)$, namely
\[
\nabla_a^2 \log\sum_{r=1}^{K+J}e^{a_r}
=
\mathrm{Diag}(q)-qq^\top.
\]
This matrix is the covariance matrix of a categorical random variable with
probabilities $q$, which is why it is positive semidefinite. Since the terms
$-a_y$ and $-\sum_{j\in\mathcal{J}}a_{K+j}$ are linear in $a$, they disappear
under second differentiation. Hence
\[
\nabla_a^2\Phi^\mathrm{CE}
=(1+|\mathcal{J}|)(\mathrm{Diag}(q)-qq^\top).
\]

For completeness, we also verify the eigenvalue bound row by row with
Gershgorin's circle theorem, writing $M\coloneqq \mathrm{Diag}(q)-qq^\top$.

For row $i$, the diagonal entry is
\[
M_{ii}=q_i-q_i^2=q_i(1-q_i),
\]
and every off-diagonal entry is
\[
M_{ij}=-q_iq_j,\qquad j\neq i.
\]
The Gershgorin radius of row $i$ is the sum of the absolute values of the
off-diagonal entries:
\[
R_i=\sum_{j\neq i}|M_{ij}|
=\sum_{j\neq i} q_iq_j
=q_i\sum_{j\neq i}q_j
=q_i(1-q_i),
\]
because $\sum_j q_j=1$. Gershgorin's theorem therefore says that every
eigenvalue $\lambda$ of $M$ lies in at least one interval
\[
[M_{ii}-R_i,\;M_{ii}+R_i]
=
[q_i(1-q_i)-q_i(1-q_i),\;q_i(1-q_i)+q_i(1-q_i)]
=
[0,\;2q_i(1-q_i)].
\]
Since the scalar function $u\mapsto u(1-u)$ is maximized at $u=1/2$, we have
$q_i(1-q_i)\le 1/4$ for every $i$. Hence every eigenvalue of $M$ lies in
$[0,1/2]$, and in particular
\[
\lambda_{\max}(M)\le \frac{1}{2}.
\]
Finally, multiplying by the prefactor $1+|\mathcal{J}|$ yields
\[
\lambda_{\max}(\nabla_a^2\Phi^\mathrm{CE})
\le
\frac{1+|\mathcal{J}|}{2}.
\]
\end{proof}

\subsection{Head-to-Head with the Decoupled Surrogate}

Additive CE and the decoupled surrogate differ on both analytical axes, and
the contrast is useful because the two losses address the same decision
problem in different ways. Additive CE folds classes and experts into one
augmented softmax, whereas the decoupled surrogate keeps the class posterior
and the expert utilities as separate objects throughout training and
prediction. The next two paragraphs make the consequences of this difference
precise.

\paragraph{Axis~(i): statistical target.}
Additive CE does not recover $(\eta,\alpha)$ directly. By
Proposition~\ref{prop:additive-optimum}, it learns the normalized object
\[
\left(
\frac{\eta_1}{1+U},\dots,\frac{\eta_K}{1+U},
\frac{\alpha_1}{1+U},\dots,\frac{\alpha_J}{1+U}
\right),
\]
so the target itself moves with the total expert overlap $U(x)$, which is the
first misalignment. If one duplicates an already-correct expert several times,
the underlying Bayes decision problem does not change, but the additive-CE
target does: every class and expert coordinate is shrunk by the larger
denominator $1+U(x)$. In other words, additive CE does not estimate the native
decision quantities on their original scale.

The decoupled surrogate has no such dependence, since its conditional
minimizer is exactly
$p^\star(x)=\eta(x)$ and $u_j^\star(x)=\alpha_j(x)$ for every expert
(Appendix~\ref{app:dec-consistency}). Duplicating one expert leaves the target for
every other coordinate unchanged. This is the sense in which the decoupled surrogate is
target-aligned: the statistical object being estimated is the Bayes object
itself, not a renormalized proxy.

\paragraph{Axis~(ii): local geometry.}
Additive CE allocates one positive unit of gradient mass to the true class and
one additional unit to every correct expert, which yields the multiplicative
factor $1+|\Jset|$ in both the gradient and the Hessian. In practice, samples
on which many experts happen to be correct are up-weighted in both
first-order and second-order optimization, regardless of whether those
samples are actually more informative.

The decoupled surrogate eliminates this samplewise multiplicity entirely. Its class gradient is the
standard multiclass cross-entropy gradient, and each expert receives an
independent Bernoulli gradient of size $(\lambda/J)(u_j-t_j)$, unaffected by
how many other experts are correct. Thus the decoupled surrogate keeps the optimization problem on
the same scale across low-overlap and high-overlap regions of the input space.

\subsection{Worked Example: Redundant Experts}

Consider $K=2$ classes and $J=4$ redundant experts. Suppose the Bayes-optimal
expert behavior is represented by expert~1 and experts~2,3,4 are exact copies
of it. On a sample with
true label $y=1$, assume all four experts are correct, so $|\Jset|=4$.

In this clean redundancy scenario, the decision problem contains only one
useful expert behavior, copied four times. A reasonable surrogate should
therefore treat this sample almost the same way as a sample with only one
correct expert, and the calculation below shows that additive CE does not.

It is useful to make the axis-(i) issue explicit first. Suppose the class
posterior is
\[
\eta(x)=(0.6,0.4),
\]
and the underlying useful expert behavior has success probability $0.9$. If
there is only one such expert, additive CE targets
\[
\left(\frac{0.6}{1+0.9},\frac{0.4}{1+0.9},\frac{0.9}{1+0.9}\right)
=
(0.316,\;0.211,\;0.474).
\]
If we now replace that one expert by four exact copies, the Bayes decision does
not change: the largest class posterior is still $0.6$ and the best expert
utility is still $0.9$. However, additive CE now targets
\[
\left(
\frac{0.6}{1+4\cdot 0.9},
\frac{0.4}{1+4\cdot 0.9},
\frac{0.9}{1+4\cdot 0.9},
\frac{0.9}{1+4\cdot 0.9},
\frac{0.9}{1+4\cdot 0.9},
\frac{0.9}{1+4\cdot 0.9}
\right)
\]
\[
=
(0.130,\;0.087,\;0.196,\;0.196,\;0.196,\;0.196).
\]
So even before discussing gradients, the surrogate target has already moved,
although only the number of redundant copies changed and the underlying
decision problem did not.

Take an uninformative additive-CE initialization with equal logits on the
$K+J=6$ actions, hence $q_i=1/6$ for all $i$. Proposition~\ref{prop:additive-gradient}
gives
\[
\frac{\partial\Phi^\mathrm{CE}}{\partial a_1}
=
5\cdot\frac{1}{6}-1=-\frac{1}{6},
\qquad
\frac{\partial\Phi^\mathrm{CE}}{\partial a_{2}}
=
5\cdot\frac{1}{6}= \frac{5}{6},
\]
and for every correct expert $j\in\{1,2,3,4\}$,
\[
\frac{\partial\Phi^\mathrm{CE}}{\partial a_{K+j}}
=
5\cdot\frac{1}{6}-1=-\frac{1}{6}.
\]
Thus the sample contributes one rewarded class term and four rewarded expert
terms, for a total positive mass of $5$. By the Hessian formula proved above,
the same sample also carries curvature prefactor $5$, so both quantities are
inflated by the redundancy count.

The rescaling matters because it differs across samples. If only expert~1
were present, the same derivation would give total positive mass $2$ rather
than $5$, because the sample would reward one class and one expert, so four
exact copies multiply the weight of this sample relative to low-overlap
samples by $5/2$, even though no new information has been added.

Now compare the decoupled surrogate on the same sample. Let the class head start at
$p=(1/2,1/2)$ and each expert head at $u_j=1/2$. Then
\[
\frac{\partial\Phi^{\mathrm{dec}}}{\partial w_1}=-\frac{1}{2},
\qquad
\frac{\partial\Phi^{\mathrm{dec}}}{\partial w_2}=\frac{1}{2},
\qquad
\frac{\partial\Phi^{\mathrm{dec}}}{\partial s_j}
=\frac{\lambda}{4}\left(\frac{1}{2}-1\right)
=-\frac{\lambda}{8},
\]
for all $j=1,\dots,4$. If only expert~1 were present ($J=1$), its gradient
would be $\lambda(\tfrac12-1)=-\tfrac{\lambda}{2}$, which equals the total
expert gradient $4\cdot(-\tfrac{\lambda}{8})$ above, so at fixed total weight
duplication leaves both the class gradient and the total expert gradient
unchanged. The decoupled surrogate updates each expert coordinate because that
expert is correct, not because other coordinates happen to agree with it.

Under additive CE, duplicating one correct expert into four copies changes
both what the model is trying to estimate and how strongly this sample pulls
on the optimizer. Under the decoupled surrogate, duplication does neither: the
class task remains the same, and each expert is trained on its own
correctness event with the same local scale. The decoupled surrogate
therefore lacks the mechanism through which additive CE can become dominated
by high-overlap samples.

\section{PiCCE \texorpdfstring{\citep{liu2026more}}{}}

PiCCE \citep{liu2026more} was designed to remove the additive amplification
of CE \citep{mozannar2021consistent} by rewarding only one correct expert per
sample. This section shows the resulting tradeoff: the samplewise scale
becomes controlled, but the winner-take-all mechanism creates a different
pathology, in which correct experts that fail to win the internal $\argmax$
receive negative updates.

The section parallels the one on additive CE. We first define the surrogate
precisely, then isolate the starvation effect at the sample level, then show
how the same winner-take-all mechanism induces a moving conditional target,
and finally compare the resulting behavior with the decoupled surrogate. Two
facts should be kept distinct throughout: PiCCE improves on additive CE
because it removes the overlap-dependent amplification factor, but it does so
by forcing correct experts to compete with one another, and we show that this
competition causes both starvation and a score-dependent target.

\subsection{Formulation}

\begin{definition}[PiCCE, \citet{liu2026more}]
\label{def:appendix-picce}
Let $a(x)\in\mathbb{R}^{K+J}$ be an augmented score vector and define
\[
q_i(x)\coloneqq \frac{\exp(a_i(x))}{\sum_{r=1}^{K+J}\exp(a_r(x))},
\qquad i\in[K+J].
\]
Let
\[
j^\star(a;x,y,m)
=\argmax_{j:\,m_j=y} a_{K+j}(x)
\]
with deterministic tie-breaking, whenever at least one expert is correct. Then
\[
\Phi^{\mathrm{PiCCE}}(a;x,y,m)
\coloneqq
-\log q_y(x)
-\mathbf{1}\{\exists j\in[J]: m_j=y\}\log q_{K+j^\star(a;x,y,m)}(x).
\]
\end{definition}

Prediction again uses the augmented argmax over $a(x)$, as for additive CE.

\subsection{Winner-Take-All Starvation}
\label{app:proof-picce-starvation}

\piccestarvation*
\begin{proof}
Once the winning expert is fixed, PiCCE becomes a softmax loss with exactly two
rewarded coordinates: the true class and the selected expert. The sign pattern
for every other correct expert can therefore be read off directly from the
gradient.

Assume $\mathcal{J}(y,m)\neq\varnothing$ and let $j^\star$ be the unique PiCCE
winner. Since $a_{K+j^\star}>a_{K+j}$ for every other $j\in\mathcal{J}(y,m)$,
this strict ordering persists on a neighbourhood of $a$, so the winner is
locally constant and, on that neighbourhood,
\[
\Phi^{\mathrm{PiCCE}}(a;x,y,m)
=
2\log\sum_{r=1}^{K+J}e^{a_r}
-a_y-a_{K+j^\star}.
\]
Differentiating with respect to $a_i$ gives
$\partial\Phi^{\mathrm{PiCCE}}/\partial a_i
= 2q_i - \mathbf{1}\{i=y\} - \mathbf{1}\{i=K+j^\star\}$.
This formula already shows the winner-take-all structure: only two coordinates
receive a subtractive reward term, namely the true class and the selected
expert, and every other coordinate sees only the positive softmax term $2q_i$.
At a tie, the deterministic tie-break fixes $j^\star$, and the same
computation gives the derivative of the selected branch.

For a correct expert $j'\neq j^\star$, set $i=K+j'$. Since $j'\neq j^\star$,
neither indicator is active, so
\[
\frac{\partial\Phi^{\mathrm{PiCCE}}}{\partial a_{K+j'}}
= 2q_{K+j'}(x) > 0,
\]
where the strict positivity uses $q_{K+j'}>0$ (softmax outputs are strictly
positive). Hence gradient descent pushes $a_{K+j'}$ downward.

The expert $j'$ is correct
on the current sample, but because it failed to win the internal competition,
the update acts as if it were an ordinary non-target coordinate and suppresses
its logit.
\end{proof}

\subsection{Score-Dependent Conditional Target}
\label{app:picce-moving-target}

For completeness, we derive the conditional risk formula used in
Section~\ref{subsec:wta}. Condition on $X=x$ and take expectations over
$(Y,\mathbf M)$. The PiCCE loss rewards expert
$j^\star=\argmax_{j\in\Jset(Y,\mathbf M)}a_{K+j}$ whenever
$\Jset\neq\varnothing$. Define
$\gamma_j(a,x)=\Pr(j\in\Jset(Y,\mathbf M) \text{ and } j=j^\star\mid X=x)$.
The class term is standard:
\[
\mathbb{E}[\log q_Y\mid X=x]=\sum_{k=1}^K\eta_k(x)\log q_k.
\]
For the expert term, partition the event $\{\Jset\neq\varnothing\}$ according
to which expert wins the internal competition. This yields
\[
\mathbb{E}\!\left[
\mathbf{1}\{\Jset\neq\varnothing\}\log q_{K+j^\star}
\mid X=x
\right]
=
\sum_{j=1}^J \gamma_j(a,x)\log q_{K+j}.
\]
Substituting both identities gives
\[
\mc{C}_{\mathrm{PiCCE}}(q,a\mid x)
=
-\sum_{k=1}^K\eta_k(x)\log q_k
-\sum_{j=1}^J\gamma_j(a,x)\log q_{K+j}.
\]
Thus, for a fixed score vector $a$, PiCCE behaves like an augmented
cross-entropy whose class-side coefficients are $\eta_k(x)$ and whose
expert-side coefficients are $\gamma_j(a,x)$, not $\alpha_j(x)$. Equivalently,
viewed as a weighted multiclass log-loss over the augmented simplex, its
instantaneous target is proportional to
\[
\bigl(\eta_1(x),\dots,\eta_K(x),\gamma_1(a,x),\dots,\gamma_J(a,x)\bigr).
\]
The coefficients $\gamma_j(a,x)$ depend on the score vector $a$ through the
winner map $j^\star$, so the expert-side target moves as $a$ is updated. This
is the precise sense in which PiCCE chases a moving target: the class-side
weights are fixed by the data distribution, but the expert-side weights are
redistributed by the current score ordering, in contrast to additive CE,
where all coefficients are fixed by the distribution alone.

For a concrete case, fix a point $x$ and suppose
there is positive conditional probability that experts~1 and~2 are both
correct. If the current scores satisfy $a_{K+1}>a_{K+2}$ on those outcomes,
then that shared probability mass contributes to $\gamma_1(a,x)$. If the score
ordering is reversed, the same data mass contributes to $\gamma_2(a,x)$
instead. The data distribution has not changed, only the current score
ordering has, which is why the PiCCE expert target is score-dependent rather
than intrinsic to the underlying decision problem.

\subsection{Curvature}
\label{app:proof-picce-curvature}

The winner-take-all repair eliminates the multiplicity factor of additive CE.
Whenever $\Jset(y,m)\neq\varnothing$,
\[
\Phi^{\mathrm{PiCCE}}(a;x,y,m)
=
2\log\sum_{r=1}^{K+J}e^{a_r}-a_y-a_{K+j^\star},
\]
where $j^\star$ is fixed on the neighbourhood of $a$ on which the winner is
unique. We now differentiate
this expression exactly as in the additive-CE case, but with prefactor $2$
instead of $1+|\Jset|$.

The only nonlinear term is the log-sum-exp. Its Hessian is
\[
\nabla_a^2 \log\sum_{r=1}^{K+J}e^{a_r}
=
\mathrm{Diag}(q)-qq^\top,
\]
while the linear terms $-a_y$ and $-a_{K+j^\star}$ vanish under second
differentiation. Therefore
\[
\nabla_a^2\Phi^{\mathrm{PiCCE}}
=
2(\mathrm{Diag}(q)-qq^\top).
\]
The matrix $\mathrm{Diag}(q)-qq^\top$ is again the softmax covariance matrix,
and by the additive-CE derivation above every eigenvalue of this matrix lies
in $[0,1/2]$. Multiplying by the prefactor $2$ gives
\[
\lambda_{\max}\!\left(\nabla_a^2\Phi^{\mathrm{PiCCE}}\right)\le 1.
\]

PiCCE is therefore well behaved in the second-order sense in which additive
CE is not, which isolates the actual failure mode. The problem with PiCCE is
not exploding curvature or overlap-dependent scaling but the direction of the
updates: the surrogate suppresses correct experts that fail to win the
internal comparison, and the same winner-take-all mechanism makes the
expert-side target score-dependent.

\subsection{Head-to-Head with the Decoupled Surrogate}

\paragraph{Axis~(i): statistical target.}
PiCCE preserves the correct class coefficients $\eta_k(x)$, but on the expert
side it replaces $\alpha_j(x)$ by the score-dependent quantity
$\gamma_j(a,x)$. The target therefore depends on the optimization trajectory
through the current winner map, which is the first contrast with the
decoupled surrogate. Under the decoupled surrogate, the expert target for
coordinate $j$ is always the same object, $\alpha_j(x)$, regardless of which
expert currently has the highest score, so the data alone determine the
target. Under PiCCE, by contrast, shared correct mass is assigned to
whichever expert is currently winning the internal competition.

Because the winner map changes during training, the quantity PiCCE
estimates can itself shift as training proceeds, whereas the class and expert
subproblems of the decoupled surrogate are proper losses with fixed
conditional targets $(\eta,\alpha)$.

\paragraph{Axis~(ii): local geometry.}
PiCCE solves the multiplicity problem by capping the number of rewarded terms
at two, but it does so by selecting a unique correct expert and repelling
every other correct expert, so that on a sample where several experts are
simultaneously correct, PiCCE turns agreement into competition.

The decoupled surrogate makes the opposite design choice, keeping the
per-coordinate updates bounded while reinforcing every correct expert
independently. Agreement among correct experts therefore creates neither
additive amplification nor winner-take-all suppression. PiCCE and the
decoupled surrogate agree on the need to avoid additive amplification, but
they differ on whether multiple correct experts should compete with one
another or all receive supportive updates.

\subsection{Worked Example: Generalist and Specialist Both Correct}

Consider one generalist expert and one specialist expert. On a sample from the
specialist region, suppose both experts are correct, so $\Jset=\{1,2\}$, but
the current logits satisfy $a_{K+1}>a_{K+2}$, meaning the generalist wins the
internal PiCCE competition. In this situation one would like the training
signal to preserve both experts, since the generalist is correct but the
specialist is also correct on the very region for which it was meant to be
useful; PiCCE instead forces the two experts to compete.

To make the signs concrete, suppose the current softmax assigns
\[
q_{K+1}=0.30,
\qquad
q_{K+2}=0.20.
\]
Then $j^\star=1$ and
\[
\frac{\partial\Phi^{\mathrm{PiCCE}}}{\partial a_{K+1}}
=2q_{K+1}-1,
\qquad
\frac{\partial\Phi^{\mathrm{PiCCE}}}{\partial a_{K+2}}
=2q_{K+2}.
\]
Substituting the numbers gives
\[
\frac{\partial\Phi^{\mathrm{PiCCE}}}{\partial a_{K+1}}
=2(0.30)-1=-0.40,
\qquad
\frac{\partial\Phi^{\mathrm{PiCCE}}}{\partial a_{K+2}}
=2(0.20)=0.40.
\]
Under gradient descent, the generalist logit is pushed upward because its
gradient is negative, while the specialist logit is pushed downward because its
gradient is positive. So the same sample provides evidence for both experts,
yet PiCCE uses that evidence to strengthen the winner and weaken the loser.

This same example also reflects the axis-(i) issue. On any population of such
samples, the shared ``both correct'' mass is assigned to whichever expert
currently wins the internal comparison. If the generalist is ahead, that mass
contributes to $\gamma_1(a,x)$; if the specialist overtakes it, the same mass
contributes to $\gamma_2(a,x)$ instead. The target therefore co-moves with the
current score ordering.

For the decoupled surrogate, let the same sample have expert probabilities
$u_1=0.80$ and $u_2=0.75$, with targets $t_1=t_2=1$. Then
\[
\frac{\partial\Phi^{\mathrm{dec}}}{\partial s_1}
=\frac{\lambda}{2}(0.80-1)=-0.10\lambda,
\qquad
\frac{\partial\Phi^{\mathrm{dec}}}{\partial s_2}
=\frac{\lambda}{2}(0.75-1)=-0.125\lambda.
\]
Both experts are reinforced: the specialist may still trail the generalist at
the current iterate, but it is not punished for being correct, and the sample
acts as supportive evidence for both experts rather than as an internal
tournament.

This example is the PiCCE analogue of the redundant-expert example for
additive CE. Whenever several experts are correct on the same sample, PiCCE
converts their agreement into competition, whereas the decoupled surrogate
keeps the bounded geometry without introducing that competition; this is the
mechanism targeted by the rare-specialist experiments of
Section~\ref{sec:experiments}.

\section{Mao25 \texorpdfstring{\citep{mao2025mastering}}{}}

Mao25 moves away from additive rewards and winner-take-all selection by
optimizing only the total mass assigned to the acceptable set. The object it
fits is therefore not the full vector $(\eta,\alpha)$ but this single set
mass. That choice resolves the two most immediate geometric issues and keeps
the local geometry comparatively benign, but it also removes the signal needed
to rank experts within the acceptable set, which is why the surrogate cannot
express preferences among several acceptable experts.

We keep the positive theory of \citet{mao2025mastering} in view throughout.
In the single-stage multiple-expert setting, their family
$L_{\Psi}$ comes with formal consistency guarantees under mild endpoint
conditions on~$\Psi$. They further single out the member $\Psi(u)=1-u$,
denoted $L_{\mathrm{mae}}$ in their paper, because it also admits quantitative
excess-risk control. Our criticism does not concern those consistency
guarantees but the practical optimization behavior of the specific
$\Psi(u)=1-u$ instantiation, whose samplewise loss is linear in acceptable-set
mass and therefore has an MAE-like optimization profile.

\subsection{Formulation}

We use the member $\Psi(u)=1-u$, which \citet{mao2025mastering} also
instantiate in their multiple-expert experiments and which is the version in
our empirical comparison.

\begin{definition}[Mao25 single-stage surrogate with $\Psi(u)=1-u$, \citet{mao2025mastering}]
\label{def:appendix-mao25}
Let $a(x)\in\mathbb{R}^{K+J}$ be an augmented score vector and define
\[
q_i(x)\coloneqq \frac{\exp(a_i(x))}{\sum_{r=1}^{K+J}\exp(a_r(x))},
\qquad i\in[K+J].
\]
Then
\[
\Phi^{\mathrm{Mao}}(a;x,y,m)
\coloneqq
1-\Bigl(q_y(x)+\sum_{j:\,m_j=y} q_{K+j}(x)\Bigr).
\]
\end{definition}

Prediction uses the augmented argmax over $a(x)$, and the surrogate depends only
on the total mass assigned to the acceptable set, which consists of the true
class and all correct experts.

To see what this implies at the conditional level, fix an input $x$ and take
the expectation over the random pair $(Y,\mathbf{M})$. The class part
contributes
\[
\mathbb{E}[q_Y\mid X=x]=\sum_{k=1}^K \eta_k(x)\,q_k,
\]
because $q_k$ is fixed once $x$ is fixed, and only the random label $Y$
remains inside the expectation. Likewise, for the expert part,
\[
\mathbb{E}\!\left[\sum_{j:\,M_j=Y} q_{K+j}\,\middle|\,X=x\right]
=
\sum_{j=1}^J \alpha_j(x)\,q_{K+j},
\]
because expert $j$ contributes exactly on the event $\{M_j=Y\}$, whose
conditional probability is $\alpha_j(x)$ by definition.

Substituting these two identities into the loss gives
\[
\mc{C}_{\mathrm{Mao}}(q\mid x)
=
1-\sum_{k=1}^K \eta_k(x) q_k-\sum_{j=1}^J \alpha_j(x) q_{K+j}.
\]
The conditional risk is thus a linear function of the augmented probabilities
$q$, and minimizing it is the same as maximizing
\[
\sum_{k=1}^K \eta_k(x) q_k+\sum_{j=1}^J \alpha_j(x) q_{K+j}
\]
over the simplex $\Delta^{K+J}$. A linear objective over a simplex attains its
maximum at a vertex, so over the closed simplex the minimum is attained by a
point mass on an action whose coefficient is maximal among
\[
\bigl(\eta_1(x),\dots,\eta_K(x),\alpha_1(x),\dots,\alpha_J(x)\bigr).
\]
Finite scores only approach this vertex. If the maximizer is unique, this
point mass is the unique minimizer over $\Delta^{K+J}$ and sits on the
Bayes-optimal action. If several actions tie for the maximum, any convex
combination over those maximizers is also optimal.

Mao25 can therefore recover a Bayes-optimal action, but it does not recover
the full vector $(\eta,\alpha)$: it identifies which action can be optimal,
not the underlying utilities themselves.

\subsection{Acceptable-Set Mass Without Ranking}
\label{app:proof-mao25-setmass}

\maosetmass*
\begin{proof}
The argument has two steps: we first show that the loss depends on the score
vector only through the acceptable-set mass, and then differentiate that mass
to obtain the gradient formula.

Let
\[
\Sset(y,m)\coloneqq \{y\}\cup\{K+j:j\in\mathcal{J}(y,m)\}
\]
denote the acceptable set. By definition,
\[
\Phi^{\mathrm{Mao}}(a;x,y,m)
=
1-\sum_{c\in \Sset(y,m)} q_c(x)
=
1-S_{\mathcal S}(x).
\]
Thus the loss depends on $a$ only through the total acceptable mass
$S_{\mathcal S}(x)$.

To differentiate, use the softmax Jacobian
$\partial q_c/\partial a_i = q_c(\mathbf{1}\{c=i\}-q_i)$. Then
\begin{align*}
\frac{\partial S_{\mathcal S}}{\partial a_i}
&=
\sum_{c\in \Sset(y,m)} q_c(\mathbf{1}\{c=i\}-q_i)\\
&=
\mathbf{1}\{i\in \Sset(y,m)\}q_i-q_i\sum_{c\in \Sset(y,m)}q_c\\
&=
q_i\bigl(\mathbf{1}\{i\in \Sset(y,m)\}-S_{\mathcal S}\bigr),
\end{align*}
and therefore
\[
\frac{\partial \Phi^{\mathrm{Mao}}}{\partial a_i}
=
-\frac{\partial S_{\mathcal S}}{\partial a_i}
=
q_i\bigl(S_{\mathcal S}-\mathbf{1}\{i\in \Sset(y,m)\}\bigr).
\]
If $i\in \Sset(y,m)$, then $S_{\mathcal S}\le 1$, so the gradient is
non-positive. If $i\notin \Sset(y,m)$, then the gradient is
$q_iS_{\mathcal S}\ge 0$. Finally, since
$\Phi^{\mathrm{Mao}}=1-S_{\mathcal S}$, any two score vectors with the same
$S_{\mathcal S}$ incur the same loss regardless of their internal allocation
inside $\Sset(y,m)$.

This formula makes the absence of direct within-set ranking explicit. If
$i,r\in \Sset(y,m)$ are both acceptable actions on the current sample, then
\[
\frac{\partial \Phi^{\mathrm{Mao}}}{\partial a_i}
=
q_i(S_{\mathcal S}-1),
\qquad
\frac{\partial \Phi^{\mathrm{Mao}}}{\partial a_r}
=
q_r(S_{\mathcal S}-1).
\]
Both coordinates share the same factor $S_{\mathcal S}-1$, and their update
magnitudes differ only through the current probabilities $q_i$ and $q_r$, not
through any term encoding which acceptable expert is statistically better.

The same formula also reveals why optimization can be weak. Summing the
attractive gradient mass over the acceptable set gives
\[
\sum_{i\in\Sset} -\frac{\partial \Phi^{\mathrm{Mao}}}{\partial a_i}
=
\sum_{i\in\Sset} q_i(1-S_{\mathcal S})
=
S_{\mathcal S}(1-S_{\mathcal S}).
\]
This quantity is at most $1/4$ and tends to zero as $S_{\mathcal S}\to 0$.
Thus, when the model currently places very little mass on the acceptable set,
the corrective signal toward that set is weak. In the single-label
special case $|\Sset|=1$, the loss reduces to $1-q_y$, which is the
multiclass MAE loss up to a constant factor. The same weak-gradient phenomenon is well known for
MAE-like losses in noisy-label learning
\citep{zhang2018generalizedcrossentropyloss}.
\end{proof}

\subsection{Curvature}
\label{app:proof-mao-curvature}

Although Mao25 does not suffer from additive amplification, its second-order
structure depends on the scores only through $q$, acceptable-set membership,
and the total acceptable mass. We differentiate the gradient formula step by step.

Write
\[
g_i(a)\coloneqq \frac{\partial \Phi^{\mathrm{Mao}}}{\partial a_i}
=
q_i\bigl(S_{\mathcal S}-\mathbf{1}\{i\in\Sset\}\bigr).
\]
Then
\[
\frac{\partial g_i}{\partial a_r}
=
\frac{\partial q_i}{\partial a_r}\bigl(S_{\mathcal S}-\mathbf{1}\{i\in\Sset\}\bigr)
+q_i\frac{\partial S_{\mathcal S}}{\partial a_r}.
\]
The first term comes from differentiating the prefactor $q_i$ and the second
term comes from differentiating the acceptable-set mass $S_{\mathcal S}$.
From the softmax Jacobian,
\[
\frac{\partial q_i}{\partial a_r}
=
q_i(\mathbf{1}\{i=r\}-q_r),
\]
and from the previous subsection,
\[
\frac{\partial S_{\mathcal S}}{\partial a_r}
=
q_r\bigl(\mathbf{1}\{r\in\Sset\}-S_{\mathcal S}\bigr).
\]
Substituting both identities into the product rule gives the full Hessian
entry:
\[
\frac{\partial^2 \Phi^{\mathrm{Mao}}}{\partial a_i\,\partial a_r}
=
q_i(\mathbf{1}\{i=r\}-q_r)\bigl(S_{\mathcal S}-\mathbf{1}\{i\in\Sset\}\bigr)
+q_iq_r\bigl(\mathbf{1}\{r\in\Sset\}-S_{\mathcal S}\bigr).
\]
This is the complete second derivative, since no further term arises from
differentiating the explicit gradient formula. It simplifies in four cases,
according to whether $i$ and $r$ belong to the acceptable set.

\paragraph{Case 1: $i=r\in\Sset$.}
Then
\[
\frac{\partial^2 \Phi^{\mathrm{Mao}}}{\partial a_i^2}
=
q_i(1-q_i)(S_{\mathcal S}-1)
+q_i^2(1-S_{\mathcal S}).
\]
This depends on $q_i$ and on the common scalar $S_{\mathcal S}$, but not on
which acceptable expert $i$ represents.

\paragraph{Case 2: $i\neq r$ and $i,r\in\Sset$.}
Then
\[
\frac{\partial^2 \Phi^{\mathrm{Mao}}}{\partial a_i\,\partial a_r}
=
-q_iq_r(S_{\mathcal S}-1)+q_iq_r(1-S_{\mathcal S})
=
2q_iq_r(1-S_{\mathcal S}).
\]
Again the expression depends only on the current masses and on the shared
acceptable-set mass.

\paragraph{Case 3: $i\in\Sset$ and $r\notin\Sset$.}
Then
\[
\frac{\partial^2 \Phi^{\mathrm{Mao}}}{\partial a_i\,\partial a_r}
=
-q_iq_r(S_{\mathcal S}-1)-q_iq_rS_{\mathcal S}
=
q_iq_r(1-2S_{\mathcal S}).
\]

\paragraph{Case 4: $i,r\notin\Sset$.}
Then
\[
\frac{\partial^2 \Phi^{\mathrm{Mao}}}{\partial a_i\,\partial a_r}
=
q_i(\mathbf{1}\{i=r\}-q_r)S_{\mathcal S}
-q_iq_rS_{\mathcal S}.
\]
In particular, for $i\neq r$ both outside the acceptable set,
\[
\frac{\partial^2 \Phi^{\mathrm{Mao}}}{\partial a_i\,\partial a_r}
=
-2q_iq_rS_{\mathcal S}.
\]

None of these expressions contains a term that depends on which acceptable
expert is statistically better, so even at second order the surrogate sees
acceptable-set mass and membership, not expert quality. This is the geometric
analogue of the first-order statement above: Mao25 distinguishes acceptable from unacceptable
actions, but its local geometry contains no utility-sensitive comparison inside
the acceptable set.

\subsection{Head-to-Head with the Decoupled Surrogate}

\paragraph{Axis~(i): statistical target.}
Mao25 is aligned with the Bayes action only at the coarse level of choosing the
best augmented action. Its conditional risk is linear in $q$, so a minimizer
places all mass on one optimal action, but the surrogate does not recover the
full utility vector $(\eta,\alpha)$. The decoupled surrogate does: at the
conditional minimum, it returns $p^\star(x)=\eta(x)$ and
$u_j^\star(x)=\alpha_j(x)$ for every expert. The difference is therefore not
whether the two surrogates can identify a good action in principle, but that
the decoupled surrogate estimates the underlying utilities themselves, whereas
Mao25 only optimizes a coarse action-level objective.

\paragraph{Axis~(ii): local geometry.}
Mao25 avoids both additive amplification and winner-take-all starvation, and
in that sense it is geometrically cleaner than additive CE and PiCCE. Its
geometry is nevertheless too coarse for utility estimation. On a sample where
several actions are acceptable, the loss depends only on the total acceptable
mass and the acceptable/unacceptable partition, and it contains no term that
says which acceptable expert should be preferred. For the concrete choice
$\Psi(u)=1-u$, the total attractive gradient mass toward the acceptable set is
$S_{\mathcal S}(1-S_{\mathcal S})$, which is at most $1/4$ and
shrinks to zero as $S_{\mathcal S}\to 0$, so the corrective signal can be weak
precisely when the model currently assigns little mass to the right set. The decoupled surrogate
keeps the benign bounded geometry while training each expert head against its
own correctness process.

The comparison with the decoupled surrogate therefore differs from the PiCCE
comparison, because Mao25 neither suppresses correct experts nor amplifies
high-overlap samples. Its limitation is coarseness: it treats expert
supervision as set membership rather than utility estimation, and under the
linear choice $\Psi(u)=1-u$ this coarse signal can also be weak away from the
optimum. The decoupled surrogate retains the stable geometry while preserving expert-specific
targets.

\subsection{Worked Example: Two Correct Experts of Unequal Quality}

Consider a defer region in which expert~1 is substantially better overall than
expert~2:
\[
\alpha_1(x)=0.9,
\qquad
\alpha_2(x)=0.5.
\]
Assume the largest class posterior is smaller, say $\max_k \eta_k(x)=0.4$, so the
Bayes-optimal action is to defer to expert~1.

The conditional Mao25 risk at this $x$ is then
\[
\mc{C}_{\mathrm{Mao}}(q\mid x)
=
1-\sum_{k=1}^K \eta_k(x) q_k-0.9\,q_{K+1}-0.5\,q_{K+2}.
\]
This conditional objective is minimized by placing all mass on expert~1, which
indeed recovers the correct action, but the surrogate does not recover the
utility values $0.9$ and $0.5$ themselves; it only prefers the larger
coefficient through a linear argmax objective.

Now inspect the samplewise supervision when both experts happen to be correct
on a particular draw. Compare two augmented allocations with the same
acceptable mass:
\[
(q_y,q_{K+1},q_{K+2})=(0.10,0.50,0.00),
\qquad
(q_y,q_{K+1},q_{K+2})=(0.10,0.25,0.25).
\]
Both have $S_{\mathcal S}=0.60$, so both incur the same Mao25 loss
$\Phi^{\mathrm{Mao}}=0.40$. The surrogate therefore cannot prefer the first
allocation over the second on the basis that expert~1 is the reliable one.

The gradients make the same point: since both experts are acceptable on this
sample, the gradient formula gives
\[
\frac{\partial \Phi^{\mathrm{Mao}}}{\partial a_i}
=
q_i(0.60-1)=-0.4\,q_i,
\qquad i\in \{y,K+1,K+2\}.
\]
For the first allocation,
\[
\left(
\frac{\partial \Phi^{\mathrm{Mao}}}{\partial a_y},
\frac{\partial \Phi^{\mathrm{Mao}}}{\partial a_{K+1}},
\frac{\partial \Phi^{\mathrm{Mao}}}{\partial a_{K+2}}
\right)
=
(-0.04,\,-0.20,\,0.00),
\]
whereas for the second allocation,
\[
\left(
\frac{\partial \Phi^{\mathrm{Mao}}}{\partial a_y},
\frac{\partial \Phi^{\mathrm{Mao}}}{\partial a_{K+1}},
\frac{\partial \Phi^{\mathrm{Mao}}}{\partial a_{K+2}}
\right)
=
(-0.04,\,-0.10,\,-0.10).
\]
The difference between the two acceptable experts is driven only by their
current masses $q_{K+1}$ and $q_{K+2}$, not by the fact that expert~1 is the
better expert overall.

Under the decoupled surrogate, and up to the positive factor $\lambda/J$,
the expert-side conditional risk is
\[
-0.9\log u_1-0.1\log(1-u_1)
-0.5\log u_2-0.5\log(1-u_2),
\]
which is uniquely minimized at $u_1=0.9$ and $u_2=0.5$, so the ranking signal
is present by construction. Mao25 can recover the correct defer action, but it
does not estimate the underlying expert utilities; on the sample itself, it
only rewards membership in the acceptable set, and the total reward scale is
capped by $S_{\mathcal S}(1-S_{\mathcal S})$. The decoupled surrogate does
both: it preserves the correct action and recovers the utility gap that
distinguishes expert~1 from expert~2.

\section{A-SM \texorpdfstring{\citep{Cao_Mozannar_Feng_Wei_An_2023}}{}}

A-SM \citep{Cao_Mozannar_Feng_Wei_An_2023} is the strongest augmented-action
baseline on the population axis, since it restores bounded estimates and
recovers the correct target. Its remaining weakness lies at the optimization
level: the expert losses still leak into the class gradient through the shared
score vector.

This section therefore plays a different role from the earlier baseline
sections. For additive CE, PiCCE, and Mao25, the main mismatch is already
visible in the conditional target, whereas for A-SM the conditional target is
already correct. Our aim here is to isolate the remaining axis-(ii) weakness
and show that it survives even after boundedness and population alignment are
repaired.

\subsection{Formulation}

A-SM replaces the symmetric augmented normalization by an asymmetric one that
keeps the class probabilities and expert probabilities bounded and semantically
valid. We use the multi-expert extension given in Appendix~H of
\citet{Cao_Mozannar_Feng_Wei_An_2023}. To stay consistent with the main text,
we write $\xi_k$ for the class estimates and $\psi_j$ for the expert estimates.

The single-expert A-SM paper is written with one defer coordinate, so its
expert denominator is naturally expressed over $K+1$ coordinates: the $K$
class logits and the single defer logit. In the multi-expert extension, this
construction is applied separately to each expert. Thus the class head still
uses the first $K$ logits only, while expert~$j$ is normalized against those
same $K$ class logits plus its own expert logit $a_{K+j}$, not against the
other expert logits.

\begin{definition}[Multi-expert A-SM, \citet{Cao_Mozannar_Feng_Wei_An_2023}]
\label{def:appendix-asm}
Let $a(x)\in\mathbb{R}^{K+J}$ be an augmented score vector. Define the class
estimates as the standard softmax over the first $K$ coordinates,
\[
\xi_k(a(x))
\coloneqq
\frac{\exp(a_k(x))}{\sum_{k'=1}^K \exp(a_{k'}(x))},
\qquad k\in[K],
\]
and the expert estimates, for each $j\in[J]$, as
\[
\psi_j(a(x))
\coloneqq
\frac{\exp(a_{K+j}(x))}
{\sum_{k'=1}^K \exp(a_{k'}(x)) + \exp(a_{K+j}(x)) -
\max_{k'\in[K]}\exp(a_{k'}(x))}.
\]
The corresponding surrogate is
\begin{equation}
\begin{aligned}
\Phi^{\mathrm{ASM}}(a;x,y,m)
\coloneqq
-\log \xi_y(a(x))
-\sum_{j=1}^J
\Bigl(
&\mathbf{1}\{m_j=y\}\log \psi_j(a(x))\\
&+\mathbf{1}\{m_j\neq y\}\log\bigl(1-\psi_j(a(x))\bigr)
\Bigr).
\end{aligned}
\end{equation}
\end{definition}

Prediction uses the maxima-preserving property of $(\xi,\psi)$, equivalently
the argmax of the underlying score vector $a(x)$. The expert estimate can be
written through the equivalent sigmoid identity
\[
\psi_j(a(x))
=
\sigma\!\left(a_{K+j}(x)-\log \sum_{k'\neq k^\star}\exp(a_{k'}(x))\right),
\]
where $k^\star\in\argmax_{k\in[K]}a_k(x)$.
To see the equivalence directly, write
\[
B(a)\coloneqq \sum_{k=1}^K \exp(a_k)-\max_{k\in[K]}\exp(a_k)
=\sum_{k\neq k^\star}\exp(a_k).
\]
Then the denominator of $\psi_j$ is exactly $B(a)+\exp(a_{K+j})$, so
\[
\psi_j(a(x))
=
\frac{\exp(a_{K+j}(x))}{\exp(a_{K+j}(x))+B(a(x))}
=
\sigma\!\left(a_{K+j}(x)-\log B(a(x))\right).
\]
This form is algebraically equivalent to the original max-based denominator and
is the convenient form for the derivative calculations below. It also makes
clear that the map is only piecewise smooth: when the maximizing class changes,
the active comparison set changes as well.

The sigmoid representation also shows that each expert head is, by itself, an
ordinary sigmoid probability whose reference level is set by the non-max class
logits. The expert-side issue in A-SM is therefore not that the expert
probability is ill-defined or unbounded, but that the class logits appear
inside that reference level, so class and expert optimization remain
entangled.

\subsection{Conditional Target}

At the population level, A-SM is intended to repair exactly the defect that
motivated Cao et al.: the single augmented softmax of additive CE does not
directly produce bounded estimates of the class posterior and expert
correctness. A-SM changes the parameterization so that the class side remains a
categorical distribution over the first $K$ coordinates, while each expert side
is a bounded probability in $[0,1]$.

The multi-expert extension in Appendix~H of
\citet{Cao_Mozannar_Feng_Wei_An_2023} states the corresponding optimality
property directly: at a conditional optimum,
\[
\xi_k^\star(x)=\eta_k(x),
\qquad
\psi_j^\star(x)=\alpha_j(x),
\]
for every class $k\in[K]$ and expert $j\in[J]$. This property is the reason
A-SM is the strongest augmented-action baseline on axis~(i). Unlike additive
CE, PiCCE, and Mao25, A-SM identifies the right statistical object, but it
reaches that object through a shared geometry that still couples class and
expert optimization.

Once the target axis is repaired, any remaining weakness must come from the
optimization geometry rather than from statistical misspecification. A-SM is
therefore the cleanest test case for asking whether a correct target is
sufficient by itself, and our answer is no: even with the right conditional
optimum, the local update structure can remain coupled in a way that the
decoupled surrogate avoids.

\subsection{Class--Expert Gradient Coupling}
\label{app:proof-asm-coupling}

\asmcoupling*
\begin{proof}
We work on a region where the maximizing class index $k^\star$ is unique and
therefore locally fixed; at a tie, the formula is the derivative of the
selected branch. On such a region, the A-SM map is smooth and the active comparison set in the
expert denominators does not change. The goal of this proof is only the
\emph{first-order} coupling statement: we derive the ordinary class term, then
show how each expert residual contributes an additional term to every non-max
class logit.

\paragraph{Class term.}
The term $-\log\xi_y$ is the standard multiclass cross-entropy over
$a_1,\dots,a_K$. Its gradient with respect to $a_r$ is
$\xi_r-\mathbf{1}\{r=y\}$.

\paragraph{Expert term.}
For a fixed expert $j$, define
\[
t_j\coloneqq \mathbf{1}\{m_j=y\},
\qquad
B(a)\coloneqq \sum_{k\neq k^\star}\exp(a_k).
\]
Then
\[
\psi_j
=
\sigma\!\bigl(a_{K+j}-\log B(a)\bigr)
=
\frac{\exp(a_{K+j})}{\exp(a_{K+j})+B(a)}.
\]
The corresponding expert contribution to the loss is the Bernoulli
cross-entropy
\[
L_j
=
-t_j\log\psi_j-(1-t_j)\log(1-\psi_j).
\]
Differentiating with respect to the sigmoid argument gives the standard
identity
\[
\frac{\partial L_j}{\partial v_j}=\psi_j-t_j,
\qquad
v_j\coloneqq a_{K+j}-\log B(a).
\]
This already shows that the expert head itself is well behaved: with respect to
its own effective score $v_j$, it is an ordinary Bernoulli logistic loss. The
only remaining task is to differentiate $v_j$ with respect to the class logits
and see how the coupling arises.

If $r\neq k^\star$, then
\[
\frac{\partial v_j}{\partial a_r}
=
-\frac{1}{B(a)}\frac{\partial B(a)}{\partial a_r}
=
-\frac{\exp(a_r)}{\sum_{k\neq k^\star}\exp(a_k)}
=
-\pi_r,
\]
where
\[
\pi_r\coloneqq \frac{\exp(a_r)}{\sum_{k\neq k^\star}\exp(a_k)}.
\]
If instead $r=k^\star$, then $a_{k^\star}$ does not appear in $B(a)$ on this
region, so $\partial v_j/\partial a_{k^\star}=0$.

Combining the two derivatives gives
\[
\frac{\partial L_j}{\partial a_r}
=
\frac{\partial L_j}{\partial v_j}\frac{\partial v_j}{\partial a_r}
=
-(\psi_j-t_j)\pi_r,
\qquad r\neq k^\star,
\]
and
\[
\frac{\partial L_j}{\partial a_{k^\star}}=0.
\]
This is the leakage mechanism in its simplest form: every expert residual
$\psi_j-t_j$ is injected into every non-max class logit through the same
weight~$\pi_r$.

It is also useful to record the expert-side own derivative. Since
$\partial v_j/\partial a_{K+j}=1$,
\[
\frac{\partial L_j}{\partial a_{K+j}}
=
\frac{\partial L_j}{\partial v_j}
=
\psi_j-t_j.
\]
Viewed in isolation, the expert block therefore behaves exactly like a
standard Bernoulli cross-entropy; the pathology is that the same expert
residuals also enter the class gradient.

Summing over $j$ for $r\neq k^\star$ gives
\[
\sum_{j=1}^J \frac{\partial L_j}{\partial a_r}
= -\pi_r\sum_{j=1}^J(\psi_j-\mathbf{1}\{m_j=y\}),
\]
and adding the class term $\xi_r-\mathbf{1}\{r=y\}$ yields the stated formula.

If many experts are correct but currently
underestimated, then many terms $\psi_j-\mathbf{1}\{m_j=y\}$ are negative, so
the leakage term $-\pi_r\sum_j(\psi_j-\mathbf{1}\{m_j=y\})$ adds an extra
positive contribution to each non-max class gradient. If instead many experts
are currently overestimated, the same mechanism pushes in the opposite
direction. Either way, the class update no longer reflects only class-side
error, because it also carries the aggregate expert residual.
\end{proof}

\subsection{Curvature}
\label{app:proof-asm-curvature}

For A-SM the most informative second-order object is the mixed
class--expert block, because it captures the feature that distinguishes A-SM
from the decoupled surrogate: even though the target is correct, the local
quadratic geometry does not separate the class and expert subproblems.

We again work on a region where the maximizing class index $k^\star$ is fixed.
From Proposition~\ref{prop:asm-coupling}, the gradient on a non-max class logit
$a_r$ is
\[
\frac{\partial\Phi^{\mathrm{ASM}}}{\partial a_r}
=
(\xi_r-\mathbf{1}\{r=y\})
-\pi_r\sum_{j=1}^J(\psi_j-\mathbf{1}\{m_j=y\}),
\qquad r\neq k^\star.
\]
To obtain the mixed Hessian entry, differentiate this expression with respect
to an expert logit $a_{K+j}$.

The class term $\xi_r-\mathbf{1}\{r=y\}$ depends only on the class logits, so
its derivative with respect to $a_{K+j}$ is zero. The coefficient $\pi_r$ also
depends only on the non-max class logits on the fixed-$k^\star$ region, so it
too is constant with respect to $a_{K+j}$. Thus
\[
\frac{\partial^2\Phi^{\mathrm{ASM}}}{\partial a_r\,\partial a_{K+j}}
=
-\pi_r \frac{\partial \psi_j}{\partial a_{K+j}},
\qquad r\neq k^\star.
\]
Now recall that
\[
\psi_j=\sigma(v_j),
\qquad
v_j=a_{K+j}-\log\sum_{k\neq k^\star}\exp(a_k).
\]
Since the second term of $v_j$ involves only class logits, we have
\[
\frac{\partial v_j}{\partial a_{K+j}}=1.
\]
Applying the sigmoid derivative then gives
\[
\frac{\partial \psi_j}{\partial a_{K+j}}
=
\psi_j(1-\psi_j).
\]
Substituting this into the previous display yields the full mixed entry:
\[
\frac{\partial^2\Phi^{\mathrm{ASM}}}{\partial a_r\,\partial a_{K+j}}
=
-\pi_r\,\psi_j(1-\psi_j),
\qquad r\neq k^\star.
\]
For $r=k^\star$, the expert-leakage term is absent already at first order, so
the mixed derivative is zero. Therefore the non-max-class/expert mixed block is
exactly
\[
H_{\mathrm{mix}}=-\pi\,d^\top,
\qquad
\pi_r=\frac{\exp(a_r)}{\sum_{k\neq k^\star}\exp(a_k)},
\qquad
d_j=\psi_j(1-\psi_j).
\]

\paragraph{Rank-one structure.}
The mixed block is an outer product, so it has rank at most one, and all
non-max class coordinates therefore interact with the expert block through the same class-side direction~$\pi$.

\paragraph{Operator-norm bound.}
Because $H_{\mathrm{mix}}=-\pi d^\top$, its operator norm is
\[
\|H_{\mathrm{mix}}\|_{\mathrm{op}}
=
\|\pi\|_2\,\|d\|_2.
\]
Now $\pi$ is a probability vector over the non-max classes, hence
$\|\pi\|_2\le 1$. Also each entry of $d$ satisfies
\[
0\le d_j=\psi_j(1-\psi_j)\le \frac14,
\]
because the scalar function $u\mapsto u(1-u)$ is maximized at $u=1/2$. Hence
\[
\|d\|_2
\le
\sqrt{\sum_{j=1}^J \left(\frac14\right)^2}
=
\frac{\sqrt{J}}{4},
\]
and therefore
\[
\|H_{\mathrm{mix}}\|_{\mathrm{op}}
\le
\frac{\sqrt{J}}{4}.
\]

\paragraph{Interpretation.} The factor $\pi_r$ measures
how active non-max class~$r$ is inside the class-side comparison, while
$\psi_j(1-\psi_j)$ measures how uncertain expert~$j$ currently is. Because
$\pi$ is normalized over the non-max classes only, it does not shrink as
$\xi_{k^\star}\to1$, and $\|\pi\|_2$ reaches one when the non-max mass sits on
a single class, as it always does for $K=2$. The mixed coupling is therefore
strongest when the non-max mass is concentrated and many experts are in the
uncertain middle regime rather than near $0$ or $1$.

In this regime the contrast with the decoupled surrogate is largest, since the
decoupled mixed class--expert block is identically zero. For A-SM, the target
is already correct, but the optimization geometry still contains a structured
interaction whose norm can grow with the expert pool.

For completeness, the expert block does not introduce expert--expert coupling
through the other expert logits: each expert head still depends only on its
own expert coordinate and the class-side reference level. The mixed block is
therefore the natural second-order object to compare with the decoupled
surrogate, because it is where the class and expert tasks interact most
directly.

\subsection{Head-to-Head with the Decoupled Surrogate}

\paragraph{Axis~(i): statistical target.}
A-SM and the decoupled surrogate are aligned on the conditional target, since
both recover the Bayes-sufficient quantities at the population optimum: a
categorical class posterior together with bounded expert probabilities. The
difference between them is therefore not about whether they estimate the right
object, but about how they organize the optimization problem that reaches that
object.

The remaining gap between A-SM and the decoupled surrogate therefore lies on
axis~(ii).

\paragraph{Axis~(ii): local geometry.}
A-SM still routes the expert objectives through the class scores, so the class
gradient contains the expert-leakage term of
Proposition~\ref{prop:asm-coupling}. The decoupled surrogate removes this term
completely: its class gradient depends only on $(p,y)$, and its mixed Hessian
block is zero. This leakage is the key residual weakness of A-SM, which has
already repaired the target distortion and boundedness problems but has not
fully separated the class and expert subproblems. Among the methods compared
here, the decoupled surrogate combines this
zero mixed block with a categorical class posterior; OvA also has a zero
mixed block and aligned unrestricted targets.

The practical consequence is that the two methods react differently during
training. At fixed scores under the decoupled loss, an error in expert~$j$
changes only its own score derivative, whereas under A-SM the same expert error
also perturbs the class update whenever the non-max class normalizer is active.
Even though both methods aim for the same conditional optimum, they therefore
do not follow the same local route toward it: A-SM reaches the correct target
through coupled class and expert updates, whereas the decoupled surrogate
reaches it with the two sets of updates kept separate.

\subsection{Worked Example: Complementary Experts}

Consider $K=3$ classes and $J=4$ complementary experts. Fix a region where the
maximizing class index is $k^\star=1$, and inspect the gradient on the class-2
logit $a_2$ for a sample with $y=2$. Suppose
\[
\xi_2=0.25,
\qquad
\pi_2=0.50,
\qquad
\sum_{j=1}^4\bigl(\psi_j-\mathbf{1}\{m_j=y\}\bigr)=0.80,
\]
so that $\xi_3=0.25$ and $\xi_1=0.50$, consistent with $k^\star=1$. Then
Proposition~\ref{prop:asm-coupling} gives
\[
\frac{\partial\Phi^{\mathrm{ASM}}}{\partial a_2}
=
(\xi_2-1)-\pi_2\sum_{j=1}^4\bigl(\psi_j-\mathbf{1}\{m_j=y\}\bigr)
=
-0.75-0.40
=
-1.15.
\]
The first term, $-0.75$, is the ordinary class signal saying that class~2 is
the true label and its current probability estimate is too small. The second
term, $-0.40$, carries no class information: it is leakage from the expert
heads, which appears only because A-SM routes every expert residual through the
non-max class normalizer, and it becomes larger as the aggregate expert
residual grows.

If A-SM were fully decoupled, the class-2 logit would receive exactly the
ordinary cross-entropy gradient $-0.75$. Instead, the expert block adds $-0.40$ to this coordinate and the same
$-0.40$ to the class-3 logit, whose derivative flips from $\xi_3=0.25$ to
$-0.15$, so the class gradient moves from $(0.50,-0.75,0.25)$ to
$(0.50,-1.15,-0.15)$ and changes in both magnitude and direction. The class learner is therefore not reacting purely to class error on
this sample; it is also reacting to unresolved expert uncertainty.

Under the decoupled surrogate, the same sample would have class gradient
\[
\frac{\partial\Phi^{\mathrm{dec}}}{\partial w_2}=p_2-1,
\]
with no expert contribution, so the class update would reflect only the
classifier's own error on the sample. The example thus illustrates the
axis-(ii) difference between A-SM and the decoupled surrogate numerically:
both can target the right optimum, but the decoupled loss keeps the
class-score derivative free of expert residuals.

The example also shows why the issue can worsen with many experts. The
leakage enters through the aggregate residual
$\sum_{j=1}^J(\psi_j-\mathbf{1}\{m_j=y\})$, so when several experts are still
in their uncertain regime, their residuals add before being injected into the
class gradient. The decoupled surrogate has no analogous accumulation, because
each expert head receives its own score derivative.

The conclusion for A-SM is therefore narrower than for additive CE or PiCCE:
fixing the statistical target is not enough, since one can correct
boundedness and recover the right population quantities yet still retain a
shared local geometry that mixes class and expert learning. The decoupled
surrogate closes that final gap by keeping the same target alignment while
removing the mixed interaction altogether.

\section{OvA \texorpdfstring{\citep{Verma2022LearningTD}}{}}

The multi-expert one-vs-all construction of \citet{Verma2022LearningTD} is the
simplest augmented baseline on the optimization axis. It removes the shared
softmax entirely and replaces it by independent binary logistic tasks, so
none of the earlier geometric pathologies remain: there is no
multiplicity amplification, no winner-take-all starvation, and no mixed
class--expert curvature.

The remaining question is what is lost when the class side is
represented not by one categorical distribution in $\Delta^K$, but by $K$
independent Bernoulli scores. At the conditional optimum, OvA still identifies
the correct class marginals $\eta_k(x)$ and expert utilities $\alpha_j(x)$, and
the gap appears only away from that optimum. For arbitrary parameters, the
class vector produced by OvA need not lie on the simplex, so the largest class
score is not automatically the confidence of a coherent categorical model.

\subsection{Formulation}

OvA separates the routing problem into $K+J$ binary tasks: one logistic problem
for each class and one logistic problem for each expert. The class coordinates
ask whether the true label equals a given class, while the expert coordinates
ask whether a given expert is correct on the current example.

\begin{definition}[Multi-expert OvA, \citet{Verma2022LearningTD}]
\label{def:appendix-ova}
Let $g(x)\in\mathbb{R}^K$ denote class scores and
$s(x)\in\mathbb{R}^J$ denote expert scores. Define
\[
\chi_k(y)\coloneqq 2\mathbf{1}\{y=k\}-1,
\qquad
\zeta_j(y,m)\coloneqq 2\mathbf{1}\{m_j=y\}-1,
\]
and
\[
\Phi^{\mathrm{OvA}}(g,s;\,x,y,m)
\coloneqq
\sum_{k=1}^K \log\!\bigl(1+\exp(-\chi_k(y)g_k(x))\bigr)
\;+\;
\sum_{j=1}^J \log\!\bigl(1+\exp(-\zeta_j(y,m)s_j(x))\bigr).
\]
\end{definition}

This is the signed-label form used in \citet{Verma2022LearningTD}. For the
analysis below it is convenient to rewrite the same loss using the binary
targets
\[
b_k(y)\coloneqq \mathbf{1}\{y=k\},
\qquad
t_j\coloneqq \mathbf{1}\{m_j=y\}.
\]
Using the standard identity
\[
\log\!\bigl(1+\exp(-(2b-1)v)\bigr)
=
-b\log \sigma(v) -(1-b)\log(1-\sigma(v)),
\qquad b\in\{0,1\},
\]
the OvA loss can be written equivalently as
\[
\Phi^{\mathrm{OvA}}(g,s;\,x,y,m)
=
\sum_{k=1}^K
\Bigl(
-b_k(y)\log \sigma(g_k(x))
-(1-b_k(y))\log(1-\sigma(g_k(x)))
\Bigr)
\]
\[
\qquad\qquad
+\sum_{j=1}^J
\Bigl(
-t_j\log \sigma(s_j(x))
-(1-t_j)\log(1-\sigma(s_j(x)))
\Bigr).
\]
This form makes two features explicit: every coordinate is optimized
independently, and the class head is not a softmax over the $K$ classes but
a vector of $K$ separate sigmoids.

The prediction rule chooses the classifier if
$\max_{k\in[K]} g_k(x)>\max_{j\in[J]} s_j(x)$ and otherwise defers to the
largest expert score. Since the sigmoid is strictly increasing, this is
equivalent to comparing $\max_k \sigma(g_k(x))$ and $\max_j \sigma(s_j(x))$ on
the probability scale.

\subsection{Conditional Minimizer}
\label{app:proof-ova-bernoulli}

\begin{proposition}[OvA conditional minimizer]
\label{prop:ova-bernoulli}
The conditional OvA risk is a sum of $K+J$ independent Bernoulli
cross-entropies. Its optimal probabilities satisfy
$\sigma(g_k^\star(x))=\eta_k(x)$ and
$\sigma(s_j^\star(x))=\alpha_j(x)$ whenever the corresponding targets lie in
$(0,1)$. If a target equals $0$ or $1$, the conditional infimum is approached
as the corresponding logit tends to $-\infty$ or $+\infty$, respectively.
\end{proposition}

\begin{proof}
Fix an input $x$. Taking conditional expectation over $(Y,\mathbf{M})\mid X=x$ gives
\[
\mc{C}_{\mathrm{OvA}}(g,s\mid x)
=
\sum_{k=1}^K
\Bigl(
-\eta_k(x)\log \sigma(g_k)
-(1-\eta_k(x))\log(1-\sigma(g_k))
\Bigr)
\]
\[
\qquad\qquad
+\sum_{j=1}^J
\Bigl(
-\alpha_j(x)\log \sigma(s_j)
-(1-\alpha_j(x))\log(1-\sigma(s_j))
\Bigr).
\]
This expression separates completely across coordinates, with one Bernoulli
cross-entropy for each class coordinate and one for each expert coordinate.
The class-side target for coordinate $k$ is $\eta_k(x)$, while the
expert-side target for coordinate $j$ is $\alpha_j(x)$.

To make the argument explicit, consider the generic Bernoulli cross-entropy
with target parameter $\theta\in[0,1]$:
\[
\varphi(v)
=
-\theta\log \sigma(v)-(1-\theta)\log(1-\sigma(v)),
\]
where $v$ is the Bernoulli logit. Differentiating once gives
\[
\varphi'(v)=\sigma(v)-\theta.
\]
Thus any finite stationary point must satisfy $\sigma(v)=\theta$.
Differentiating again gives
\[
\varphi''(v)=\sigma(v)(1-\sigma(v))>0,
\]
since $\sigma(v)\in(0,1)$ for every finite $v$, so $\varphi$ is strictly
convex. If $\theta\in(0,1)$, the condition $\sigma(v)=\theta$
identifies the unique finite minimizer. If $\theta=1$, then
$\varphi'(v)=\sigma(v)-1<0$ for every finite $v$, so the risk decreases
monotonically and its infimum is approached as $v\to+\infty$. Similarly, if
$\theta=0$, then $\varphi'(v)=\sigma(v)>0$ for every finite $v$, so the infimum
is approached as $v\to-\infty$.

Applying this argument coordinatewise yields
\[
\sigma(g_k^\star(x))=\eta_k(x),
\qquad
\sigma(s_j^\star(x))=\alpha_j(x),
\]
for all $k\in[K]$ and $j\in[J]$ whose targets lie in $(0,1)$, with the
boundary cases understood as the limits above.
\end{proof}

This proposition is the positive side of OvA: unlike additive CE or PiCCE, OvA
does not distort the conditional target. At the probability level, it
identifies the correct class-marginal quantities $\eta_k(x)$ and the correct
expert utilities $\alpha_j(x)$. Its limitation is therefore not that it aims
at the wrong population object but representational: the
class head becomes a coherent categorical posterior only at the optimum, not by
construction at intermediate parameter values.

\subsection{Curvature}
\label{app:proof-ova-curvature}

The local geometry of OvA is as simple as its conditional risk. Because each
loss term depends on only one score coordinate, the first and second
derivatives separate completely.

Consider first a class coordinate $g_k$. Its samplewise contribution is
\[
L_k(g_k;y)
=
-b_k(y)\log \sigma(g_k)
-(1-b_k(y))\log(1-\sigma(g_k)).
\]
Differentiating with respect to $g_k$ gives
\[
\frac{\partial L_k}{\partial g_k}
=
\sigma(g_k)-b_k(y),
\]
and differentiating once more gives
\[
\frac{\partial^2 L_k}{\partial g_k^2}
=
\sigma(g_k)(1-\sigma(g_k)).
\]

Exactly the same calculation applies to an expert coordinate $s_j$:
\[
L_j(s_j;y,m)
=
-t_j\log \sigma(s_j)
-(1-t_j)\log(1-\sigma(s_j)),
\]
so
\[
\frac{\partial L_j}{\partial s_j}
=
\sigma(s_j)-t_j,
\qquad
\frac{\partial^2 L_j}{\partial s_j^2}
=
\sigma(s_j)(1-\sigma(s_j)).
\]

No class term depends on any other class coordinate or on any
expert coordinate, and no expert term depends on any class coordinate or any
other expert coordinate. Therefore every mixed partial derivative is zero:
\[
\frac{\partial^2 \Phi^{\mathrm{OvA}}}{\partial g_k\,\partial g_r}=0
\quad (k\neq r),
\qquad
\frac{\partial^2 \Phi^{\mathrm{OvA}}}{\partial s_j\,\partial s_\ell}=0
\quad (j\neq \ell),
\qquad
\frac{\partial^2 \Phi^{\mathrm{OvA}}}{\partial g_k\,\partial s_j}=0.
\]
If we concatenate all coordinates into one vector
$a=(g_1,\dots,g_K,s_1,\dots,s_J)\in\mathbb{R}^{K+J}$, the Hessian is therefore
diagonal:
\[
\nabla_a^2\Phi^{\mathrm{OvA}}
=
\mathrm{Diag}\bigl(\sigma(a_i)(1-\sigma(a_i))\bigr),
\]
where the diagonal entries corresponding to class coordinates are
$\sigma(g_k)(1-\sigma(g_k))$ and those corresponding to expert coordinates are
$\sigma(s_j)(1-\sigma(s_j))$.

The eigenvalues of a diagonal matrix are exactly its diagonal entries. Since
$0\le \sigma(v)(1-\sigma(v))\le 1/4$ for every $v\in\mathbb{R}$, we obtain
\[
\lambda_{\max}\!\bigl(\nabla_a^2\Phi^{\mathrm{OvA}}\bigr)\le \frac14.
\]

Any weakness of OvA therefore lies outside its local optimization geometry.

\subsection{Head-to-Head with the Decoupled Surrogate}

\paragraph{Axis~(i): statistical target.}
OvA and the decoupled surrogate agree at the population level: at the
conditional optimum, both methods recover the class marginals $\eta_k(x)$ and
the expert utilities $\alpha_j(x)$. In that sense, OvA is much closer to the
decoupled surrogate than the earlier augmented-softmax surrogates are.

The difference lies in how these quantities are represented away from the
optimum. The decoupled surrogate parameterizes the class side as one
categorical distribution $p(x)\in\Delta^K$ for every parameter value, so the
vector $p(x)$ is a coherent class posterior throughout training. OvA instead parameterizes the
class side as $K$ independent sigmoids. Each coordinate can be interpreted as a
class marginal, but the full vector need not lie in $\Delta^K$ unless the model
happens already to be at the conditional optimum.

\paragraph{Axis~(ii): local geometry.}
On the geometric axis the two methods are similar: both keep the class
block separate from the expert block, both have bounded curvature (softmax
covariance with $\lambda_{\max}\le1/2$ for the decoupled class block, diagonal
entries at most $1/4$ for OvA), and both avoid all of the
shared-geometry pathologies that affect the augmented-action family. The
decoupled surrogate is therefore not preferred because it is more decoupled
than OvA, since there is little to improve on that axis.

The decoupled surrogate is preferred because it combines this geometry with
a categorical class model. Its routing rule always compares $\max_k p_k(x)$ and
$\max_j u_j(x)$ on a coherent probability scale. OvA also compares quantities in
$[0,1]$, but the class side is not forced to behave like one categorical
posterior during training. Several class sigmoids can be simultaneously large,
and the largest one can overstate classifier confidence relative to the expert
utilities.

\subsection{\texorpdfstring{Worked Example: Several Class Sigmoids Above $0.5$}{Worked Example: Several Class Sigmoids Above 0.5}}

Consider $K=3$ and suppose that during training the OvA class head outputs
\[
\sigma(g_1)=0.77,\qquad \sigma(g_2)=0.73,\qquad \sigma(g_3)=0.69.
\]
These values are individually admissible Bernoulli probabilities, but together
they cannot form a categorical posterior because they sum to $2.19$. Nothing in
the OvA objective penalizes this: each class coordinate is trained separately,
so the model is free to regard several classes as simultaneously likely in the
one-vs-all sense.

Now suppose that the best expert estimate is
\[
\sigma(s_1)=0.70.
\]
The OvA routing rule compares the largest class score with the best expert
score. Since
\[
\max_{k\in[K]} \sigma(g_k)=0.77 > 0.70 = \max_{j\in[J]} \sigma(s_j),
\]
OvA predicts rather than defers.

If the decoupled class head had the same three class logits $g$, its softmax
would give
\[
p\approx(0.40,0.33,0.27),
\qquad
u_1=0.70,
\]
and the comparison $\max_k p_k\approx0.40<0.70$ leads the decoupled surrogate
to defer. This does not
mean that OvA is inconsistent, because if the model class is rich enough and
training finds the conditional optimum, the OvA class sigmoids will match the
true class-marginal probabilities and therefore sum to one. The point is
instead that the intermediate states visited during training need not have
that property, and because routing uses the current scores, this lack of
simplex structure can matter in practice.

The example illustrates the remaining gap between OvA and the decoupled
surrogate. OvA has already fixed the geometric side of the problem but not
the representation of class uncertainty, whereas the decoupled surrogate keeps
the class--expert separation while ensuring that the class side remains a
categorical posterior at every point along training.

\section{Comparison of Excess-Risk Constants}\label{app:constant-comparison}

We compare the surrogate-regret transfer guarantees of the decoupled
surrogate with those available for the baselines, focusing on how the
calibration constant scales with the size of the multi-expert action space.
Two prior references need to be separated carefully. The constant
$\sqrt{2(J+1)}$ comes from the earlier comp-sum reduction analyzed by
\citet{mao2024principledapproacheslearningdefer}, which covers the
Mozannar-style augmented-action construction; it is not the constant for
the later Mao25 single-stage surrogate \citep{mao2025mastering}. For the
Mao25 single-stage surrogate used in this paper, the relevant excess-risk
factor scales with the augmented action dimension $K+J$.

\paragraph{The decoupled surrogate (Theorem~\ref{thm:dec-consistency}).}
The decoupled surrogate satisfies
\[
\mc{E}_\perp-\mc{E}_\perp^\star
\;\le\;
\underbrace{\max\!\left\{\sqrt{2},\;\sqrt{\tfrac{2J}{\lambda}}\right\}}_{C_{\mathrm{dec}}}
\sqrt{
\mc{E}_{\mathrm{dec}}-\mc{E}_{\mathrm{dec}}^\star}\,.
\]
Since $\sqrt{2J/\lambda}=\sqrt{2/(\lambda/J)}$, the constant depends on $J$
only through the per-expert weight $\lambda/J$; whenever
$\lambda/J\ge 1$, it equals $C_{\mathrm{dec}}=\sqrt{2}\approx 1.41$.

\paragraph{Add.\ CE \citep{mozannar2021consistent}.}
For Add.\ CE, \citet{mao2024principledapproacheslearningdefer} derive the
transfer constant $C_{\mathrm{CE}}=\sqrt{2(J+1)}$, which grows as
$O(\!\sqrt{J})$; this is the factor quoted in the main paper for the
Add.\ CE baseline.

\paragraph{Mao25
\citep{mao2025mastering}.}
For the single-stage surrogate studied by \citet{mao2025mastering}, the
corresponding excess-risk constant scales with the full augmented action
dimension. In the multi-expert setting considered here, this yields
\[
C_{\mathrm{Mao25}} = K+J,
\]
which multiplies the surrogate excess itself rather than its square root, so
the dependence on the number of experts is linear: $O(J)$ for fixed $K$.

\paragraph{OvA and the categorical-head distinction.}
Proposition~\ref{prop:revision-ova-bound} gives OvA the quantitative
coefficient $\sqrt2$ for unit per-coordinate Bernoulli weights, so
$J$-independence is not specific to the decoupled surrogate and is shared by
this other separated surrogate. For the decoupled loss the expert
weights are $\lambda/J$; at $\lambda/J=1$, which matches the unit Bernoulli
weights of OvA, both coefficients equal $\sqrt2$. The distinction between the two lies in the
categorical class constraint, which the independent OvA class sigmoids need
not satisfy away from their optimum.

\paragraph{The linear Mao25 instance.}
For the acceptable-mass loss actually analyzed here, a direct calculation
also gives a \emph{linear} transfer with factor $K+J$.
Let $\mu=(\eta,\alpha)$, $\mu_\star=\max_i\mu_i$, and let
$\hat\imath\in\argmax_iq_i$. Its conditional surrogate excess is
$\sum_i(\mu_\star-\mu_i)q_i$, which is at least
$(\mu_\star-\mu_{\hat\imath})q_{\hat\imath}$.
Since $q_{\hat\imath}\ge1/(K+J)$, the conditional defer regret of the plug-in
action $\hat\imath$ is at most $(K+J)$ times this excess, and taking expectation
preserves the linear bound. This statement concerns the MAE-like member, not
every member of its surrogate family.

\paragraph{Summary.}
Write $\Delta\mathcal E$ for excess over the unrestricted population
infimum of the indicated loss. Because the guarantees have different transfer
forms, the table below displays each form together with its loss
normalization:
\begin{center}
\small
\begin{tabular}{@{}lll@{}}
\toprule
Surrogate & Transfer to $\Delta\mathcal E_\perp$ & Normalization \\
\midrule
Add.\ CE & $\sqrt{2(J+1)}\sqrt{\Delta\mathcal E_{\mathrm{CE}}}$ & summed CE \\
Mao25, linear member & $(K+J)\Delta\mathcal E_{\mathrm{Mao}}$ & $1-S_{\Sset}$ \\
OvA & $\sqrt2\sqrt{\Delta\mathcal E_{\mathrm{OvA}}}$ & unit Bernoulli weights \\
Decoupled & $\max\{\sqrt2,\sqrt{2J/\lambda}\}
\sqrt{\Delta\mathcal E_{\mathrm{dec}}}$ & expert weight $\lambda/J$ \\
\bottomrule
\end{tabular}
\end{center}
For PiCCE and A-SM, we use the cited Bayes-consistency results without
asserting a new quantitative coefficient. The unrestricted transfer
statements do not constitute restricted-class $\mathcal H$-consistency
bounds; those require the relevant minimizability gaps and hypotheses.

\section{Further Consequences of the Decoupled Geometry}
\label{app:revision-theory}

The following results develop the same two objects used in the main paper,
the class posterior and the marginal expert utilities. They distinguish
sharing a representation from sharing an action normalizer, and explain
why the utility-based objective is compatible with either training schedule.
All logarithms are natural, conditional identities hold $P_X$-almost surely
for a version of the conditional probabilities, and ties are broken by a
fixed measurable rule. Log losses on closed probability domains use the
convention $0\log0=0$, while a positive target assigned zero estimated
probability has infinite loss, and boundary optima need not be attained by
finite logits.

\subsection{Shared Trunk, Separate Head Parameters}
\label{app:revision-sharing}

Write $F=\Phi_{\mathrm{cls}}$ and $G=\Phi_{\mathrm{exp}}$, including
the expert weight in $G$. Suppose
$w=w(\theta,\theta_c)\in\mathbb R^K$ and
$s=s(\theta,\theta_e)\in\mathbb R^J$ are twice differentiable.
The trunk parameters $\theta$ are shared, while $w$ does not depend on
$\theta_e$ and $s$ does not depend on $\theta_c$, so the loss is exactly
\[
\Phi(\theta,\theta_c,\theta_e)
=F(w(\theta,\theta_c))+G(s(\theta,\theta_e)).
\]
Consequently, its parameter Hessian has the form
\[
\nabla^2\Phi=
\begin{pmatrix}
F_{\theta\theta}+G_{\theta\theta}&F_{\theta\theta_c}&G_{\theta\theta_e}\\
F_{\theta_c\theta}&F_{\theta_c\theta_c}&0\\
G_{\theta_e\theta}&0&G_{\theta_e\theta_e}
\end{pmatrix},
\]
where subscripts denote derivatives of the composed scalar losses.
This proves the parameter-level part of
Proposition~\ref{prop:due-gradient}: differentiating the class loss with
respect to the expert-head parameters, or the expert loss with respect to
the class-head parameters, gives zero, so the head--head block vanishes even
when the trunk is learned.

For the trunk, the chain rule gives
\[
\nabla_\theta\Phi
=(D_\theta w)^\top(p-e_y)
+\frac{\lambda}{J}(D_\theta s)^\top(u-t).
\]
Each residual is the same residual its estimation task would supply on
its own, and the loss introduces neither a shared class--expert normalizer
nor an additional product of class and expert residuals. The two
contributions can still have a negative inner product, and a trunk update
from either task can change both score blocks. Thus direct loss coupling
is removed while joint representation learning is retained. The
initialization diagnostic in Appendix~\ref{app:revision-ablation}
measures the remaining trunk interaction rather than assuming it vanishes.

\subsection{OvA Also Has a Dimension-Independent Transfer}
\label{app:revision-ova}

\begin{proposition}[OvA surrogate-regret transfer]
\label{prop:revision-ova-bound}
Let $v:\mathcal X\to[0,1]^K$ and
$u:\mathcal X\to[0,1]^J$ be measurable OvA estimates, trained with
unit weight on every class and expert Bernoulli log loss. Select a class
maximizing $v_k$ if $\max_kv_k\ge\max_ju_j$, and otherwise an expert
maximizing $u_j$. Then
\[
\mathcal E_\perp(f_{v,u})-\mathcal E_\perp^\star
\le\sqrt2\sqrt{\mathcal E_{\mathrm{OvA}}(v,u)
-\mathcal E_{\mathrm{OvA}}^\star}.
\]
The coefficient is independent of both $K$ and $J$.
\end{proposition}
\begin{proof}
Let $\delta_{\mathrm{OvA}}(x)$ denote the conditional OvA surrogate excess
at $x$. Conditional expectation and the binary cross-entropy identity give
\[
\delta_{\mathrm{OvA}}(x)
=\sum_{k=1}^K\mathrm{KL}\bigl(\mathrm{Bern}(\eta_k(x))\Vert\mathrm{Bern}(v_k(x))\bigr)
+\sum_{j=1}^J\mathrm{KL}\bigl(\mathrm{Bern}(\alpha_j(x))\Vert\mathrm{Bern}(u_j(x))\bigr).
\]
Bernoulli Pinsker implies that each coordinate error is at most
$\sqrt{\delta_{\mathrm{OvA}}(x)/2}$. The four-case comparison in
\eqref{eq:due-regret-via-errors} uses only the maxima of the class and
expert estimates; it does not require the class estimates to sum to one.
Therefore $\delta_\perp(x)\le\sqrt{2\delta_{\mathrm{OvA}}(x)}$.
Taking expectation and applying Jensen's inequality proves the claim.
Infinite surrogate excess makes the inequality vacuous but valid.
\end{proof}

The bound extends our proof technique to OvA, without asserting that the
categorical constraint is necessary for Bayes consistency. OvA recovers
the same class and expert probabilities at its unrestricted optimum,
while the decoupled class head additionally enforces $\sum_kp_k=1$ at every
parameter value. This is a structural distinction between the function
parameterizations, not a theorem that every softmax estimate is better
calibrated than every OvA estimate.

\subsection{Cross-Entropy Reductions and the Comp-Sum Extension}
\label{app:revision-compsum}

Proposition~\ref{prop:additive-gradient} gives the comp-sum geometry and
its overlap dependence. We extend it here to nonnegative weights and then
to arbitrary increasing comp-sum outer functions. Let
$d\ge2$, $z\in\mathbb R^d$, $q=\softmax(z)$, and let
$v\in[0,\infty)^d$ be finite, score-independent action weights. Set
$V=\sum_av_a$ and $C(q)=\mathrm{Diag}(q)-qq^\top$.
For $L_v(z)=-\sum_av_a\log q_a$,
\begin{equation}\label{eq:revision-weighted-ce}
\nabla_zL_v=Vq-v,\qquad \nabla_z^2L_v=VC(q).
\end{equation}
Indeed, $L_v=-v^\top z+V\log\sum_be^{z_b}$, so differentiating proves
both identities. If $V>0$, we also have
$L_v=V\operatorname{CE}(v/V,q)$, and if $V=0$, the loss and both derivatives
are zero. The calculation applies to any such weighted-CE surrogate,
including bounded soft-utility targets and frozen-predictor routers;
it does not assume that the weights are binary.

For Add.\ CE, define
$\mathcal T=\{y\}\cup\{K+j:j\in\Jset(y,m)\}$ and
$N=|\mathcal T|=1+|\Jset|$. Taking $v_a=\mathbf1\{a\in\mathcal T\}$
gives $L=N\operatorname{CE}(\mathbf1_{\mathcal T}/N,q)$.
The normalized target places $1/N$ on the true class and $|\Jset|/N$
in total on the correct expert actions, and no incorrect class receives
target mass. At the conditional optimum the total class mass is
$1/(1+\sum_j\alpha_j)$, and every raw augmented probability margin is
contracted by this denominator. Renormalizing the class block recovers
$\eta$ exactly at the optimum, and the common denominator preserves
the Bayes action ordering. Thus this augmented-action smoothing is
compatible with Fisher consistency; it concerns the raw probability
representation and sample weighting.

Now consider the comp-sum base family
\citep{mao2023crossentropylossfunctionstheoretical,mao2024principledapproacheslearningdefer}:
\[
\ell_\tau(z,a)=
\begin{cases}
(q_a^{\tau-1}-1)/(1-\tau),&\tau\ne1,\\
-\log q_a,&\tau=1,
\end{cases}
\qquad \tau\ge0.
\]
For the same score-independent nonnegative weights, set
$L_{\tau,v}=\sum_av_a\ell_\tau(z,a)$,
$b_a=v_aq_a^{\tau-1}$, and $B_\tau=\sum_ab_a$.

\begin{proposition}[Exact weighted comp-sum geometry]
\label{prop:revision-compsum}
For finite logits and every $\tau\ge0$,
\begin{align}
\nabla_zL_{\tau,v}&=B_\tau q-\sum_ab_ae_a,
\label{eq:revision-compsum-gradient}\\
\nabla_z^2L_{\tau,v}
&=B_\tau C(q)
-(\tau-1)\sum_ab_a(q-e_a)(q-e_a)^\top.
\label{eq:revision-compsum-hessian}
\end{align}
When $v=\mathbf1_{\mathcal T}$, define
$\bar q_\tau=N^{-1}\sum_{a\in\mathcal T}q_a^{\tau-1}$,
$\omega_a=q_a^{\tau-1}/(N\bar q_\tau)$,
$\bar e_\tau=\sum_{a\in\mathcal T}\omega_ae_a$, and
$\bar R_\tau=\sum_{a\in\mathcal T}\omega_a(q-e_a)(q-e_a)^\top$.
Then
\[
\nabla L_{\tau,v}=N\bar q_\tau(q-\bar e_\tau),\qquad
\nabla^2L_{\tau,v}
=N\bar q_\tau[C(q)-(\tau-1)\bar R_\tau].
\]
\end{proposition}
\begin{proof}
Since $\nabla\log q_a=e_a-q$, differentiating the base loss gives
$\nabla\ell_\tau=q_a^{\tau-1}(q-e_a)$, including its continuous
limit at $\tau=1$. Differentiate this product, using
$\nabla q=C(q)$ and
$\nabla q_a^{\tau-1}=(\tau-1)q_a^{\tau-1}(e_a-q)$.
The result is
\[
\nabla^2\ell_\tau
=q_a^{\tau-1}[C(q)-(\tau-1)(q-e_a)(q-e_a)^\top].
\]
Summing with weights $v_a$ proves the first two identities; dividing
the active weights by their positive sum proves the factorization.
\end{proof}

Every omitted action $k$ has
derivative $B_\tau q_k>0$ when some weight is positive, so for every member
of this family an omitted correct expert is suppressed by a winner-take-all
selection on every region with a fixed winner (at score ties, these are
selected-branch derivatives). At $\tau=1$, $B_1=V$ is score-independent and
the correction term vanishes, recovering the exact weighted-CE
result. For $0\le\tau\le1$, both Hessian contributions are positive
semidefinite, whereas for $\tau>1$ the correction is negative semidefinite,
so positive curvature accumulation is no longer guaranteed.
For example, with $d=2$, $v=e_1$, and $\tau>1$, the Hessian is
\[
q_1^{\tau-1}(1-q_1)[\tau q_1-(\tau-1)]
\begin{pmatrix}1&-1\\-1&1\end{pmatrix}.
\]
Its contrast eigenvalue is negative whenever
$0<q_1<(\tau-1)/\tau$, explicitly exhibiting negative curvature.
At $\tau=2$ and binary active weights,
$L_{2,v}=N-\sum_{a\in\mathcal T}q_a$ and
$B_2=\sum_{a\in\mathcal T}q_a\le1$.
Thus writing $B_\tau=N\bar q_\tau$ exposes the active-target count
without implying global linear amplification for every $\tau$.
The MAE endpoint is precisely the set-mass geometry studied for Mao25,
up to the score-independent additive constant $N-1$.

\paragraph{Arbitrary increasing comp-sum outer functions.}
The competition mechanism also extends beyond the one-parameter family.
Let $\ell_\Upsilon(z,i)=\Upsilon(\sum_{r\ne i}e^{z_r-z_i})$, where
$\Upsilon$ is differentiable and $\Upsilon'(\nu)>0$ for $\nu>0$.
Writing $\nu_i=1/q_i-1$ gives
\[
\nabla_z\ell_\Upsilon(z,i)=\frac{\Upsilon'(\nu_i)}{q_i}(q-e_i).
\]
For score-independent nonnegative weights $v$, define
$b_i=v_i\Upsilon'(\nu_i)/q_i$ and $B=\sum_i b_i$.
Then $\nabla_z L_{\Phi,v}=Bq-\sum_i b_i e_i$.
If at least one weight is positive, an omitted coordinate $k$ has
$\partial L_{\Phi,v}/\partial z_k=Bq_k>0$.
Thus suppression of omitted actions follows from increasing comp-sum
structure rather than from a logarithmic outer function; we do not make
this sign claim for arbitrary multiclass base losses.

\subsection{Two-Stage Training and Predictor-Dependent Routers}
\label{app:revision-two-stage}

With disjoint parameters $(\theta_p,\theta_u)$ in a product feasible
set, write the empirical objective as
\[
\widehat{\mathcal E}_{\mathrm{dec}}(\theta_p,\theta_u)
=\widehat{\mathcal E}_{\mathrm{cls}}(\theta_p)
+\lambda\widehat{\mathcal E}_{\mathrm{exp}}(\theta_u),
\]
where the expert term is mean BCE over examples and experts. Since
$\lambda>0$, a pair minimizes this sum if and only if each component
minimizes its summand, whenever finite component minima exist.
Sequential training is therefore an alternative implementation of the
same separable learning objective. Parameter constraints or regularizers
that couple the blocks would invalidate this product argument.

For any fixed measurable classifier $h$, let
$a_h=\eta_{h(x)}(x)$, $\eta_{\max}=\max_k\eta_k(x)$, and
$\alpha_{\max}=\max_j\alpha_j(x)$. Its best achievable conditional
defer risk is $1-\max\{a_h,\alpha_{\max}\}$, so freezing $h$ leaves the
exact gap
\begin{equation}\label{eq:revision-frozen-gap}
\delta_h(x)=\max\{\eta_{\max},\alpha_{\max}\}
-\max\{a_h,\alpha_{\max}\}\ge0,
\end{equation}
which follows by minimizing over the $J+1$ available action risks.
The gap is zero when $h$ is Bayes or an expert is at least as accurate
as the best class action. It applies to any frozen-predictor router,
including a two-stage implementation of our own model; it is not a
claim that two-stage learning is necessarily inferior.

\paragraph{The 2023 fixed-score router.}
In the score-based construction of \citet{mao2023twostage}, fix
$z_0(x)=\max_kw_k^p(x)$ and the induced classifier $h_p$, and learn
$z_1,\ldots,z_J$. For the logistic instance,
$q=\softmax(z_0,\ldots,z_J)$ and
$L_{23}=-t_0\log q_0-\sum_jt_j\log q_j$,
where $t_0=\mathbf1\{h_p(x)=y\}$ and $t_j=\mathbf1\{m_j=y\}$.
Equation~\eqref{eq:revision-weighted-ce} gives, for $i,j\in[J]$,
\[
\frac{\partial L_{23}}{\partial z_j}
=(t_0+|\Jset|)q_j-t_j,\qquad
\frac{\partial^2L_{23}}{\partial z_i\partial z_j}
=(t_0+|\Jset|)(q_i\mathbf1\{i=j\}-q_iq_j).
\]
With $\beta_0=a_{h_p}$, $\beta_j=\alpha_j$, and
$Z=\sum_{a=0}^J\beta_a>0$, conditional excess is
$Z\mathrm{KL}(\beta/Z\Vert q)$, since conditional expectation of the
loss is $-\sum_a\beta_a\log q_a$. The optimum in the closed simplex
is $q_a^\star=\beta_a/Z$; zero weights are interpreted by limits.
If $Z=0$, every action has zero correctness utility and all router
scores minimize the conditional surrogate. Adding a positive-utility
expert changes the normalization of every existing optimum, whereas
an existing Bernoulli utility target remains $\alpha_j$.
Shifting every frozen class logit by the same function leaves the class
softmax and prediction unchanged but shifts the fixed router
coordinate $z_0$ alone. An unrestricted router can compensate, whereas a
restricted router need not; our probability-space comparison is invariant
to this shift.

\paragraph{The predictor--rejector formulation used in our experiments.}
The regression/cost-sensitive construction of
\citet{mao2024regressionmultiexpertdeferral} uses a learned score for
every router action. With $\ell_h=\mathbf1\{h(x)\ne y\}$,
$c_j=\mathbf1\{m_j\ne y\}$, and $C=\sum_jc_j$, its logistic action
weights are $v_0=C$ and $v_j=\ell_h+C-c_j$.
Their sum is $V=J(\ell_h+C)$, so the same weighted-CE calculation gives
gradient $Vq-v$ and Hessian $VC(q)$. At the population level, let
$\rho_0=1-a_h$, $\rho_j=1-\alpha_j$, and $T=\sum_{a=0}^J\rho_a$.
When $T>0$,
\[
q_a^\star=\frac{T-\rho_a}{JT}.
\]
This optimum orders actions by increasing conditional risk and therefore
selects the frozen-predictor Bayes action, although $q_a^\star$ is not the
action utility $1-\rho_a$. For positive optimal weights and
$\rho_a<\rho_b$,
\[
0<z_a^\star-z_b^\star
=\log\frac{T-\rho_a}{T-\rho_b}
\le\frac{\rho_b-\rho_a}{T-\rho_b},
\]
using $\log(1+r)\le r$. If $T\ge cJ$ for a fixed $c>0$ and the
action risks lie in $[0,1]$, this margin is $O(1/J)$ as $J\to\infty$
for a fixed action-risk gap. For two expert actions, the Bernoulli optimum
instead preserves the utility difference $\alpha_a-\alpha_b$. When $T=0$,
every action has zero error and the surrogate loss is identically zero.
These output-level identities complement the restricted-class and
realizable-consistency guarantees of the two-stage literature; our
unrestricted theorem does not subsume those guarantees.

\subsection{Consultation Costs and Class-Weighted Risk}
\label{app:revision-costs}

For any measurable expert cost $c_j\in[0,B]$, with fixed $B\ge1$,
train the expert target $\tilde t_j=1-c_j/B$. By the same binary
cross-entropy decomposition, its conditional optimum is
$u_j^\star=1-\mathbb E[c_j\mid X=x]/B$. Class action $k$ has normalized
utility $1-(1-\eta_k)/B$, so routing compares
$(B-1+\max_kp_k)/B$ with $\max_ju_j$. For
$c_j=b_j+\mathbf1\{m_j\ne y\}$ this is, at the optimum, exactly
$\max_k\eta_k\ge\max_j(\alpha_j-b_j)$.
The bound transfers with an overall factor $B$ to the raw cost risk:
normalized action-utility errors satisfy the same four-case bound,
the class utility error is at most $\|p-\eta\|_\infty/B$, and
the expert error is controlled by its Bernoulli KL. Applying the
original bound's inequalities and Jensen gives
\[
\mathcal E_{\mathrm{cost}}-\mathcal E_{\mathrm{cost}}^\star
\le B\max\{\sqrt2,\sqrt{2J/\lambda}\}
\sqrt{\mathcal E_{\mathrm{dec,cs}}-\mathcal E_{\mathrm{dec,cs}}^\star}.
\]
For $B=1$, this is the bounded-cost version of the original coefficient.

For class weights $\kappa_y>0$, set $W=\max_y\kappa_y$ and
$g_j(x)=\mathbb E[\kappa_Y\mathbf1\{M_j=Y\}\mid X=x]$.
The weighted error of class $k$ is
$\sum_y\kappa_y\eta_y-\kappa_k\eta_k$, while that of expert $j$ is
$\sum_y\kappa_y\eta_y-g_j$, and the common first term cancels in the
decision. We keep ordinary class CE, train the expert target
$(\kappa_y/W)\mathbf1\{m_j=y\}$, and compare
$\max_k\kappa_kp_k/W$ with $\max_ju_j$, predicting
$\argmax_k \kappa_kp_k$ when the class action wins. The expert optimum is
$g_j/W$. Since $\kappa_k/W\le1$, the same coordinate-error argument yields
the original coefficient for risk normalized by $W$, and a factor $W$
for raw weighted risk. Balanced error takes
$\kappa_y=(K\pi_y)^{-1}$ with $\pi_y=\Pr(Y=y)>0$ for every class.
The weights are assumed known in this population statement; estimating
them introduces an additional error not analyzed here.

For binary labels, encoded as $\{0,1\}$ in this paragraph, a fitted
posterior also supplies predictions
$\widehat y_a(x)=\mathbf1\{p_1(x)\ge a\}$. Comparing
$p_{\widehat y_a(x)}$ with $\max_ju_j+t$ defines a family of routing
policies. Evaluating their realized outputs while varying $a$ at fixed
$t$ produces an empirical system precision--recall curve, with coverage
allowed to change. This is an evaluation procedure, not a consistency
claim for constrained precision--recall objectives.

\section{Compute Resources}\label{app:compute}

Unless stated otherwise, experiments used NVIDIA A100 GPUs with 40GB of
memory, and the Fashion-MNIST studies also used NVIDIA H100 and H200 GPUs.
Each training run uses a single GPU, and seed replication repeats the
specified configuration. The synthetic suites finish within minutes per
seed, and each real-data benchmark (CIFAR-10, CIFAR-10H, Covertype, and
Fashion-MNIST) takes about one hour on a single NVIDIA H100 GPU.

\section{Synthetic Experiments}\label{app:synthetic}

This appendix describes the synthetic suites used in the paper. Each suite
instantiates one mechanism from Section~\ref{sec:mismatch} under a
distribution for which the Bayes rule is known exactly.
In every suite we specify $(\eta,\alpha)$ analytically, so the
Bayes-optimal L2D action at each $x$ is determined by the comparison
\[
\max_{k\in[K]}\eta_k(x)
\qquad\text{versus}\qquad
\max_{j\in[J]}\alpha_j(x),
\]
and the exact defer regret
$\Delta_\perp(f)\coloneqq\mathcal E_\perp(f)-\mathcal E_\perp^\star$
can be evaluated directly.

\paragraph{Protocol.}
All suites use linear models (no hidden layer) unless stated otherwise.
Every method receives the same feature vector $X$ and training tuple $(X,Y,M)$.
The ground-truth quantities $\eta(x)$ and $\alpha(x)$ never enter the
training losses: they are used \emph{after} training to evaluate the exact
Bayes regret and, on the validation split, to select the checkpoint with
the lowest true defer loss, a model-selection rule applied identically to
every surrogate.
All reported numbers are means over 3 to 5 seeds (stated per suite), with
one standard deviation in parentheses.

\subsection{Nested Redundant Experts}\label{app:synth-ce}

\paragraph{Objective.}
We vary the overlap multiplicity while keeping the Bayes decision problem
fixed, so that any performance degradation is attributable to the
surrogate itself rather than to a harder task.
Although the suite is designed around
Proposition~\ref{prop:additive-gradient} (Add.~CE amplification), the
nesting also stresses every other baseline, for surrogate-specific reasons
discussed below.

\paragraph{Setting.}
We use $K{=}16$ classes, $J{=}24$ experts,
$n_{\mathrm{train}}{=}900$, $n_{\mathrm{test}}{=}8{,}000$, and linear
models (no hidden layer).

\emph{Features and class posteriors.}\quad
Each point draws a region indicator
$R\sim\mathrm{Bernoulli}(0.65)$, and the feature vector lies in
$\mathbb R^{K}$.
On the defer-favorable region~$D$ ($R{=}1$), the label is
uniform on $[K]$, $\eta_k(x){=}1/K$ for all~$k$, and
$X = 3.5\,\mathbf 1_{K}+\varepsilon$ with
$\varepsilon\sim\mathcal N(0,0.05^2 I_{K})$;
hence no class is linearly separable.
On~$D^c$ ($R{=}0$), the class posterior is near-deterministic,
$\eta_Y(x){=}0.998$,
and the feature for the true label is shifted to~$5$ (easily separable).

\emph{Expert utilities.}\quad
The $J$ experts have monotonically decreasing accuracies on~$D$.
Define a spacing variable
$\rho_j \coloneqq \log(1{+}j)/\log J$ for $j\in\{0,\dots,J{-}1\}$ and
$J\ge2$, with $\rho_0\coloneqq0$ when $J=1$
(logarithmic spacing, so quality drops fastest over the first few experts
and the weaker experts cluster near the lowest accuracy).
On~$D$,
\begin{equation}\label{eq:nested-alpha}
\alpha_j(x)
\;=\;0.99\;-\;(0.99-0.75)\,\rho_j,
\end{equation}
giving $\alpha_0{=}0.99$ for the strongest expert down to
$\alpha_{J-1}{=}0.75$ for the weakest.
On~$D^c$, all experts are almost always wrong:
$\alpha_j(x)=0.04-(0.04-0.002)\,\rho_j$.

\emph{Nested correctness.}\quad
A single latent $U\sim\mathrm{Unif}[0,1]$ is drawn per sample, and expert~$j$
is correct if and only if $U\le \alpha_j(x)$.
Because $\alpha_0\ge\alpha_1\ge\cdots\ge\alpha_{J-1}$, this implies
\[
\{M_{j+1}=Y\}\;\subseteq\;\{M_j=Y\}
\qquad\text{for all } j.
\]
When expert~$j$ is incorrect, its prediction is drawn uniformly from
$[K]\setminus\{Y\}$.
Since $\alpha_0=0.99$ for every $J$ and $\alpha_j<\alpha_0$ for $j\ge1$,
changing $J$ never changes the Bayes action: expert~$0$ is always
Bayes-best on~$D$, classification is always best on~$D^c$, and only the
realized overlap $|\Jset(Y,\mathbf M)|$ grows with~$J$.

\paragraph{Bayes rule.}
For every~$J$, the Bayes rule defers to expert~$0$ on~$D$ and classifies
on~$D^c$.

\paragraph{Why each baseline degrades.}
The propositions in Section~\ref{sec:mismatch} predict a different failure
mechanism for each surrogate:
\begin{itemize}[itemsep=3pt,topsep=3pt,leftmargin=*]
\item \textbf{Add.~CE}
(Proposition~\ref{prop:additive-gradient}).
The gradient and the Hessian both carry the factor $1{+}|\Jset|$.
With $J{=}24$ nested experts, all with $\alpha_j\ge0.75$ on~$D$, every
expert is correct on $75\%$ of defer-region samples and
$\mathbb E[|\Jset|\mid D]=\sum_j\alpha_j\approx19.6$, so a defer-region
sample is weighted about $20\times$ (and $25\times$ when all experts are
correct), compared with roughly $1\times$ for a classify-region sample.
This gives high-overlap samples disproportionate gradient and curvature
relative to samples near the classify-versus-defer boundary.
\item \textbf{PiCCE}
(Proposition~\ref{prop:picce-starvation}).
Only the winning correct expert receives a negative gradient; all other
correct experts receive the positive push $2q_{K+j}{>}0$.
With $24$ nested experts, as many as $23$ correct experts are suppressed on
samples where all experts are correct, and in general $|\Jset|-1$ correct
experts are suppressed whenever multiple experts are correct.
The resulting winner-take-all lock-in concentrates all expert mass on
whichever expert had the highest logit at initialization, preventing the
model from learning the correct ranking $\alpha_0{>}\cdots{>}\alpha_{J-1}$.
\item \textbf{A-SM}
(Proposition~\ref{prop:asm-coupling}).
The class-logit gradient contains the expert leakage term
$-\pi_r\sum_j(\psi_j-t_j)$.
When many experts are correct, the residuals $\psi_j{-}t_j$ accumulate
and inject noise into the class update.
Because the leakage enters the class-logit gradient itself, it perturbs
the class weights of the linear model and hence the
classify-versus-defer boundary.
\item \textbf{OvA}
(Appendix~\ref{app:proof-ova-bernoulli}).
The $K$ independent sigmoid class heads are not coupled by a simplex
constraint, so the class confidence
$\max_k\sigma(g_k)$ can be miscalibrated.
With $K{=}16$ uniform classes on~$D$, the population binary targets are
$1/16$ for each class, but finite-sample and optimization errors are not
constrained to preserve a simplex. The maximum over $16$ independently
trained sigmoids can therefore be inflated relative to the true
$\max_k\eta_k=1/16$, biasing the routing comparison toward classifying when
it should defer.
\item \textbf{Mao25}
(Proposition~\ref{prop:mao25-setmass}).
The gradient factor $(S_{\Sset}{-}1)$ is common to all acceptable actions,
so the update cannot rank experts.
Additionally, the MAE-like structure means
$\sum_{i\in\Sset}-\partial\Phi/\partial a_i = S_\Sset(1{-}S_\Sset)\le 1/4$,
so the total corrective signal is weak and vanishes as $S_\Sset\to 0$,
which slows the learning of the defer region.
\end{itemize}

\paragraph{Result.}
Table~\ref{tab:appendix-redundant-final} reports the results at $J{=}24$,
and Table~\ref{tab:appendix-redundant-sweep} summarizes a sweep over
$J\in\{16,24,32\}$.

\begin{table}[h]
\centering
\small
\caption{Nested redundant suite at $J{=}24$.
The decoupled surrogate remains essentially Bayes-optimal, while all
baselines degrade severely.
Means with standard deviations in parentheses over 3 seeds.}
\label{tab:appendix-redundant-final}
\begin{tabular}{@{}lcc@{}}
\toprule
Method & Exact defer regret & System accuracy \\
\midrule
\textbf{Decoupled} & $\mathbf{0.0002\stdp{0.0001}}$ & $\mathbf{0.993\stdp{0.001}}$ \\
PiCCE        & $0.238\stdp{0.116}$             & $0.755\stdp{0.118}$ \\
OvA          & $0.294\stdp{0.020}$             & $0.699\stdp{0.020}$ \\
A-SM         & $0.304\stdp{0.035}$             & $0.688\stdp{0.035}$ \\
Mao25        & $0.323\stdp{0.035}$             & $0.669\stdp{0.036}$ \\
Add.\ CE     & $0.348\stdp{0.038}$             & $0.646\stdp{0.036}$ \\
\bottomrule
\end{tabular}
\end{table}

\begin{table}[h]
\centering
\small
\caption{Summary of the $J$ sweep in the nested redundant suite.
The decoupled surrogate stays near Bayes-optimal as the expert pool grows,
while the best non-decoupled regret increases monotonically.}
\label{tab:appendix-redundant-sweep}
\begin{tabular}{@{}lccc@{}}
\toprule
 & $J=16$ & $J=24$ & $J=32$ \\
\midrule
Decoupled defer regret           & $0.013$ & $0.0002$ & $0.007$ \\
Best non-decoupled defer regret  & $0.226$ & $0.238$  & $0.293$ \\
Best non-decoupled method        & PiCCE   & PiCCE    & OvA     \\
\bottomrule
\end{tabular}
\end{table}

\paragraph{Analysis.}
Because the Bayes rule is unchanged across the $J$-sweep, the large
degradation of every baseline cannot be explained by a harder routing
problem; it is attributable to how each surrogate responds to redundant
multiplicity.
The decoupled surrogate remains near Bayes-optimal ($\Delta_\perp{=}0.0002$ at $J{=}24$)
because its expert loss decomposes into $J$ independent BCE terms: each
$\alpha_j$ is estimated coordinatewise, and the gradient
$\partial\Phi^{\mathrm{dec}}/\partial s_j = (\lambda/J)(u_j{-}t_j)$ carries
no cross-expert coupling and no dependence on $|\Jset|$
(Proposition~\ref{prop:due-gradient}).
As the $J$-sweep in Table~\ref{tab:appendix-redundant-sweep} shows, the
best non-decoupled regret \emph{increases} with $J$ (from $0.226$ to $0.293$),
while the decoupled surrogate stays at or below $0.013$ throughout,
consistent with the prediction that the baselines' failures scale with the
overlap multiplicity.

\subsection{Rare Specialist for PiCCE}\label{app:synth-picce}

\paragraph{Objective.}
We isolate the winner-take-all starvation of
Proposition~\ref{prop:picce-starvation} with a specialist expert that is
Bayes-preferred on a rare region~$R$ but whose correct events are mostly
shared with a generalist that is correct more often overall.
PiCCE should suppress the specialist because only the current $\argmax$
winner receives the reinforcing gradient $2q_{K+j^\star}{-}1$ (negative
while $q_{K+j^\star}<1/2$), while every other correct expert is pushed
down by $2q_{K+j}>0$ (eq.~\eqref{eq:picce-starvation}).

\paragraph{Setting.}
We use $K{=}2$ classes, $J{=}2$ experts (generalist $j{=}1$, specialist $j{=}2$),
$n_{\mathrm{train}}{=}7{,}000$, $n_{\mathrm{test}}{=}18{,}000$,
5~seeds, and linear models.

\emph{Features and class posteriors.}\quad
Draw $Z\sim\mathrm{Unif}([-1,1]^2)$ and an independent rare indicator
$R\sim\mathrm{Bernoulli}(0.15)$.
The feature vector is $X=(Z_1,Z_2,\mathbf 1\{R\})\in\mathbb R^3$.
The class posterior depends on the region:
\[
\max_k \eta_k(x)=
\begin{cases}
0.60, & x\in R,\\
0.90, & x\notin R,
\end{cases}
\]
with the maximizing class determined by $\mathrm{sgn}(Z_1)$.
Classification is therefore easy outside~$R$ ($\eta_{\max}{=}0.90$) and
uncertain on~$R$ ($\eta_{\max}{=}0.60$).

\emph{Expert utilities.}\quad
Outside~$R$, expert correctness is independent, with
$\alpha_{\mathrm{gen}}(x){=}0.45$ and
$\alpha_{\mathrm{spec}}(x){=}0.15$, so neither expert is useful and the
Bayes action is to classify.
On~$R$, the marginal utilities are
\[
\alpha_{\mathrm{spec}}(x)=0.75
\;>\;
\alpha_{\mathrm{gen}}(x)=0.60
\;=\;
\max_k\eta_k(x)=0.60,
\]
so the Bayes action is to defer to the specialist, with a tie between the
generalist and the classifier below the specialist.

\emph{Joint correctness on~$R$.}\quad
On~$R$, correctness is \emph{not} independent, and we specify the joint
law directly:
\begin{center}
\small
\begin{tabular}{@{}lcc@{}}
\toprule
 & Specialist correct & Specialist incorrect \\
\midrule
Generalist correct   & $0.55$ & $0.05$ \\
Generalist incorrect & $0.20$ & $0.20$ \\
\bottomrule
\end{tabular}
\end{center}
This gives
$\alpha_{\mathrm{spec}}{=}0.55{+}0.20{=}0.75$ and
$\alpha_{\mathrm{gen}}{=}0.55{+}0.05{=}0.60$ as required.
When the specialist is correct on~$R$, the generalist is also correct with
probability
\[
\Pr(C_{\mathrm{gen}}{=}1\mid C_{\mathrm{spec}}{=}1,\,x\in R)
\;=\;\frac{0.55}{0.75}\;\approx\;0.73,
\]
so $73\%$ of the specialist's positive training signal comes from
shared-correct samples on which both experts are right.

\paragraph{Bayes rule.}
The Bayes rule defers to the specialist on~$R$ ($15\%$ of the data), since
$\alpha_{\mathrm{spec}}{=}0.75 > \alpha_{\mathrm{gen}}{=}0.60 =
\eta_{\max}{=}0.60$, and classifies on~$R^c$ ($85\%$ of the data), since
$\eta_{\max}{=}0.90 > \alpha_{\mathrm{gen}}{=}0.45$.

\paragraph{Why PiCCE fails.}
On the $73\%$ of specialist-correct samples on~$R$ where the generalist
is also correct ($\Jset{=}\{1,2\}$), PiCCE selects the winner
$j^\star{=}\argmax_{j\in\Jset}a_{K+j}$.
The gradient on the \emph{non-winning} correct expert is
$\partial\Phi^{\mathrm{PiCCE}}/\partial a_{K+j}=2q_{K+j}>0$
(Proposition~\ref{prop:picce-starvation}), which pushes its logit
\emph{down}.
If the generalist starts with a marginally higher logit through
initialization noise, the specialist is suppressed on every shared-correct
sample, and the gap compounds across iterations: the generalist wins more
$\argmax$ selections and receives more reinforcing gradient, while the
specialist is driven down.
On the $27\%$ of specialist-correct events that are specialist-only
($C_{\mathrm{gen}}{=}0$), the specialist does receive a negative gradient,
but these events make up only $15\%\times 20\%{=}3\%$ of the training data,
which is not enough to overcome the opposite signal from the dominant
shared-correct mass.

\paragraph{Why the decoupled surrogate succeeds.}
The decoupled surrogate trains each expert head with an independent BCE loss:
$\partial\Phi^{\mathrm{dec}}/\partial s_j=(\lambda/J)(u_j{-}t_j)$,
where $u_j{=}\sigma(s_j)$ and $t_j{=}\mathbf 1\{M_j{=}Y\}$.
On a shared-correct sample, \emph{both} experts independently receive the
reinforcing gradient $(\lambda/J)(u_j{-}1)<0$.
Since no $\argmax$ selection occurs, the generalist's logit does not
compete with the specialist's, each head targets its own marginal utility,
and the routing comparison
$\sigma(s_{\mathrm{spec}})\gtrless\sigma(s_{\mathrm{gen}})$ learns
$0.75>0.60$ from the data.

\paragraph{Metrics.}
We report two specialist-specific diagnostics alongside the standard
defer regret and system accuracy:
\begin{itemize}[itemsep=2pt,topsep=2pt,leftmargin=*]
\item \emph{Specialist selection on~$R$}:
$\Pr(r(X){=}\mathrm{spec}\mid X\in R)$, computed over all test samples
in~$R$ (including those that classify; classifying samples count as not
selecting the specialist).
\item \emph{Shared-correct routing}:
$\Pr(r(X){=}\mathrm{spec}\mid X\in R,\,C_{\mathrm{gen}}{=}1,\,
C_{\mathrm{spec}}{=}1)$, i.e.\ specialist selection restricted to the
shared-correct subset.
This is the hardest case, because both experts are right and only the
ranking $\alpha_{\mathrm{spec}}>\alpha_{\mathrm{gen}}$ distinguishes them.
\end{itemize}

\paragraph{Result.}
Table~\ref{tab:appendix-picce-final} reports the results (means with
standard deviations in parentheses over 5 seeds).

\begin{table}[h]
\centering
\small
\caption{Rare-specialist suite.
\emph{Coverage} is $\Pr(r(X){=}\mathrm{classify})$, and the Bayes-optimal
coverage is~$0.85$.
PiCCE defers but never to the specialist, Mao25 never defers at all, and
the decoupled surrogate, OvA, and A-SM recover the specialist almost
perfectly. Add.\ CE reaches $1.000$ on both specialist metrics only because
it defers on $87\%$ of inputs, so no entry in those columns is bold.}
\label{tab:appendix-picce-final}
\begin{tabular}{@{}lccccc@{}}
\toprule
Method & Defer regret & System acc.\ & Coverage & Spec.\ sel.\ on $R$ & Shared-correct routing \\
\midrule
\textbf{Decoupled}  & $\mathbf{0.003\stdp{0.002}}$ & $0.875\stdp{0.002}$ & $0.849$ & $0.993\stdp{0.007}$ & $0.994\stdp{0.007}$ \\
OvA           & $0.003\stdp{0.001}$          & $0.875\stdp{0.003}$ & $0.850$ & $0.988\stdp{0.023}$          & $0.987\stdp{0.026}$ \\
A-SM          & $0.003\stdp{0.002}$          & $0.875\stdp{0.003}$ & $0.850$ & $0.988\stdp{0.017}$          & $0.988\stdp{0.017}$ \\
Mao25         & $0.026\stdp{0.002}$          & $0.852\stdp{0.004}$ & $1.000$ & $0.000\stdp{0.000}$          & $0.000\stdp{0.000}$ \\
PiCCE         & $0.035\stdp{0.001}$          & $0.843\stdp{0.002}$ & $0.821$ & $0.000\stdp{0.000}$          & $0.000\stdp{0.000}$ \\
Add.\ CE      & $0.324\stdp{0.067}$          & $0.554\stdp{0.069}$ & $0.128$ & $1.000\stdp{0.000}$          & $1.000\stdp{0.000}$ \\
\bottomrule
\end{tabular}
\end{table}

\paragraph{Analysis.}
The table shows three distinct failure modes, one each for PiCCE, Mao25,
and Add.~CE.

\emph{PiCCE (winner-take-all starvation).}\quad
PiCCE has reasonable coverage ($0.821$, close to the Bayes-optimal $0.85$),
so it does learn the classify-versus-defer boundary.
However, its specialist selection is exactly~$0$ on~$R$: when it defers, it
always picks the generalist.
The shared-correct routing rate is also~$0$, confirming that on the decisive
subset where both experts are correct, PiCCE never learns to prefer the
specialist.
This behavior matches the mechanism of
Proposition~\ref{prop:picce-starvation}: the generalist captures the
$\argmax$ early, and the positive gradient $2q_{K+\mathrm{spec}}>0$
keeps the specialist suppressed for the rest of training.

\emph{Mao25 (no deferral learned).}\quad
Mao25 has coverage~$1.0$, so it \emph{never} defers and its specialist
selection is trivially~$0$.
Its system accuracy of $0.852$ is close to that of a Bayes-optimal
classifier that ignores experts,
$0.85\times 0.90 + 0.15\times 0.60 = 0.855$.
This is consistent with the MAE-like optimization pathology of
Proposition~\ref{prop:mao25-setmass}: the corrective gradient on acceptable
actions has total magnitude $S_\Sset(1{-}S_\Sset)\le 1/4$, which vanishes
both as $S_\Sset\to0$ and as $S_\Sset\to1$, so Mao25 receives little signal
to move away from the classify-everything solution.

\emph{Add.~CE (global collapse).}\quad
Add.~CE has coverage~$0.128$ (it defers $87\%$ of the time) and
defer regret~$0.324$.
Even with only $J{=}2$ experts, the gradient amplification of
Proposition~\ref{prop:additive-gradient} is a plausible cause: the factor
$1{+}|\Jset|$ reaches~$3$ on shared-correct samples, which is consistent with
the over-deferral observed here.

\emph{The decoupled surrogate, OvA, and A-SM (correct routing).}\quad
All three methods achieve near-Bayes-optimal defer regret ($0.003$),
coverage close to~$0.85$, and specialist selection ${\ge}0.988$ on~$R$.
With only $J{=}2$ experts, the A-SM coupling
(Proposition~\ref{prop:asm-coupling}) and the OvA miscalibration are
negligible here, which indicates that the PiCCE failure comes from the
winner-take-all mechanism rather than from a general difficulty of the
task.

\subsection{Shared Acceptability and Expert Ranking for Mao25}\label{app:synth-mao}

\paragraph{Objective.}
We isolate the ranking limitation of
Proposition~\ref{prop:mao25-setmass}: the Mao25 gradient factor
$(S_\Sset{-}1)$ is the same for all acceptable actions, so the samplewise
update cannot distinguish the best expert from a merely acceptable one.
We therefore construct a task in which the main difficulty is expert
ranking \emph{within} a defer region under heavy shared acceptability.

\paragraph{Setting.}
We use $K{=}10$ classes, $J{=}4$ experts,
$n_{\mathrm{train}}{=}1{,}000$, $n_{\mathrm{test}}{=}12{,}000$,
3~seeds, and linear models.

\emph{Features and class posteriors.}\quad
Each point draws a defer-region indicator
$D\sim\mathrm{Bernoulli}(0.60)$, and the feature vector lies in
$\mathbb R^{K+J}=\mathbb R^{14}$.
On~$D$, a sector $Q\sim\mathrm{Unif}([J])$ is drawn, the label is
uniform on $[K]$ ($\eta_k{=}1/10$ for all~$k$), and
\[
X_{1:K}=3.1\,\mathbf 1_{K}+\varepsilon_{\mathrm{cls}},\qquad
X_{K+Q}=2.2+\varepsilon_{Q},
\qquad\varepsilon\sim\mathcal N(0,\,0.05^2).
\]
The first~$K$ features are uninformative for the class
(all near~$3.1$); the sector feature $X_{K+Q}$ identifies which expert is
locally best.
On~$D^c$, the class posterior is near-deterministic
($\eta_Y{=}0.998$, feature of the true label shifted to~$5$).

\emph{Joint expert correctness on~$D$.}\quad
Inside~$D$, the identity of the best expert changes with the sector~$Q$,
and the joint correctness of the $J{=}4$ experts is drawn from the
following event model:
\begin{center}
\small
\begin{tabular}{@{}lcc@{}}
\toprule
Event & Probability & Who is correct \\
\midrule
All correct                    & $0.10$ & all 4 experts \\
$Q$ + one random other         & $0.78$ & $Q$ and one uniformly chosen $j{\neq}Q$ \\
$Q$ only                       & $0.05$ & only $Q$ \\
One random non-$Q$ only        & $0.03$ & one uniformly chosen $j{\neq}Q$ \\
None correct                   & $0.04$ & no expert \\
\bottomrule
\end{tabular}
\end{center}
The marginal utilities on~$D$ are therefore
\[
\alpha_Q(x)
\;=\;0.10+0.78+0.05\;=\;0.93,
\qquad
\alpha_{j\neq Q}(x)
\;=\;0.10+\tfrac{0.78}{3}+\tfrac{0.03}{3}
\;=\;0.37,
\]
while on~$D^c$ all experts have utility $\alpha_j{=}0.05$.

The design has two properties that the test relies on.
First, \emph{shared acceptability is high}: on~$88\%$ of defer-region
samples ($0.10{+}0.78$) at least two experts are correct, so the
acceptable set $\Sset(y,m)$ typically contains the true class plus
multiple experts.
Second, \emph{expert ranking matters}: the Bayes-optimal expert changes
with the sector~$Q$, and the gap between the best and non-best utilities
is large ($0.93$ versus $0.37$).

\paragraph{Bayes rule.}
The Bayes rule defers to the sector-best expert~$Q$ on~$D$ ($60\%$ of the
data), since $\alpha_Q{=}0.93 \gg \eta_{\max}{=}0.10$, and classifies
on~$D^c$ ($40\%$ of the data), since $\eta_Y{=}0.998 \gg \alpha_j{=}0.05$.
The resulting Bayes-optimal system accuracy is
$0.60\times 0.93+0.40\times 0.998=0.957$, with Bayes-optimal coverage
$0.40$.

\paragraph{Why Mao25 fails at ranking.}
Proposition~\ref{prop:mao25-setmass} shows that the Mao25 gradient on an
acceptable action~$i\in\Sset$ is $q_i(S_\Sset{-}1)$, where
$S_\Sset=\sum_{c\in\Sset}q_c$ is the total softmax mass on the acceptable
set.
The factor $(S_\Sset{-}1)$ is the same for every acceptable action: the
best expert~$Q$ and a non-best expert $j{\neq}Q$ receive the same
gradient magnitude (up to the softmax weight $q_i$, which is controlled by
the model, not the data).
The samplewise loss therefore provides no direct signal that $Q$ should be
preferred over other acceptable experts.

Across samples, the model can still learn the ranking indirectly because
expert~$Q$ appears in the acceptable set much more often than the other
experts on sector~$Q$.
However, this signal comes from marginal frequency across samples rather than
from an explicit within-sample preference among acceptable experts, so we
expect a degraded but non-trivial ranking, which we measure by the
best-expert selection rate.

\paragraph{Metrics.}
We report \emph{best-expert selection on~$D$}: among test samples in~$D$
where the model defers, the fraction that selects the sector-best expert~$Q$.
With $J{=}4$ symmetric sectors, chance level (uniform random among experts)
is~$0.25$.
A model that picks uniformly among \emph{correct} experts would achieve
$\approx 0.48$
($0.10\times\tfrac14+0.78\times\tfrac12+0.05\times 1+0.03\times 0
+0.04\times\tfrac14$), and perfect ranking gives $1.0$.

\paragraph{Result.}
Table~\ref{tab:appendix-mao-final} reports the results (means with
standard deviations in parentheses over 3 seeds).

\begin{table}[h]
\centering
\small
\caption{Shared-acceptability ranking suite.
Chance-level best-expert selection is~$0.25$, and uniform selection among
correct experts gives~$0.48$.
Mao25 reaches~$0.750$, well above chance but far below the $1.000$ achieved
by the decoupled surrogate and three other baselines, consistent with a
partial but degraded ranking signal.}
\label{tab:appendix-mao-final}
\begin{tabular}{@{}lcccc@{}}
\toprule
Method & Defer regret & System acc.\ & Coverage & Best expert sel.\ on $D$ \\
\midrule
\textbf{Decoupled} & $\mathbf{0.0008\stdp{0.001}}$ & $\mathbf{0.957\stdp{0.003}}$ & $0.397$ & $1.000$ \\
A-SM         & $0.064\stdp{0.020}$            & $0.894\stdp{0.018}$          & $0.330$ & $1.000$ \\
OvA          & $0.082\stdp{0.076}$            & $0.876\stdp{0.078}$          & $0.349$ & $1.000$ \\
PiCCE        & $0.129\stdp{0.045}$            & $0.829\stdp{0.043}$          & $0.383$ & $0.661$ \\
Add.\ CE     & $0.151\stdp{0.033}$            & $0.808\stdp{0.031}$          & $0.241$ & $1.000$ \\
Mao25        & $0.154\stdp{0.043}$            & $0.804\stdp{0.045}$          & $0.324$ & $0.750$ \\
\bottomrule
\end{tabular}
\end{table}

\paragraph{Analysis.}
The decoupled surrogate reaches the Bayes-optimal system accuracy ($0.957$) with coverage
$0.397\approx 0.40$ and perfect best-expert selection.
Its independent BCE heads estimate each $\alpha_j$ accurately enough that
the routing comparison $\max_j\sigma(s_j)\gtrless\max_k p_k$ places both
the classify-versus-defer boundary and the expert ranking correctly.

\emph{Mao25.}\quad
Mao25 achieves best-expert selection of~$0.750$, well above chance
($0.25$) and above the uniform-among-correct baseline ($0.48$), which
indicates that the linear model does extract partial ranking from the
sector feature.
However, it remains well below the $1.000$ achieved by the decoupled
surrogate, A-SM, OvA, and Add.~CE, all of which have loss terms that
produce a per-expert ranking signal.
We attribute the gap from $0.750$ to $1.000$ to the
$(S_\Sset{-}1)$ factor being common to all acceptable actions
(Proposition~\ref{prop:mao25-setmass}).
Combined with imperfect coverage ($0.324$ versus the optimal $0.40$),
this yields a defer regret of $0.154$, nearly $200\times$ the regret of
the decoupled surrogate.

\emph{PiCCE.}\quad
PiCCE has even worse ranking ($0.661$) with high variance across seeds
(one seed drops to $0.50$), consistent with the winner-take-all lock-in
from Proposition~\ref{prop:picce-starvation}: which expert ``wins'' depends
on initialization noise, producing unstable ranking.

\emph{A-SM and OvA.}\quad
Both achieve perfect ranking ($1.000$) and reasonable coverage, but their
defer regrets ($0.064$ and $0.082$) are still $80$ to $100$ times that of
the decoupled surrogate.
The ranking is correct, but the classify-versus-defer boundary is less
well placed: A-SM coverage is $0.330$ against the optimal $0.40$, which
indicates over-deferral and, for A-SM, is consistent with the class--expert
coupling of Proposition~\ref{prop:asm-coupling} perturbing the class
logits.

\emph{Add.~CE.}\quad
Add.~CE also achieves perfect ranking but has the lowest coverage ($0.241$),
reflecting severe over-deferral from the gradient amplification of
Proposition~\ref{prop:additive-gradient}: on the $88\%$ of defer-region
samples with shared acceptability, $|\Jset|$ equals $2$ or $4$, which
inflates the gradient by a factor of $3$ or $5$.

\subsection{Logit-Space Geometry Diagnostics}\label{app:synth-geometry}

The task-level suites above measure end-to-end performance; this section
complements them with \emph{direct measurements of the gradients and
Hessians} predicted by Propositions~\ref{prop:additive-gradient},
\ref{prop:picce-starvation}, and~\ref{prop:asm-coupling}.

\paragraph{Protocol.}
For each diagnostic we train a linear model (no hidden layer) to
convergence on a synthetic L2D task, then compute per-sample gradients and
Hessians of the surrogate loss on held-out test logits.
Using trained rather than random logits ensures that the measurements
reflect the geometry at the operating point the optimizer actually reaches.
All quantities are computed on the \emph{augmented} action vector
$a=(a_1,\dots,a_K,a_{K+1},\dots,a_{K+J})$, i.e.\ on the logit space of
the surrogate itself.
Each diagnostic is repeated across multiple seeds (stated per table);
we report means with standard deviations in parentheses across all per-sample measurements
pooled over seeds.

\paragraph{Add.\ CE curvature inflation
  (Corollary~\ref{cor:mozannar-ce-geometry}).}
Corollary~\ref{cor:mozannar-ce-geometry} shows that the Hessian of
the Add.~CE surrogate $\Phi^{\mathrm{CE}}$ has largest eigenvalue
$\lambda_{\max}\le(1+|\Jset|)/2$, an upper bound linear in the
correct-expert count $|\Jset|$.
The decoupled surrogate, by contrast, has $\lambda_{\max}\le 1/2$ for the
class block and expert entries at most $\lambda/(4J)$, independently of
$|\Jset|$ (see~\eqref{eq:due-hessian}).

\emph{Setup.}
We use the nested redundant-expert construction
(Appendix~\ref{app:synth-ce}) with $J\in\{1,4,8,16,32\}$ to control the
overlap multiplicity $|\Jset|$.
For each value of $J$, we train an Add.\ CE and a decoupled surrogate model
($\lambda=1.0$) on $n_{\mathrm{train}}{=}3{,}500$ samples and evaluate
on $n_{\mathrm{test}}{=}6{,}000$.
We restrict to test samples on which at least one expert is correct
($|\Jset|>0$), and for each such sample we compute
(i) the gradient norm $\|\nabla_a\Phi\|_2$ and
(ii) the top Hessian eigenvalue $\lambda_{\max}(\nabla_a^2\Phi)$
from the full $(K{+}J)\times(K{+}J)$ Hessian.
Results are pooled over ${\sim}900$ qualifying samples per
$(J,\text{method})$ combination across 5 independent seeds.

\begin{table}[h]
\centering
\small
\caption{Per-sample curvature as the pool size $J$ varies (5 seeds, about
$900$ pooled samples per row). The Add.\ CE top eigenvalue rises from $0.82$
at $J{=}1$ to $1.09$ at $J{=}4$ and then saturates near $1.12$, far below the
bound $(1{+}|\Jset|)/2$ of Corollary~\ref{cor:mozannar-ce-geometry}, because
the softmax mass spreads over $K{+}J$ actions.
The decoupled surrogate remains flat at ${\approx}0.36$ for all $|\Jset|$,
consistent with the constant bound $1/2$ from~\eqref{eq:due-hessian}.}
\label{tab:geometry-curvature}
\begin{tabular}{@{}lccc@{}}
\toprule
Method & $J$ & Gradient norm $\|\nabla_a\Phi\|_2$ & Top eigenvalue $\lambda_{\max}(\nabla_a^2\Phi)$ \\
\midrule
Add.\ CE & 1  & $0.586\stdp{0.305}$ & $0.821\stdp{0.084}$ \\
Add.\ CE & 4  & $0.609\stdp{0.266}$ & $1.093\stdp{0.067}$ \\
Add.\ CE & 8  & $0.616\stdp{0.292}$ & $1.101\stdp{0.116}$ \\
Add.\ CE & 16 & $0.633\stdp{0.287}$ & $1.125\stdp{0.186}$ \\
Add.\ CE & 32 & $0.613\stdp{0.282}$ & $1.117\stdp{0.181}$ \\
\midrule
Decoupled      & 1  & $0.704\stdp{0.222}$ & $0.368\stdp{0.111}$ \\
Decoupled      & 4  & $0.573\stdp{0.250}$ & $0.358\stdp{0.137}$ \\
Decoupled      & 8  & $0.551\stdp{0.280}$ & $0.361\stdp{0.125}$ \\
Decoupled      & 16 & $0.534\stdp{0.290}$ & $0.360\stdp{0.134}$ \\
Decoupled      & 32 & $0.510\stdp{0.297}$ & $0.359\stdp{0.132}$ \\
\bottomrule
\end{tabular}
\end{table}

\emph{Reading the table.}
The Add.~CE eigenvalue jumps from $0.82$ to $1.09$ as $J$ goes from 1
to~4, then saturates near $1.12$ for larger~$J$.
The saturation is expected: at $J{=}32$ the theoretical bound is
$(1+32)/2=16.5$, but the \emph{realized} eigenvalue is much smaller because
the softmax probabilities $q_i$ become very small when spread over
$K{+}32$ actions. The relevant contrast is between this sharp increase and
saturation for Add.~CE and the decoupled eigenvalue, which stays at about
$0.36$ for every~$|\Jset|$.

\paragraph{PiCCE starvation gradient
  (Proposition~\ref{prop:picce-starvation}).}
Proposition~\ref{prop:picce-starvation} predicts that under PiCCE, every
correct expert other than the winner $j^\star$ receives a strictly
positive gradient $\partial\Phi^{\mathrm{PiCCE}}/\partial a_{K+k}
= 2q_{K+k}>0$, pushing its logit \emph{downward} even though the expert is
correct.
Under the decoupled surrogate, each expert gradient is $(\lambda/J)(u_j-t_j)$, which is
\emph{negative} (pulling toward deferral) whenever the expert is correct
($t_j=1$) and its sigmoid output satisfies $u_j<1$.

\emph{Setup.}
We use the PiCCE specialist task (Appendix~\ref{app:synth-picce}) with
$K{=}2$ classes and $J{=}2$ experts, where samples routinely have both
experts correct.
For each method (PiCCE and the decoupled surrogate, $\lambda{=}1.0$) we train on
$n_{\mathrm{train}}{=}4{,}500$ and evaluate on $n_{\mathrm{test}}{=}7{,}000$.
We select test samples where $|\Jset|>1$ (multiple correct experts).
For PiCCE, the ``suppressed expert'' is any $k\in\Jset$ with $k\neq j^\star$,
where $j^\star=\argmax_{j\in\Jset}\, a_{K+j}$ is the winning expert
(determined by the trained logit magnitudes).
For the decoupled surrogate, which has no winner selection, we measure the
gradient on \emph{all} correct expert logits.
Results pool 784 decoupled-surrogate gradient samples and 392 PiCCE
gradient samples across 5 seeds.

\begin{table}[h]
\centering
\small
\caption{Gradient on correct expert logits (5 seeds; 784
decoupled-surrogate and 392 PiCCE measurements). For PiCCE, only the
suppressed (non-winning) correct expert is measured.
PiCCE produces a \emph{positive} gradient on $100\%$ of suppressed
experts, pushing them away from deferral, exactly as
Proposition~\ref{prop:picce-starvation} predicts ($2q_{K+k}>0$).
The decoupled surrogate always produces a \emph{negative} gradient, pulling
every correct expert toward deferral.}
\label{tab:geometry-picce}
\begin{tabular}{@{}lcc@{}}
\toprule
Method & Mean gradient $\partial\Phi/\partial a_{K+k}$ & Positive-gradient rate \\
\midrule
Decoupled   & $-0.245\stdp{0.103}$ & $0.000$ \\
PiCCE & $+0.240\stdp{0.076}$ & $1.000$ \\
\bottomrule
\end{tabular}
\end{table}

\emph{Reading the table.}
The positive-gradient rate is $1.000$ for PiCCE and $0.000$ for the decoupled surrogate,
with zero variance: the predicted sign is not a statistical tendency but
an exact per-sample property.
The mean PiCCE gradient on the suppressed expert ($+0.240$) is
comparable in magnitude to the decoupled surrogate gradient ($-0.245$), so
the starvation is not a negligible effect but an active repulsive signal of
the same scale as the decoupled surrogate's attractive one.

\paragraph{A-SM class--expert coupling
  (Proposition~\ref{prop:asm-coupling}).}
Proposition~\ref{prop:asm-coupling} shows that the A-SM gradient on
each class logit $a_r$ ($r\neq k^\star$) includes a leakage term
$-\pi_r\sum_j(\psi_j-t_j)$ from the expert estimates.
At second order, this coupling produces a non-zero mixed
class--expert Hessian block
$H_{\mathrm{mix}}=\partial^2\Phi^{\mathrm{ASM}}/(\partial
a_{1:K}\,\partial a_{K+1:K+J})$, whereas the decoupled surrogate's Hessian
is block-diagonal by~\eqref{eq:due-hessian}, so its mixed block is
\emph{identically} zero.

\emph{Setup.}
We vary $J\in\{1,3,5,9,17\}$ by adding distractor experts (always
incorrect) alongside one correct expert.
For each $J$, we train an A-SM and a decoupled surrogate model ($\lambda{=}1.0$) on
$n_{\mathrm{train}}{=}4{,}500$ and evaluate on
$n_{\mathrm{test}}{=}7{,}000$.
For each test sample we assemble the full $(K{+}J)\times(K{+}J)$ Hessian
via second-order \texttt{autograd}, extract the $K\times J$ upper-right
block $H_{\mathrm{mix}}$, and compute its operator norm
$\|H_{\mathrm{mix}}\|_{\mathrm{op}}$.
We report the mean operator norm over 128 test samples, averaged across
5 seeds.
For the decoupled surrogate, the mixed block is analytically zero by
construction, since the class logits $w$ and expert logits $s_j$ enter
disjoint terms of $\Phi^{\mathrm{dec}}_\lambda$; we therefore record the
value $0$ directly rather than computing it numerically.

\begin{table}[h]
\centering
\small
\caption{Operator norm of the mixed class--expert Hessian block
(5 seeds, 128 samples each).
The A-SM coupling grows monotonically from $0.12$ at $J{=}1$ to $0.48$ at
$J{=}17$, consistent with the $O(\sqrt{J})$ bound of
Appendix~\ref{app:proof-asm-curvature}.
The decoupled mixed block is zero for every $J$ by the block-diagonal
structure of~\eqref{eq:due-hessian}.}
\label{tab:geometry-asm}
\begin{tabular}{@{}lcc@{}}
\toprule
Method & Experts $J$ & Mixed-block norm $\|H_{\mathrm{mix}}\|_{\mathrm{op}}$ \\
\midrule
A-SM & 1  & $0.122\stdp{0.004}$ \\
A-SM & 3  & $0.207\stdp{0.011}$ \\
A-SM & 5  & $0.263\stdp{0.009}$ \\
A-SM & 9  & $0.351\stdp{0.009}$ \\
A-SM & 17 & $0.475\stdp{0.012}$ \\
\midrule
Decoupled  & 1  & $0.000\stdp{0.000}$ \\
Decoupled  & 3  & $0.000\stdp{0.000}$ \\
Decoupled  & 5  & $0.000\stdp{0.000}$ \\
Decoupled  & 9  & $0.000\stdp{0.000}$ \\
Decoupled  & 17 & $0.000\stdp{0.000}$ \\
\bottomrule
\end{tabular}
\end{table}

\emph{Reading the table.}
The A-SM coupling norm grows roughly as $\sqrt{J}$, from $0.12$ at $J{=}1$
to $0.48$ at $J{=}17$, close to $0.12\sqrt{17}\approx 0.49$, which is
consistent with the theoretical $O(\sqrt{J})$ bound derived in
Appendix~\ref{app:proof-asm-curvature}.
Adding experts, even \emph{distractor} experts that are never correct,
therefore injects progressively more expert-side noise into the class
gradient.
The mixed block of the decoupled score Hessian remains zero at every pool
size, because expert residuals do not enter the class-score derivative.
With a learned shared trunk, the parameter interaction is the additive one
described in Appendix~\ref{app:revision-sharing}.

\section{CIFAR-10 with Synthetic Experts}\label{app:cifar10}

The synthetic experiments of Appendix~\ref{app:synthetic} use linear
models on low-dimensional Gaussian data, which isolates surrogate
pathologies in a setting where the Bayes solution is known. This
appendix tests whether the same effects persist on a real vision task.

\paragraph{Setup.}
We train a ResNet-12 backbone end-to-end, without pre-training, on
CIFAR-10 ($K{=}10$ classes) with standard data augmentation
(random crop, horizontal flip). All methods share the same backbone
architecture and optimizer (AdamW, lr~$10^{-3}$, 80~epochs).
All reported metrics are computed on the $10\,000$-image CIFAR-10 test
set. A held-out validation split ($5\,000$ training samples) is used to
select the best checkpoint (lowest validation deferral loss) and, for the decoupled
surrogate, to choose the total expert weight $\lambda=1/2$, so that the
per-expert weight is $1/(2J)$.
For the decoupled surrogate, the linear head is a \emph{split head} made of
two independent projections
$f_{\text{cls}}\colon\mathbb{R}^{256}\to\mathbb{R}^K$ and
$f_{\text{exp}}\colon\mathbb{R}^{256}\to\mathbb{R}^J$, matching the
structure of the decoupled loss, whereas all other methods use a single
$\mathbb{R}^{256}\to\mathbb{R}^{K+J}$ head; the total number of parameters
is therefore the same across methods.

\paragraph{Redundant expert suite.}
We construct $J$ synthetic experts with \emph{nested} correctness
(Appendix~\ref{app:synth-ce}): a single latent
$U\sim\mathrm{Uniform}(0,1)$ per sample determines all expert outcomes,
with expert~$j$ correct if and only if $U<\alpha_j$, where
$\alpha_j = 0.70 - 0.01\,j$.
The first~$8$ experts in a $J{=}32$ run are therefore \emph{identical}
to those in a $J{=}8$ run, so performance changes across $J$ can be
attributed to the surrogate's response to the added experts.

\begin{table}[h]
\centering
\small
\setlength{\tabcolsep}{3.5pt}
\caption{CIFAR-10 redundant expert suite (mean$\pm$std over $4$~seeds).
Best values per $J$-block in \textbf{bold};
second-best \underline{underlined}.
``Cls.\ Acc.''~$=$~classifier accuracy (ignoring the routing decision).
``Cov.''~$=$~coverage (fraction classified by the model, not deferred).}
\label{tab:cifar10-redundant}
\begin{tabular}{@{}l c cccc@{}}
\toprule
 & & \multicolumn{4}{c}{Test metrics} \\
\cmidrule(lr){3-6}
Method & $J$ & Sys.\ Acc.\ $\uparrow$ & Defer Loss $\downarrow$ & Cls.\ Acc.\ & Cov. \\
\midrule
\textbf{Decoupled}      & 8  & $\mathbf{0.918}_{\pm.001}$ & $\mathbf{0.082}_{\pm.001}$ & $\underline{0.909}_{\pm.001}$ & $0.961_{\pm.004}$ \\
Classifier only   & 8  & $\underline{0.910}_{\pm.002}$ & $\underline{0.090}_{\pm.002}$ & $\mathbf{0.910}_{\pm.002}$ & $1.000_{\pm.000}$ \\
OvA               & 8  & $0.901_{\pm.001}$          & $0.099_{\pm.001}$          & $0.875_{\pm.001}$          & $0.857_{\pm.015}$ \\
PiCCE             & 8  & $0.858_{\pm.007}$          & $0.142_{\pm.007}$          & $0.888_{\pm.008}$          & $0.616_{\pm.032}$ \\
A-SM              & 8  & $0.860_{\pm.002}$          & $0.140_{\pm.002}$          & $0.831_{\pm.002}$          & $0.774_{\pm.014}$ \\
Add.\ CE          & 8  & $0.815_{\pm.002}$          & $0.185_{\pm.002}$          & $0.902_{\pm.004}$          & $0.529_{\pm.006}$ \\
Mao25             & 8  & $0.703_{\pm.000}$          & $0.297_{\pm.000}$          & $0.105_{\pm.024}$          & $0.000_{\pm.000}$ \\
\midrule
\textbf{Decoupled}      & 16 & $\mathbf{0.919}_{\pm.003}$ & $\mathbf{0.081}_{\pm.003}$ & $\underline{0.909}_{\pm.002}$ & $0.961_{\pm.004}$ \\
Classifier only   & 16 & $\underline{0.911}_{\pm.001}$ & $\underline{0.089}_{\pm.001}$ & $\mathbf{0.911}_{\pm.001}$ & $1.000_{\pm.000}$ \\
OvA               & 16 & $0.890_{\pm.002}$          & $0.110_{\pm.002}$          & $0.859_{\pm.004}$          & $0.829_{\pm.014}$ \\
PiCCE             & 16 & $0.860_{\pm.011}$          & $0.140_{\pm.011}$          & $0.895_{\pm.006}$          & $0.633_{\pm.064}$ \\
A-SM              & 16 & $0.827_{\pm.003}$          & $0.174_{\pm.003}$          & $0.772_{\pm.003}$          & $0.662_{\pm.035}$ \\
Add.\ CE          & 16 & $0.790_{\pm.013}$          & $0.210_{\pm.013}$          & $0.895_{\pm.003}$          & $0.502_{\pm.035}$ \\
Mao25             & 16 & $0.703_{\pm.000}$          & $0.297_{\pm.000}$          & $0.111_{\pm.010}$          & $0.000_{\pm.000}$ \\
\midrule
\textbf{Decoupled}      & 32 & $\mathbf{0.919}_{\pm.002}$ & $\mathbf{0.081}_{\pm.002}$ & $\mathbf{0.908}_{\pm.003}$ & $0.957_{\pm.005}$ \\
Classifier only   & 32 & $\underline{0.908}_{\pm.002}$ & $\underline{0.092}_{\pm.002}$ & $\mathbf{0.908}_{\pm.002}$ & $1.000_{\pm.000}$ \\
OvA               & 32 & $0.880_{\pm.003}$          & $0.120_{\pm.003}$          & $0.844_{\pm.003}$          & $0.798_{\pm.015}$ \\
PiCCE             & 32 & $0.859_{\pm.008}$          & $0.141_{\pm.008}$          & $\underline{0.880}_{\pm.014}$          & $0.636_{\pm.027}$ \\
A-SM              & 32 & $0.783_{\pm.006}$          & $0.218_{\pm.006}$          & $0.680_{\pm.009}$          & $0.484_{\pm.039}$ \\
Add.\ CE          & 32 & $0.712_{\pm.015}$          & $0.288_{\pm.015}$          & $0.872_{\pm.009}$          & $0.305_{\pm.022}$ \\
Mao25             & 32 & $0.703_{\pm.000}$          & $0.297_{\pm.000}$          & $0.105_{\pm.025}$          & $0.000_{\pm.000}$ \\
\bottomrule
\end{tabular}
\end{table}

\paragraph{Analysis.}

\emph{The decoupled surrogate is the only method in our comparison whose system accuracy exceeds that of a standalone classifier.}
At every~$J$, the decoupled surrogate achieves higher system accuracy than the
classifier-only baseline, by $+0.8$~pp at $J{=}8$, $+0.8$~pp at
$J{=}16$, and $+1.1$~pp at $J{=}32$.
Every other L2D surrogate falls \emph{below} the classifier-only
baseline, which is consistent with the redundancy effects of
Appendix~\ref{app:synth-ce} persisting beyond linear models.

\emph{The decoupled surrogate preserves classifier quality, whereas several baselines degrade it.}
The decoupled surrogate keeps classifier accuracy at $90.8$--$90.9\%$ across all~$J$,
matching the classifier-only baseline.
By contrast, A-SM drops from $83.1\%$ ($J{=}8$) to $68.0\%$
($J{=}32$), a $15$~pp collapse consistent with expert gradients
reaching the class head through the class log-normalizer inside the expert
sigmoid (Proposition~\ref{prop:asm-coupling}), and OvA also degrades
($87.5\%\to 84.4\%$).
In the decoupled surrogate, the expert BCE term has no direct derivative
with respect to the class-head parameters, although both tasks can still
update the learned representation.

\emph{For most baselines, deferral lowers accuracy.}
When system accuracy falls \emph{below} classifier accuracy, the routing
decisions cost accuracy relative to never deferring.
At $J{=}8$, PiCCE achieves $88.8\%$ classifier accuracy but only
$85.8\%$ system accuracy ($-3.0$~pp from deferral), and Add.\ CE
similarly loses $8.7$~pp. Both route a substantial share of samples
(PiCCE $38.4\%$, Add.\ CE $47.1\%$) to experts with ${\approx}70\%$
accuracy, well below that of the classifier.

\emph{Add.\ CE and A-SM degrade monotonically with~$J$.}
Add.\ CE system accuracy drops from $81.5\%$ ($J{=}8$) to $71.2\%$
($J{=}32$), consistent with the gradient amplification of
Proposition~\ref{prop:additive-gradient}: the $(1{+}|\Jset|)$ factor
grows with the redundant pool while coverage falls to $30.5\%$.
A-SM shows a similar pattern ($86.0\%\to 78.3\%$), consistent with the
class--expert coupling of Proposition~\ref{prop:asm-coupling}.

\emph{Mao25 fails to learn a classifier.}
Across all~$J$, Mao25 ends at $100\%$ deferral and reaches system accuracy
$\approx70.3\%$ from expert answers alone, which is consistent with the
$O(1/(K{+}J))$ gradient magnitude through the softmax
(Appendix~\ref{app:proof-mao-curvature}).

\emph{The decoupled surrogate is stable in~$J$.}
The system accuracy of the decoupled surrogate stays at $91.8$--$91.9\%$
from $J{=}8$ to $J{=}32$, a change of only $0.1$~pp, compared with
$2.1$~pp for OvA and $10.3$~pp for Add.\ CE.
This stability on the present expert panel is consistent with
Proposition~\ref{prop:due-gradient}: each expert derivative depends only on
that expert's own score and target, and the classifier score gradient does
not depend on the experts at all, so adding redundant experts does not
rescale the class update.

\section{CIFAR-10H with Real Human Annotators}\label{app:cifar10h}

The synthetic-expert experiments above control the expert accuracy
profile exactly, but they leave open whether the same surrogate
pathologies appear when the annotators are real humans with correlated,
noisy labels. We therefore evaluate all methods on CIFAR-10H
\citep{peterson2019human}, a dataset that pairs every CIFAR-10 test
image with soft human annotations collected from $2{,}571$
Amazon Mechanical Turk workers.

\paragraph{Setup.}
We first pretrain a ResNet-12 backbone on the full CIFAR-10 training
set ($50\,000$ images, 80~epochs, SGD, lr~$0.05$, momentum~$0.9$,
weight decay~$5{\times}10^{-4}$) and then \emph{freeze} it, so that each
L2D method trains only a linear head on top of the frozen
$256$-dimensional features.
We split the CIFAR-10 test set ($10\,000$ images annotated by
CIFAR-10H) into $6\,000$ training, $2\,000$ validation, and $2\,000$ test
images, and all methods share the same pretrained backbone and data split.

\paragraph{Annotator selection.}
We retain workers with at least $200$ annotations and sample $J$ observed
responses without replacement for each image. The worker identity assigned
to an expert coordinate can therefore vary across images, and no response
is imputed. The single-seed runs use training seed~$0$ for each of five
response draws ($0$--$4$), and the three-seed study reruns training seeds
$0$, $1$, and $2$ on the same draws, so its seed-$0$ runs are separate runs.
Table~\ref{tab:cifar10h} reports the single-seed runs, while
Table~\ref{tab:cifar10h-main} in the main text and
Appendix~\ref{app:revision-cifar10h} report the three-seed study.
All methods use AdamW (lr~$10^{-3}$, 30~epochs).
For the decoupled surrogate, the total expert weight is $\lambda=2.0$ and
the weight decay is $10^{-3}$; the other methods use weight decay
$10^{-4}$. The three-seed study in Appendix~\ref{app:revision-cifar10h}
holds these settings fixed.

\begin{table}[htbp]
\centering\small
\setlength{\tabcolsep}{4pt}
\caption{CIFAR-10H with one training seed (seed 0): test metrics, mean $\pm$ sample SD over five annotation draws. Table~\ref{tab:cifar10h-main} uses three training seeds.}
\label{tab:cifar10h}
\begin{tabular}{@{}lccccc@{}}
\toprule
Method & $J$ & System acc. & Defer loss & Classifier acc. & Coverage \\
\midrule
Decoupled & 5 & $0.960_{\pm0.005}$ & $0.041_{\pm0.005}$ & $0.892_{\pm0.001}$ & $0.517_{\pm0.036}$ \\
OvA & 5 & $0.959_{\pm0.005}$ & $0.041_{\pm0.005}$ & $0.889_{\pm0.001}$ & $0.375_{\pm0.016}$ \\
A-SM & 5 & $0.959_{\pm0.004}$ & $0.042_{\pm0.004}$ & $0.890_{\pm0.002}$ & $0.341_{\pm0.035}$ \\
Add. CE & 5 & $0.953_{\pm0.004}$ & $0.047_{\pm0.004}$ & $0.889_{\pm0.002}$ & $0.248_{\pm0.037}$ \\
PiCCE & 5 & $0.954_{\pm0.003}$ & $0.046_{\pm0.003}$ & $0.889_{\pm0.001}$ & $0.390_{\pm0.020}$ \\
Classifier only & 5 & $0.889_{\pm0.000}$ & $0.111_{\pm0.000}$ & $0.889_{\pm0.000}$ & $1.000_{\pm0.000}$ \\
Mao25 (linear) & 5 & $0.950_{\pm0.005}$ & $0.050_{\pm0.005}$ & $0.116_{\pm0.133}$ & $0.000_{\pm0.000}$ \\
Decoupled & 10 & $0.958_{\pm0.003}$ & $0.042_{\pm0.003}$ & $0.892_{\pm0.003}$ & $0.517_{\pm0.038}$ \\
OvA & 10 & $0.954_{\pm0.007}$ & $0.046_{\pm0.007}$ & $0.799_{\pm0.203}$ & $0.285_{\pm0.164}$ \\
A-SM & 10 & $0.952_{\pm0.007}$ & $0.048_{\pm0.007}$ & $0.747_{\pm0.313}$ & $0.170_{\pm0.134}$ \\
Add. CE & 10 & $0.952_{\pm0.004}$ & $0.048_{\pm0.004}$ & $0.888_{\pm0.001}$ & $0.248_{\pm0.004}$ \\
PiCCE & 10 & $0.957_{\pm0.004}$ & $0.043_{\pm0.004}$ & $0.890_{\pm0.002}$ & $0.400_{\pm0.006}$ \\
Classifier only & 10 & $0.889_{\pm0.000}$ & $0.112_{\pm0.000}$ & $0.889_{\pm0.000}$ & $1.000_{\pm0.000}$ \\
Mao25 (linear) & 10 & $0.948_{\pm0.003}$ & $0.052_{\pm0.003}$ & $0.077_{\pm0.024}$ & $0.000_{\pm0.000}$ \\
Decoupled & 20 & $0.961_{\pm0.004}$ & $0.040_{\pm0.004}$ & $0.889_{\pm0.001}$ & $0.530_{\pm0.038}$ \\
OvA & 20 & $0.956_{\pm0.005}$ & $0.044_{\pm0.005}$ & $0.890_{\pm0.001}$ & $0.335_{\pm0.047}$ \\
A-SM & 20 & $0.949_{\pm0.005}$ & $0.051_{\pm0.005}$ & $0.470_{\pm0.435}$ & $0.002_{\pm0.004}$ \\
Add. CE & 20 & $0.951_{\pm0.003}$ & $0.049_{\pm0.003}$ & $0.742_{\pm0.331}$ & $0.183_{\pm0.103}$ \\
PiCCE & 20 & $0.953_{\pm0.002}$ & $0.047_{\pm0.002}$ & $0.890_{\pm0.001}$ & $0.399_{\pm0.022}$ \\
Classifier only & 20 & $0.888_{\pm0.000}$ & $0.113_{\pm0.000}$ & $0.888_{\pm0.000}$ & $1.000_{\pm0.000}$ \\
Mao25 (linear) & 20 & $0.948_{\pm0.004}$ & $0.053_{\pm0.004}$ & $0.045_{\pm0.014}$ & $0.000_{\pm0.000}$ \\
\bottomrule
\end{tabular}
\end{table}

\paragraph{Analysis.}

\emph{The decoupled surrogate leads at every pool size.}
At $J{=}5$ the rounded margin is narrow ($0.960$ vs.\ $0.959$ for OvA), but
the gap to OvA widens as the pool grows: at $J{=}20$ the decoupled
surrogate reaches $0.961$ while OvA drops to $0.956$ ($+0.5$~pp).
As with the synthetic experts of Table~\ref{tab:cifar10-redundant}, the
decoupled surrogate keeps a stable accuracy in these single-seed runs, and its
per-expert score derivative has no overlap-count factor, with a magnitude
controlled by $\lambda/J$.

\emph{The decoupled surrogate preserves classifier quality in these runs, whereas several baselines collapse.}
The classifier accuracy of the decoupled surrogate stays at $0.889$--$0.892$
across all~$J$, matching the standalone classifier baseline
($0.888$--$0.889$).
A-SM maintains $0.890$ at $J{=}5$ but collapses to $0.470$ at
$J{=}20$, with a standard deviation of $0.435$ that indicates highly
unstable training across draws; Add.\ CE drops from $0.889$ to $0.742$
over the same range, and OvA degrades to $0.799$ (std $0.203$) at $J{=}10$.
These methods compensate by deferring more often (A-SM reaches
${\ge}99\%$ deferral at $J{=}20$), so their system accuracy is
sustained by expert answers rather than by the classifier.
The collapses are consistent with the class--expert coupling
(Proposition~\ref{prop:asm-coupling}) and gradient amplification
(Proposition~\ref{prop:additive-gradient}) seen in the synthetic settings.
The replication in Appendix~\ref{app:revision-cifar10h} does not, however,
reproduce the Add.\ CE collapse: across its $45$ runs, the Add.\ CE
classifier accuracy stays within $0.8850$--$0.8925$.

\emph{The decoupled surrogate keeps a stable coverage.}
The coverage of the decoupled surrogate stays at ${\approx}0.52$, meaning
that it classifies about half the samples itself and defers the rest; this
routing yields a system accuracy above both the standalone classifier and
the full-deferral Mao25 system.

\emph{Mao25 defers everything.}
As on CIFAR-10 with synthetic experts, Mao25 converges to $100\%$ deferral
and never learns a classifier (mean classifier accuracy at most $0.116$).
Its system accuracy (${\approx}0.950$) therefore comes entirely from the
selected expert answers, which is consistent with the set-mass pathology of
Proposition~\ref{prop:mao25-setmass} on real human-annotator data.

\section{Covertype with Model Experts}\label{app:covertype}
The preceding experiments use either synthetic accuracy profiles
(Appendix~\ref{app:cifar10}) or real human annotators
(Appendix~\ref{app:cifar10h}). This benchmark replaces both with
\emph{pretrained models}: the experts are machine-learning classifiers
with different architectures, feature views, and training biases, as in
model routing, where a system sends each input to the most reliable model
in a pool.

\paragraph{Setup.}
We use Forest CoverType~\citep{blackard1999comparative}, a $7$-class
tabular benchmark with $54$ features (continuous topographic variables,
binary wilderness indicators, and binary soil-type indicators).
A single stratified $60/20/20$ train/validation/test split is made
once, and only the continuous features are standardized, using
training-set statistics.
The L2D backbone is a shared two-layer MLP
($54 \!\to\! 256 \!\to\! 128$, ReLU, batch normalization, no dropout).

\paragraph{Expert pool.}
The pool contains $J{=}5$ pretrained models, trained once on the
training split and then frozen:
\begin{enumerate}[label=(\roman*),itemsep=1pt,topsep=2pt,leftmargin=*]
\item a global histogram-gradient-boosted tree on all features;
\item a low-elevation specialist (reweighted toward low elevations
  and wilderness areas~1--2);
\item a high-elevation specialist (reweighted toward high elevations
  and wilderness areas~3--4);
\item a topography-view MLP trained only on the continuous variables;
\item a context-view random forest trained only on wilderness and soil
  indicators.
\end{enumerate}
The expert oracle accuracy on the validation split is $0.946$, a
$+9.8$~pp gain over the best single expert ($0.848$), so routing among the
experts is nontrivial.
The data split and the expert-pool training seed are held fixed, and only
the downstream L2D training seed changes across runs.

\paragraph{Protocol.}
All methods use MLP heads with one $64$-unit hidden layer and are trained
for $50$ epochs with AdamW, cosine learning-rate decay, learning rate
$10^{-3}$, and weight decay~$10^{-4}$. The decoupled model uses separate
class and expert heads, and it selects the total expert weight $\lambda$ from $\{0.25,0.5,1,2,4\}$ by validation deferral
loss; the selected total weights for seeds $0$--$3$ are $0.5,1,0.25,1$,
respectively. For every method, checkpoints are selected by the lowest
validation true deferral loss.

\begin{table}[h]
\centering
\small
\setlength{\tabcolsep}{3.5pt}
\caption{Covertype with model experts ($J{=}5$, mean$\pm$std over $4$ seeds).
Best values in \textbf{bold}; second-best \underline{underlined}.
``Cls.\ Acc.''~$=$~classifier accuracy (ignoring routing).
``Cov.''~$=$~coverage (fraction classified, not deferred).}
\label{tab:covertype}
\begin{tabular}{@{}l cccc@{}}
\toprule
Method & Sys.\ Acc.\ $\uparrow$ & Defer Loss $\downarrow$ & Cls.\ Acc.\ & Cov. \\
\midrule
\textbf{Decoupled}      & $\mathbf{0.934}_{\pm.000}$ & $\mathbf{0.066}_{\pm.000}$ & $\mathbf{0.941}_{\pm.001}$ & $0.712_{\pm.040}$ \\
Classifier only   & $\underline{0.929}_{\pm.001}$ & $\underline{0.071}_{\pm.001}$ & $\underline{0.929}_{\pm.001}$ & $1.000_{\pm.000}$ \\
PiCCE             & $0.908_{\pm.003}$ & $0.092_{\pm.003}$ & $0.911_{\pm.002}$ & $0.449_{\pm.007}$ \\
OvA               & $0.907_{\pm.001}$ & $0.093_{\pm.001}$ & $0.906_{\pm.000}$ & $0.437_{\pm.011}$ \\
A-SM              & $0.899_{\pm.001}$ & $0.101_{\pm.001}$ & $0.892_{\pm.001}$ & $0.439_{\pm.007}$ \\
Add.\ CE          & $0.899_{\pm.001}$ & $0.101_{\pm.001}$ & $0.907_{\pm.001}$ & $0.098_{\pm.002}$ \\
Mao25             & $0.877_{\pm.001}$ & $0.123_{\pm.001}$ & $0.351_{\pm.023}$ & $0.000_{\pm.000}$ \\
\bottomrule
\end{tabular}
\end{table}

\paragraph{Analysis.}

\emph{The decoupled surrogate is the only L2D method in our comparison whose system accuracy exceeds that of the standalone classifier.}
The decoupled surrogate reaches $0.934$ system accuracy, $+0.5$~pp above the
classifier-only baseline ($0.929$) and $+2.6$~pp above the best
baseline (PiCCE at $0.908$).
Every other L2D surrogate falls below the standalone classifier, so its
routing decisions cost accuracy relative to never deferring.

\emph{The decoupled surrogate preserves and even improves classifier quality.}
The decoupled surrogate achieves the highest classifier accuracy ($0.941$),
surpassing the standalone classifier ($0.929$), which shows that joint
utility learning is compatible with a strong class predictor in this
experiment. By contrast, the classifier of Mao25 collapses to $0.351$ and
that of A-SM degrades to $0.892$.

\emph{The decoupled surrogate has the highest coverage among L2D methods.}
The decoupled surrogate classifies $71.2\%$ of samples directly, far more than
PiCCE ($44.9\%$), A-SM ($43.9\%$), or Add.\ CE ($9.8\%$).
The baselines defer more often, so a larger share of their system accuracy
comes from the experts rather than from their own classifiers.

\emph{Mao25 collapses to full deferral.}
As in CIFAR-10 and CIFAR-10H, Mao25 converges to $100\%$ deferral
(zero coverage) and never learns a useful classifier, consistent with the
vanishing attractive gradient of Proposition~\ref{prop:mao25-setmass}.

\emph{The ordering extends to model experts.}
As on CIFAR-10 and in the fixed-draw CIFAR-10H table, the decoupled
surrogate ranks first and Mao25 last among L2D methods on this tabular benchmark with
pretrained model experts, with the ordering
decoupled $>$ PiCCE $\approx$ OvA $>$ A-SM $\approx$ Add.\ CE $>$ Mao25.
These results extend the comparison beyond vision backbones and human
annotators to a pool of heterogeneous model experts.

\section{Controlled Ablations and Replication}
\label{app:revision-experiments}

The experiments in this appendix separate the loss from its implementation
and test the extensions derived in Appendix~\ref{app:revision-theory}.
Fashion-MNIST~\citep{xiao2017fashion} provides a real image-classification task with controlled
synthetic experts, while the CIFAR-10H replication separately checks
training and annotation variation. Reported uncertainty is the sample
standard deviation, with the replicate unit specified in each caption, and
none of the comparisons below uses test metrics for checkpoint or
hyperparameter selection.

\subsection{Fashion-MNIST Protocol}
\label{app:revision-fashion-protocol}

We use the official Fashion-MNIST train/test split, holding out $5{,}000$
training examples for validation. The base panel has eight overlapping
specialists, with in-domain accuracy $0.90$ and out-of-domain accuracy $0.10$;
an incorrect prediction is uniform over the other nine labels. The domains,
in expert order, are
\[
\{0,2,6\},\ \{2,4,6\},\ \{0,3,4,6\},\ \{1,3,8\},\
\{5,7,9\},\ \{5,7,8\},\ \{7,9\},\ \{0,1,8\}.
\]
The experts are generated with seed $12345$, and within each training seed
all methods use the same split and realized expert predictions. Larger pools
repeat the realized eight-column panel exactly: $\mathcal M_{16}$ contains
two copies and $\mathcal M_{32}$ four copies. Duplication therefore changes
the action count without changing the best available action or adding
information.

All Fashion-MNIST experiments use four training seeds ($0$--$3$), AdamW
with learning rate $10^{-3}$ and weight decay $10^{-4}$, batch size $1024$,
and $25$ epochs. Validation is performed every $54$ optimization steps;
the checkpoint minimizes true validation deferral loss, or raw deployment
cost for the price study. The shared CNN has a $128$-dimensional feature
representation and a linear output projection. Separate-trunk arms use
two narrower CNNs to match the deployed parameter count. At $J=8,16,32$,
shared systems have $422{,}770$, $423{,}802$, and $425{,}866$ parameters;
separate systems have $423{,}014$, $423{,}702$, and $425{,}078$.

Cells with the same nominal configuration in different studies use one
canonical result: the shared/split decoupled arm at total $\lambda=1$ is the
same saved weight-study run in the architecture, schedule, and zero-price
tables, so that one stochastic realization of a configuration is not
presented as several methods.

\subsection{Loss, Head Layout, and Representation Sharing}
\label{app:revision-ablation}

The factorial ablation compares the decoupled and Add.\ CE losses with
(i) a shared trunk and split projections, (ii) a shared trunk and one
$(K{+}J)$ projection, and (iii) separate parameter-matched trunks. A single
projection can implement the decoupled loss by applying a softmax only to
its class block and a sigmoid to its expert block, so separation is a
property of the loss and normalization rather than of whether a layer is
stored as one tensor or two. The split and single shared projections are
initialized by copying the corresponding weights exactly.

\begin{table}[htbp]
\centering\small
\setlength{\tabcolsep}{4pt}
\caption{Fashion-MNIST loss and architecture ablation: test system accuracy (mean $\pm$ sample SD, four seeds). Parameter counts differ by less than $0.2\%$.}
\label{tab:revision-ablation}
\begin{tabular}{@{}lccc@{}}
\toprule
Loss and architecture & $J=8$ & $J=16$ & $J=32$ \\
\midrule
Dec., shared/split & $0.9593_{\pm0.0025}$ & $0.9585_{\pm0.0011}$ & $0.9585_{\pm0.0005}$ \\
Dec., shared/single & $0.9589_{\pm0.0010}$ & $0.9577_{\pm0.0013}$ & $0.9593_{\pm0.0013}$ \\
Dec., separate & $0.9560_{\pm0.0006}$ & $0.9552_{\pm0.0008}$ & $0.9544_{\pm0.0022}$ \\
Add. CE, shared/split & $0.9333_{\pm0.0053}$ & $0.9275_{\pm0.0041}$ & $0.9289_{\pm0.0063}$ \\
Add. CE, shared/single & $0.9351_{\pm0.0040}$ & $0.9276_{\pm0.0032}$ & $0.9249_{\pm0.0063}$ \\
Add. CE, separate & $0.9342_{\pm0.0028}$ & $0.9302_{\pm0.0029}$ & $0.9306_{\pm0.0031}$ \\
\bottomrule
\end{tabular}
\end{table}

\begin{table}[htbp]
\centering\small
\setlength{\tabcolsep}{4pt}
\caption{Standalone classifier accuracy for the same ablation (mean $\pm$ sample SD, four seeds).}
\label{tab:revision-ablation-classifier}
\begin{tabular}{@{}lccc@{}}
\toprule
Loss and architecture & $J=8$ & $J=16$ & $J=32$ \\
\midrule
Dec., shared/split & $0.9200_{\pm0.0016}$ & $0.9191_{\pm0.0019}$ & $0.9164_{\pm0.0031}$ \\
Dec., shared/single & $0.9189_{\pm0.0007}$ & $0.9152_{\pm0.0025}$ & $0.9182_{\pm0.0028}$ \\
Dec., separate & $0.9127_{\pm0.0043}$ & $0.9141_{\pm0.0022}$ & $0.9111_{\pm0.0019}$ \\
Add. CE, shared/split & $0.9074_{\pm0.0045}$ & $0.8995_{\pm0.0045}$ & $0.9049_{\pm0.0034}$ \\
Add. CE, shared/single & $0.9098_{\pm0.0041}$ & $0.9038_{\pm0.0020}$ & $0.9022_{\pm0.0027}$ \\
Add. CE, separate & $0.9060_{\pm0.0039}$ & $0.9042_{\pm0.0063}$ & $0.9052_{\pm0.0046}$ \\
\bottomrule
\end{tabular}
\end{table}

The decoupled loss has higher system accuracy in all $36$ paired
seed--pool--architecture contrasts. Its weakest four-seed mean is
$0.9544$, above Add.\ CE's strongest mean of $0.9351$. With separate
trunks, the mean loss advantage remains $2.18$--$2.50$ percentage points,
whereas head-layout differences are much smaller and have no consistent
sign. Sharing a representation adds a further $0.33$--$0.41$ percentage
points to the decoupled split-head system. These controlled comparisons
support the loss explanation while also showing the benefit of joint
representation learning. Because the paired contrasts share data, we do not
treat them as independent samples for a significance test.

\begin{table}[htbp]
\centering\small
\setlength{\tabcolsep}{4pt}
\caption{Shared-trunk gradients at aligned initialization (mean $\pm$ sample SD, four seeds). Ratio is expert/class gradient norm; cosine is the normalized inner product.}
\label{tab:revision-trunk-gradient}
\begin{tabular}{@{}lcccc@{}}
\toprule
$J$ & Dec. ratio & Dec. cosine & Add. CE ratio & Add. CE cosine \\
\midrule
8 & $0.406_{\pm0.056}$ & $-0.001_{\pm0.027}$ & $2.849_{\pm0.163}$ & $-0.721_{\pm0.027}$ \\
16 & $0.305_{\pm0.061}$ & $-0.004_{\pm0.053}$ & $3.123_{\pm0.080}$ & $-0.746_{\pm0.042}$ \\
32 & $0.208_{\pm0.021}$ & $-0.016_{\pm0.062}$ & $3.437_{\pm0.212}$ & $-0.732_{\pm0.028}$ \\
\bottomrule
\end{tabular}
\end{table}

For the shared-trunk diagnostic, we differentiate the class-target and
expert-target loss components separately with respect to the CNN parameters
at aligned initialization, using the same diagnostic batch within a seed.
The expert/class norm ratio is $0.21$--$0.41$ for the decoupled loss,
with mean cosine near zero. For Add.\ CE the ratio is $2.85$--$3.44$,
with mean cosine between $-0.75$ and $-0.72$. These quantities measure
trunk interaction at initialization only: Proposition~\ref{prop:due-gradient}
establishes zero direct mixed head curvature, and it does not assert
orthogonality of shared-trunk gradients throughout training.

\subsection{Expert-Weight Sensitivity}
\label{app:revision-weight}

We sweep total $\lambda\in\{0.0625,0.125,0.25,0.5,1,2,4,8\}$ under
exact panel duplication, completing $96$ seed--pool--weight cells. The
initial class parameters and the first eight expert rows are identical across
pool sizes within each seed.

\begin{table}[htbp]
\centering\small
\setlength{\tabcolsep}{4pt}
\caption{Fashion-MNIST under exact expert-panel duplication: test system accuracy for fixed total weight $\lambda$ (mean $\pm$ sample SD, four seeds).}
\label{tab:revision-weights}
\begin{tabular}{@{}lccc@{}}
\toprule
Total $\lambda$ & $J=8$ & $J=16$ & $J=32$ \\
\midrule
0.0625 & $0.9573_{\pm0.0014}$ & $0.9570_{\pm0.0015}$ & $0.9566_{\pm0.0006}$ \\
0.125 & $0.9579_{\pm0.0019}$ & $0.9574_{\pm0.0018}$ & $0.9571_{\pm0.0022}$ \\
0.25 & $0.9581_{\pm0.0024}$ & $0.9579_{\pm0.0007}$ & $0.9567_{\pm0.0027}$ \\
0.5 & $0.9587_{\pm0.0017}$ & $0.9589_{\pm0.0011}$ & $0.9584_{\pm0.0013}$ \\
1 & $0.9593_{\pm0.0025}$ & $0.9585_{\pm0.0011}$ & $0.9585_{\pm0.0005}$ \\
2 & $0.9591_{\pm0.0009}$ & $0.9594_{\pm0.0017}$ & $0.9607_{\pm0.0019}$ \\
4 & $0.9587_{\pm0.0013}$ & $0.9588_{\pm0.0017}$ & $0.9600_{\pm0.0009}$ \\
8 & $0.9582_{\pm0.0008}$ & $0.9588_{\pm0.0025}$ & $0.9577_{\pm0.0007}$ \\
\bottomrule
\end{tabular}
\end{table}

\begin{table}[htbp]
\centering\small
\setlength{\tabcolsep}{4pt}
\caption{Three weight choices on the same validation/test split (mean $\pm$ sample SD, four seeds); validation selection uses true deferral loss.}
\label{tab:revision-weight-rules}
\begin{tabular}{@{}lccc@{}}
\toprule
Weight choice & $J=8$ & $J=16$ & $J=32$ \\
\midrule
Fixed total $\lambda=1$ & $0.9593_{\pm0.0025}$ & $0.9585_{\pm0.0011}$ & $0.9585_{\pm0.0005}$ \\
Fixed $\lambda/J=0.125$ & $0.9593_{\pm0.0025}$ & $0.9594_{\pm0.0017}$ & $0.9600_{\pm0.0009}$ \\
Validation-selected & $0.9580_{\pm0.0013}$ & $0.9587_{\pm0.0015}$ & $0.9589_{\pm0.0007}$ \\
\bottomrule
\end{tabular}
\end{table}

At any fixed total weight in this grid, the spread of mean accuracy across
the three pool sizes is at most $0.0016$. Fixed total $\lambda=1$,
fixed per-expert weight $\lambda/J=0.125$, and validation selection all
produce stable performance here. Validation selection chooses, for each
seed and pool size, the total weight that minimizes validation deferral
loss; in the one tied cell ($J=16$, seed~$1$), the selected total weight
is $4$. Fixed total weight keeps the total expert loss budget fixed, while
fixed per-expert weight keeps each expert's gradient and the transfer
coefficient fixed. Both choices give valid surrogate objectives, and the
theoretical coefficient alone does not determine the best empirical
trade-off for a learned representation.

\subsection{One-Stage and Two-Stage Training}
\label{app:revision-schedules}

We compare shared one-stage training, disjoint one-stage training, and
our two-stage implementation. In the latter, stage one selects a class
CNN by validation classification error and freezes it; stage two trains
an independent Bernoulli-utility CNN, selecting by true validation deferral
loss. The disjoint simultaneous and sequential arms use the same
component architectures, and their objectives have the product minimizer
structure proved in Appendix~\ref{app:revision-two-stage}.

The fourth arm uses the logistic predictor--rejector surrogate of
\citet{mao2024regressionmultiexpertdeferral}, implemented with a learned
$(J{+}1)$-action router and the cost weights derived in that appendix, and
it follows the same frozen-classifier protocol and parameter budget as our
two-stage arm. This controlled comparison holds the optimizer fixed, so it
measures behavior under that protocol rather than after a method-specific
hyperparameter search.

\begin{table}[htbp]
\centering\small
\setlength{\tabcolsep}{4pt}
\caption{Fashion-MNIST training schedules: test system accuracy (mean $\pm$ sample SD, four seeds). The last row uses the logistic predictor--rejector surrogate of \citet{mao2024regressionmultiexpertdeferral}; all arms share the same optimization protocol.}
\label{tab:revision-schedules}
\begin{tabular}{@{}lccc@{}}
\toprule
Training procedure & $J=8$ & $J=16$ & $J=32$ \\
\midrule
Dec., one-stage shared & $0.9593_{\pm0.0025}$ & $0.9585_{\pm0.0011}$ & $0.9585_{\pm0.0005}$ \\
Dec., one-stage disjoint & $0.9546_{\pm0.0021}$ & $0.9548_{\pm0.0020}$ & $0.9557_{\pm0.0009}$ \\
Dec., two-stage & $0.9551_{\pm0.0013}$ & $0.9546_{\pm0.0012}$ & $0.9550_{\pm0.0009}$ \\
Predictor--rejector & $0.9036_{\pm0.0122}$ & $0.9107_{\pm0.0049}$ & $0.9118_{\pm0.0044}$ \\
\bottomrule
\end{tabular}
\end{table}

\begin{table}[htbp]
\centering\small
\setlength{\tabcolsep}{4pt}
\caption{Standalone classifier accuracy for the schedule comparison (mean $\pm$ sample SD, four seeds).}
\label{tab:revision-schedules-classifier}
\begin{tabular}{@{}lccc@{}}
\toprule
Training procedure & $J=8$ & $J=16$ & $J=32$ \\
\midrule
Dec., one-stage shared & $0.9200_{\pm0.0016}$ & $0.9191_{\pm0.0019}$ & $0.9164_{\pm0.0031}$ \\
Dec., one-stage disjoint & $0.9105_{\pm0.0090}$ & $0.9126_{\pm0.0032}$ & $0.9137_{\pm0.0043}$ \\
Dec., two-stage & $0.9134_{\pm0.0050}$ & $0.9121_{\pm0.0030}$ & $0.9121_{\pm0.0042}$ \\
Predictor--rejector & $0.9134_{\pm0.0050}$ & $0.9121_{\pm0.0030}$ & $0.9121_{\pm0.0042}$ \\
\bottomrule
\end{tabular}
\end{table}

Shared one-stage training exceeds disjoint one-stage training by
about $0.5$, $0.4$, and $0.3$ percentage points at $J=8,16,32$, with $11$
paired wins and one tie. Disjoint one-stage and our two-stage means differ
by less than $0.001$, with no consistent direction. The predictor--rejector
arm reaches $0.9036$--$0.9118$ under the fixed protocol; since its class
predictor is the same as in our two-stage arm, this gap isolates the router
comparison. Together, these results show that our formulation supports either schedule,
while shared one-stage training can use expert supervision to improve
the representation.

\subsection{Heterogeneous Consultation Prices}
\label{app:revision-price}

For $J=8$, expert $j$ has consultation price $b_j=\beta/j$, with
$\beta\in\{0,0.2,0.4,0.6,0.8\}$. We evaluate and select checkpoints
using the raw cost
\[
\widehat R_{\rm raw}=\frac1n\sum_{i=1}^n
\left[\mathbf1\{r(x_i)=0\}\mathbf1\{h(x_i)\ne y_i\}
+\sum_{j=1}^8\mathbf1\{r(x_i)=j\}
\left(\frac\beta j+\mathbf1\{m_j(x_i)\ne y_i\}\right)\right].
\]
For bounded-target surrogates, the training utility is
$1-(\beta/j+\mathbf1\{m_j\ne y\})/(1+\beta)$, and our class utility is
transformed to the same scale, as in Appendix~\ref{app:revision-costs}.
The classifier-only objective does not depend on the prices, so its single
canonical run per seed is reused at all prices. For positive prices,
PiCCE-CE is our own cost-weighted extension of the zero-price PiCCE
objective, not a published consultation-price formulation.

\begin{table}[htbp]
\centering\small
\setlength{\tabcolsep}{4pt}
\caption{Fashion-MNIST consultation prices: raw deployment risk (mean $\pm$ sample SD, four seeds). Expert $j$ has price $\beta/j$.}
\label{tab:revision-cost-risk}
\begin{tabular}{@{}lccccc@{}}
\toprule
Method & $\beta=0$ & $0.2$ & $0.4$ & $0.6$ & $0.8$ \\
\midrule
Decoupled & $0.0407_{\pm0.0025}$ & $0.0598_{\pm0.0014}$ & $0.0710_{\pm0.0017}$ & $0.0774_{\pm0.0019}$ & $0.0818_{\pm0.0014}$ \\
OvA & $0.0465_{\pm0.0008}$ & $0.0689_{\pm0.0033}$ & $0.0800_{\pm0.0027}$ & $0.0915_{\pm0.0027}$ & $0.0992_{\pm0.0024}$ \\
A-SM & $0.0514_{\pm0.0018}$ & $0.0725_{\pm0.0032}$ & $0.0869_{\pm0.0015}$ & $0.0960_{\pm0.0020}$ & $0.0996_{\pm0.0024}$ \\
Add. CE & $0.0701_{\pm0.0047}$ & $0.0937_{\pm0.0042}$ & $0.1182_{\pm0.0069}$ & $0.1317_{\pm0.0060}$ & $0.1399_{\pm0.0102}$ \\
PiCCE-CE extension & $0.0632_{\pm0.0026}$ & $0.1009_{\pm0.0102}$ & $0.1118_{\pm0.0092}$ & $0.1221_{\pm0.0034}$ & $0.1286_{\pm0.0043}$ \\
Classifier only & $0.0844_{\pm0.0014}$ & $0.0844_{\pm0.0014}$ & $0.0844_{\pm0.0014}$ & $0.0844_{\pm0.0014}$ & $0.0844_{\pm0.0014}$ \\
Mao25, linear member & $0.2436_{\pm0.1150}$ & $0.2963_{\pm0.0755}$ & $0.3068_{\pm0.0391}$ & $0.3958_{\pm0.0853}$ & $0.4352_{\pm0.0910}$ \\
\bottomrule
\end{tabular}
\end{table}

\begin{table}[htbp]
\centering\small
\setlength{\tabcolsep}{4pt}
\caption{Fashion-MNIST consultation prices: classifier coverage (mean $\pm$ sample SD, four seeds). Expert $j$ has price $\beta/j$.}
\label{tab:revision-cost-coverage}
\begin{tabular}{@{}lccccc@{}}
\toprule
Method & $\beta=0$ & $0.2$ & $0.4$ & $0.6$ & $0.8$ \\
\midrule
Decoupled & $0.7739_{\pm0.0139}$ & $0.8215_{\pm0.0033}$ & $0.8692_{\pm0.0007}$ & $0.9054_{\pm0.0050}$ & $0.9357_{\pm0.0026}$ \\
OvA & $0.6941_{\pm0.0040}$ & $0.7415_{\pm0.0213}$ & $0.8226_{\pm0.0232}$ & $0.8719_{\pm0.0077}$ & $0.9204_{\pm0.0054}$ \\
A-SM & $0.6169_{\pm0.0098}$ & $0.7256_{\pm0.0067}$ & $0.7890_{\pm0.0064}$ & $0.8599_{\pm0.0131}$ & $0.9127_{\pm0.0033}$ \\
Add. CE & $0.4740_{\pm0.0306}$ & $0.5586_{\pm0.0193}$ & $0.6067_{\pm0.0245}$ & $0.6734_{\pm0.0217}$ & $0.7358_{\pm0.0304}$ \\
PiCCE-CE extension & $0.6094_{\pm0.0134}$ & $0.8862_{\pm0.0960}$ & $0.9018_{\pm0.0884}$ & $0.9557_{\pm0.0190}$ & $0.9646_{\pm0.0143}$ \\
Classifier only & $1.0000_{\pm0.0000}$ & $1.0000_{\pm0.0000}$ & $1.0000_{\pm0.0000}$ & $1.0000_{\pm0.0000}$ & $1.0000_{\pm0.0000}$ \\
Mao25, linear member & $0.0000_{\pm0.0000}$ & $0.0000_{\pm0.0000}$ & $0.0000_{\pm0.0000}$ & $0.0000_{\pm0.0000}$ & $0.0236_{\pm0.0472}$ \\
\bottomrule
\end{tabular}
\end{table}

The decoupled surrogate has the lowest mean raw deployment risk at every
price. Its mean coverage rises from $0.7739$ to $0.9357$, increasing
monotonically in each seed, while its classifier accuracy remains between
$0.9154$ and $0.9200$. At the highest price, its raw risk is $0.0818$,
below the classifier-only risk of $0.0844$, so the router still identifies
useful deferrals. We report raw risk rather than accuracy because, with
positive consultation prices, raw cost is no longer one minus system
accuracy.

\subsection{CIFAR-10H: Training Seeds and Annotation Draws}
\label{app:revision-cifar10h}

We repeat the frozen-feature protocol of Appendix~\ref{app:cifar10h}
with three training seeds ($0,1,2$) for each of five annotation draws ($0$--$4$) and
$J\in\{5,10,20\}$, giving $45$ runs per method, while the data split,
backbone, and method-specific hyperparameters remain fixed. The main-text
Table~\ref{tab:cifar10h-main} reports draw $0$, with uncertainty across
training seeds, whereas Table~\ref{tab:revision-human-draws} first averages
the training seeds within each draw and reports variation across the five
draw means.

\begin{table}[htbp]
\centering\small
\setlength{\tabcolsep}{4pt}
\caption{CIFAR-10H: system accuracy over five annotation draws (mean $\pm$ sample SD of three-seed means); final column is the classifier-accuracy range over all 45 runs. The draw is the replicate unit in this table.}
\label{tab:revision-human-draws}
\begin{tabular}{@{}lcccc@{}}
\toprule
Method & $J=5$ & $J=10$ & $J=20$ & Classifier range \\
\midrule
Decoupled & $0.9606_{\pm0.0036}$ & $0.9596_{\pm0.0039}$ & $0.9599_{\pm0.0037}$ & $0.8860$--$0.8935$ \\
OvA & $0.9591_{\pm0.0033}$ & $0.9535_{\pm0.0046}$ & $0.9566_{\pm0.0034}$ & $0.4350$--$0.8920$ \\
A-SM & $0.9590_{\pm0.0042}$ & $0.9525_{\pm0.0047}$ & $0.9475_{\pm0.0033}$ & $0.0045$--$0.8930$ \\
Add. CE & $0.9550_{\pm0.0037}$ & $0.9526_{\pm0.0019}$ & $0.9538_{\pm0.0017}$ & $0.8850$--$0.8925$ \\
PiCCE & $0.9555_{\pm0.0045}$ & $0.9571_{\pm0.0028}$ & $0.9548_{\pm0.0010}$ & $0.7265$--$0.8925$ \\
Classifier only & $0.8897_{\pm0.0000}$ & $0.8885_{\pm0.0000}$ & $0.8890_{\pm0.0000}$ & $0.8875$--$0.8915$ \\
Mao25 (linear) & $0.9500_{\pm0.0037}$ & $0.9496_{\pm0.0016}$ & $0.9478_{\pm0.0021}$ & $0.0025$--$0.3230$ \\
\bottomrule
\end{tabular}
\end{table}

The classifier of the decoupled surrogate ranges from $0.8860$ to $0.8935$
across all $45$ runs, while its system accuracy benefits from routing.
Add.\ CE also preserves classifier quality in this replication, whereas
OvA, PiCCE, A-SM, and the linear Mao25 instance exhibit wider classifier
ranges. Our method defers less in $44$ or $45$ of the $45$ matched
comparisons against each other L2D method; we report these counts
descriptively and do not treat the three pool sizes as independent
replicates.

Because worker identities vary across images
(Appendix~\ref{app:cifar10h}), this experiment tests routing over human
responses, whereas Covertype uses a fixed pool of named model experts.

\section{Limitations}\label{sec:limitations}
The analysis assumes labeled training examples with observed expert outputs
and compares policies under a fixed deployment distribution. Missing
annotations, changes in expert behavior, and availability constraints
would require an observation or deployment model beyond this setting. The
population result is an unrestricted surrogate-regret transfer, and it does
not assert that finite neural-network training reaches the population
infimum. Representation sharing can improve estimation, as the ablations
show, but the zero mixed head block does not remove interaction through
the shared trunk. The controlled studies isolate specific mechanisms, and
their results are tied to the reported expert panels and training protocols.

\section{Broader Impacts}\label{sec:impact}
This work contributes theory and algorithms for learning-to-defer.
Deferral policies that estimate class posteriors and expert utilities more
accurately could improve human--AI systems in settings such as medical
triage, by routing difficult cases to qualified experts. As with any
deferral system, the risks include over-reliance on imperfect experts, error
rates that differ across subgroups, and unclear accountability when a model
and a human share a decision.

\section*{NeurIPS Paper Checklist}

\begin{enumerate}

\item {\bf Claims}
    \item[] Question: Do the main claims made in the abstract and introduction accurately reflect the paper's contributions and scope?
    \item[] Answer: \answerYes{}
    \item[] Justification: The abstract and introduction state the two-axis analysis (statistical target and optimization geometry) of five augmented-action surrogates, the decoupled surrogate, and the consistency bound whose constant is independent of $J$ at fixed $\lambda/J$; each claim is established in Sections~\ref{sec:mismatch}--\ref{sec:experiments} and the appendix.

\item {\bf Limitations}
    \item[] Question: Does the paper discuss the limitations of the work performed by the authors?
    \item[] Answer: \answerYes{}
    \item[] Justification: Appendix~\ref{sec:limitations} discusses the observation and deployment assumptions, the unrestricted nature of the guarantee, interaction through a shared representation, and the experimental scope; broader impacts are discussed in Appendix~\ref{sec:impact}.

\item {\bf Theory assumptions and proofs}
    \item[] Question: For each theoretical result, does the paper provide the full set of assumptions and a complete (and correct) proof?
    \item[] Answer: \answerYes{}
    \item[] Justification: All propositions and the main consistency theorem are stated with their assumptions in the main text; complete proofs, together with worked examples, are given in Appendices~\ref{app:proof-bayes-rule}--\ref{app:constant-comparison} and Appendix~\ref{app:revision-theory}.

    \item {\bf Experimental result reproducibility}
    \item[] Question: Does the paper fully disclose all the information needed to reproduce the main experimental results of the paper to the extent that it affects the main claims and/or conclusions of the paper (regardless of whether the code and data are provided or not)?
    \item[] Answer: \answerYes{}
    \item[] Justification: Data splits, expert construction, and seed counts are reported in Appendices~\ref{app:synthetic}--\ref{app:covertype}; backbones, optimizers, learning rates, schedules, and the per-expert weight $\lambda/J$ for the real-data benchmarks are reported in Appendices~\ref{app:cifar10}--\ref{app:covertype} and \ref{app:revision-experiments}. Source code for all experiments is included in the supplementary material.

\item {\bf Open access to data and code}
    \item[] Question: Does the paper provide open access to the data and code, with sufficient instructions to faithfully reproduce the main experimental results, as described in supplemental material?
    \item[] Answer: \answerYes{}
    \item[] Justification: Source code with run scripts is provided in the supplementary material. All datasets used (CIFAR-10, CIFAR-10H, Forest CoverType, Fashion-MNIST) are public; loading and preprocessing instructions are included in the code repository.

\item {\bf Experimental setting/details}
    \item[] Question: Does the paper specify all the training and test details (e.g., data splits, hyperparameters, how they were chosen, type of optimizer) necessary to understand the results?
    \item[] Answer: \answerYes{}
    \item[] Justification: Training protocols (optimizer, learning rate, epochs, head architecture, validation-based checkpointing, and expert-weight selection) are described for each benchmark, together with weight decay for CIFAR-10H, Covertype, and Fashion-MNIST and the Fashion-MNIST batch size, in Appendices~\ref{app:cifar10}--\ref{app:covertype} and \ref{app:revision-experiments}.

\item {\bf Experiment statistical significance}
    \item[] Question: Does the paper report error bars suitably and correctly defined or other appropriate information about the statistical significance of the experiments?
    \item[] Answer: \answerYes{}
    \item[] Justification: Tables report means and sample standard deviations: four training seeds for CIFAR-10, Covertype, and Fashion-MNIST; three training seeds for the fixed-draw CIFAR-10H main table. The full CIFAR-10H replication separately reports variation across five three-seed draw means. The synthetic suites report standard deviations over three to five seeds, and single-training-seed CIFAR-10H results over five draws (Table~\ref{tab:cifar10h}) are labeled separately. Repeated pool sizes are not treated as independent replicates for inference.

\item {\bf Experiments compute resources}
    \item[] Question: For each experiment, does the paper provide sufficient information on the computer resources (type of compute workers, memory, time of execution) needed to reproduce the experiments?
    \item[] Answer: \answerYes{}
    \item[] Justification: Appendix~\ref{app:compute} reports the hardware (NVIDIA A100 GPUs with 40GB of memory, and H100/H200 GPUs for the Fashion-MNIST studies), one GPU per run, and the execution time: the synthetic suites finish within minutes per seed, and each real-data benchmark takes about one hour on a single H100 GPU.
    
\item {\bf Code of ethics}
    \item[] Question: Does the research conducted in the paper conform, in every respect, with the NeurIPS Code of Ethics \url{https://neurips.cc/public/EthicsGuidelines}?
    \item[] Answer: \answerYes{}
    \item[] Justification: The work uses only public benchmark datasets, collects no new human-subject data, and the authors have reviewed and adhere to the NeurIPS Code of Ethics.

\item {\bf Broader impacts}
    \item[] Question: Does the paper discuss both potential positive societal impacts and negative societal impacts of the work performed?
    \item[] Answer: \answerYes{}
    \item[] Justification: Appendix~\ref{sec:impact} discusses broader impacts, covering both the potential benefits of more accurate human--AI deferral and the risks of over-reliance on imperfect experts.
    
\item {\bf Safeguards}
    \item[] Question: Does the paper describe safeguards that have been put in place for responsible release of data or models that have a high risk for misuse (e.g., pre-trained language models, image generators, or scraped datasets)?
    \item[] Answer: \answerNA{}
    \item[] Justification: The paper releases only research code for training small classification and routing models on public benchmarks; it does not release datasets, generative models, or pretrained language models posing high misuse risk.

\item {\bf Licenses for existing assets}
    \item[] Question: Are the creators or original owners of assets (e.g., code, data, models), used in the paper, properly credited and are the license and terms of use explicitly mentioned and properly respected?
    \item[] Answer: \answerYes{}
    \item[] Justification: CIFAR-10~\citep{krizhevsky2009learning}, CIFAR-10H~\citep{peterson2019human}, Forest CoverType~\citep{blackard1999comparative}, and Fashion-MNIST~\citep{xiao2017fashion} are credited where used. Fashion-MNIST is released under the MIT license, CIFAR-10H under CC BY-NC-SA 4.0, and Covertype under CC BY 4.0 (UCI repository); CIFAR-10 is used under its distribution terms. All are used in accordance with their licenses and terms of use.

\item {\bf New assets}
    \item[] Question: Are new assets introduced in the paper well documented and is the documentation provided alongside the assets?
    \item[] Answer: \answerYes{}
    \item[] Justification: The supplementary material contains the research code implementing the decoupled surrogate and all baselines, with a README documenting usage, experiment scripts, and expected outputs.

\item {\bf Crowdsourcing and research with human subjects}
    \item[] Question: For crowdsourcing experiments and research with human subjects, does the paper include the full text of instructions given to participants and screenshots, if applicable, as well as details about compensation (if any)?
    \item[] Answer: \answerNA{}
    \item[] Justification: No new crowdsourcing or human-subjects data collection was conducted; CIFAR-10H labels are used as released by \citet{peterson2019human}.

\item {\bf Institutional review board (IRB) approvals or equivalent for research with human subjects}
    \item[] Question: Does the paper describe potential risks incurred by study participants, whether such risks were disclosed to the subjects, and whether Institutional Review Board (IRB) approvals (or an equivalent approval/review based on the requirements of your country or institution) were obtained?
    \item[] Answer: \answerNA{}
    \item[] Justification: The paper does not involve research with human subjects; no IRB review is required.

\item {\bf Declaration of LLM usage}
    \item[] Question: Does the paper describe the usage of LLMs if it is an important, original, or non-standard component of the core methods in this research? Note that if the LLM is used only for writing, editing, or formatting purposes and does \emph{not} impact the core methodology, scientific rigor, or originality of the research, declaration is not required.
    \item[] Answer: \answerNA{}
    \item[] Justification: The proposed method does not use an LLM. LLM assistance was used during manuscript revision for exposition, derivation checks, source auditing, and table-generation code; the authors are responsible for validating the final claims and artifacts.

\end{enumerate}

\end{document}